\documentclass{article}

\usepackage{silence}
\usepackage[table,dvipsnames,xcdraw]{xcolor}
\definecolor{salmon}{RGB}{197,90,17}
\definecolor{mydarkblue}{rgb}{0,0.08,0.45}
\definecolor{note_fontcolor}{rgb}{0.800781, 0.800781, 0.800781}

\usepackage{enumerate}

    \PassOptionsToPackage{numbers, compress, sort}{natbib}

     \usepackage[final]{neuripsx}

\usepackage[utf8]{inputenc} %
\usepackage[T1]{fontenc}    %
\usepackage[hyphens]{url}            %
\usepackage[
    colorlinks,
    bookmarks=true,
    breaklinks=true,
    linkcolor=mydarkblue,
    citecolor=mydarkblue,
    filecolor=mydarkblue,
    urlcolor=mydarkblue,]{hyperref}       %
\usepackage{booktabs}       %
\usepackage{amsfonts}       %
\usepackage{nicefrac}       %
\usepackage{microtype}      %
\usepackage{multicol}
\usepackage{subcaption}
\usepackage{amssymb, amsmath, amsthm}
\usepackage{thmtools}
\usepackage{mathtools}
\usepackage{verbatim}
\usepackage{tcolorbox}
\usepackage{algorithm,algpseudocode}
\makeatletter
\renewcommand{\ALG@name}{Algorithm}
\makeatother

\usepackage{subfiles}

\usepackage{wrapfig}

\usepackage{breakurl}
\usepackage[capitalize,nameinlink]{cleveref}
\usepackage{mymacros}

\newcommand{\Zhat}{\hat Z}
\newcommand{\Zdot}{\dot Z}

\newtheorem{thm}{Theorem}[section]
\newtheorem{cor}[thm]{Corollary}

\newtheorem{assm}[thm]{Assumption}

\newtheorem{setup}[thm]{Setup}
\newtheorem{lem}[thm]{Lemma}
\newtheorem{prop}[thm]{Proposition}
\theoremstyle{definition}
\newtheorem{defn}[thm]{Definition}
\newtheorem{rem}[thm]{Remark}

\newtcolorbox[auto counter,crefname={Box}{Boxes}]{pabox}[2][]{%
title=Box~\thetcbcounter $\quad$ #2, label={#1}}

\usepackage{bbm}
\crefname{thm}{\text{Theorem}}{\text{Theorems}}
\crefname{assm}{\text{Assumption}}{\text{Assumptions}}
\crefname{defn}{\text{Definition}}{\text{Definitions}}
\crefname{prop}{\text{Proposition}}{\text{Propositions}}
\crefname{cor}{\text{Corollary}}{\text{Corollaries}}
\crefname{lemma}{\text{Lemma}}{\text{Lemmas}}

\newcommand{\onev}{\mathbf{1}}
\newcommand{\trsp}{\top}
\newcommand{\distto}{\xrar{\mathrm{d}}}
\newcommand{\asto}{\xrar{\mathrm{a.s.}}}

\newcommand{\disteq}{\overset{\mathrm{d}}{=}}

\newcommand{\odefeq}{\mathbin{\overset{\text{\tiny{def}}}{=}}}

\newcommand{\defeq}{\mathrel{\raisebox{-0.3ex}{$\odefeq$}}}

\renewcommand{\cite}{\citep}

\renewcommand{\cite}{\citep}

\newcommand{\netsort}{\texorpdfstring{{$\textsc{Netsor}\trsp$}}{NetsorT}}

\newcommand{\netsortplus}{\texorpdfstring{{$\textsc{Netsor}\trsp^+$}}{NetsorT+}}

\newcommand{\loss}{\mathcal{L}}
\newcommand{\id}{\mathrm{id}}

\global\long\def\del{\delta}%

\global\long\def\isp{\mathcal{X}}%

\global\long\def\Ee{\mathcal{E}}%

\newcommand{\refZMatMul}{\textbf{\ref{Z:MatMul}}}

\newcommand{\refZhat}{\textbf{\ref{Z:Zhat}}}
\newcommand{\refZdot}{\textbf{\ref{Z:Zdot}}}
\newcommand{\refMatMul}{\textbf{\ref{instr:matmul}}}

\newcommand{\refNonlinPlus}{\textbf{\ref{instr:nonlin+}}}
\newcommand{\refMoment}{\textbf{\ref{instr:moment}}}

\makeatletter
\let\orgdescriptionlabel\descriptionlabel
\newcommand*{\@restrictlabeltext}[1]{#1\protected@edef\@currentlabel{#1}}
\newcommand*{\nolabel}[1]{#1}%
\renewcommand*{\descriptionlabel}[1]{%
  \let\orglabel\label
  \let\label\@gobble
  \let\orig@hfil\hfil
  \def\hfil{}%
  \let\nolabel\@gobble
  \let\restrictlabeltext\@firstofone
  \phantomsection
  \protected@edef\@currentlabel{#1}%
  \let\hfil\orig@hfil
  \let\label\orglabel
  \let\restrictlabeltext\@restrictlabeltext
  \orgdescriptionlabel{#1}%
}
\makeatother

\title{
Feature Learning in Infinite-Width Neural Networks}

\author{%
  Greg Yang\\
  Microsoft Research AI\\
  \texttt{gregyang@microsoft.com} \\
  \And
  Edward J. Hu\thanks{Work done partly during the \href{https://www.microsoft.com/en-us/research/academic-program/microsoft-ai-residency-program/}{Microsoft AI Residency Program}} \\
  Microsoft Azure AI\\
  \texttt{edwardhu@microsoft.com} \\
}

\begin{document}

\maketitle

\enlargethispage*{\baselineskip}
\newcommand{\repo}{\url{github.com/edwardjhu/TP4}}

\begin{abstract}
    As its width tends to infinity, a deep neural network's behavior under gradient descent can become simplified and predictable (e.g.\ given by the Neural Tangent Kernel (NTK)), if it is parametrized appropriately (e.g.\ the NTK parametrization).
    However, we show that the standard and NTK parametrizations of a neural network do not admit infinite-width limits that can \emph{learn} features, which is crucial for pretraining and transfer learning such as with BERT.
    We propose simple modifications to the standard parametrization to allow for feature learning in the limit. 
    Using the \emph{Tensor Programs} technique, we derive explicit formulas for such limits.
    On Word2Vec and few-shot learning on Omniglot via MAML, two canonical tasks that rely crucially on feature learning, we compute these limits exactly.
    We find that they outperform both NTK baselines and finite-width networks, with the latter approaching the infinite-width feature learning performance as width increases.

    More generally, we classify a natural space of neural network parametrizations that generalizes standard, NTK, and Mean Field parametrizations.
    We show 1) any parametrization in this space either admits feature learning or has an infinite-width training dynamics given by kernel gradient descent, but not both; 2) any such infinite-width limit can be computed using the Tensor Programs technique.
    Code for our experiments can be found at \repo{}.
\end{abstract}
\begin{figure}[h!]
    \vspace{-10pt}
    \centering
    \includegraphics[width=0.8\textwidth]{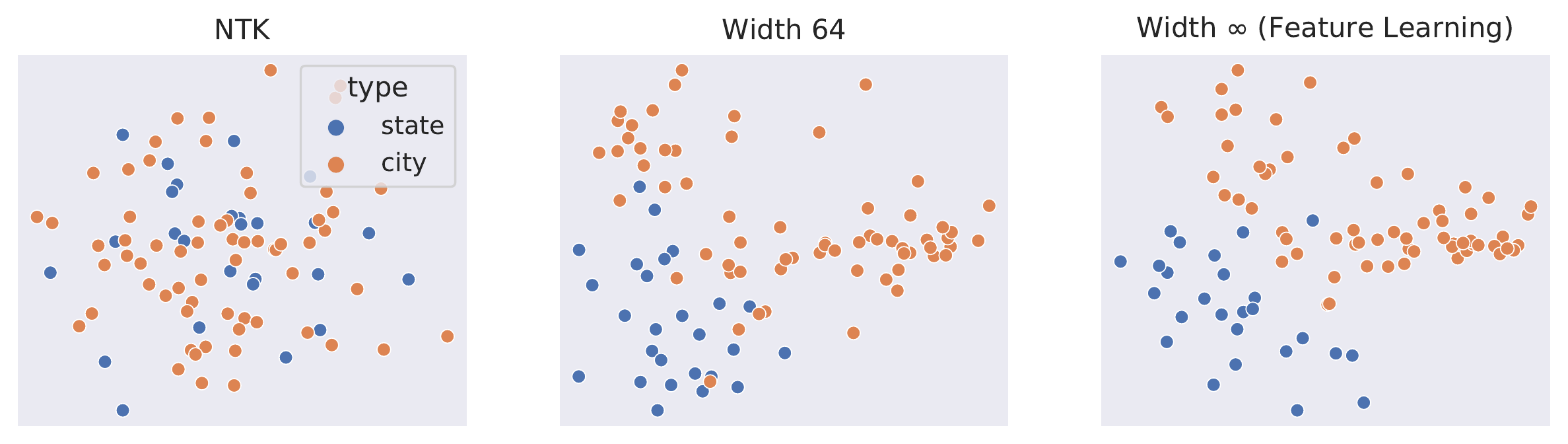}
    \caption{PCA of Word2Vec embeddings of top US cities and states, for NTK, width-64, and width-$\infty$ feature learning networks (\cref{defn:MUP}).
    NTK embeddings are essentially random, while cities and states get naturally separated in embedding space as width increases in the feature learning regime.
    }
    \label{fig:UScitystateFL}
    \vspace{-10pt}
\end{figure}

\section{Introduction}
\label{sec:intro}

The study of infinite-width limits of neural networks, in particular the Neural Tangent Kernel (NTK), has recently solved many longstanding open problems on the optimization and generalization of overparametrized neural networks \cite{jacot_neural_2018}. However, in the NTK limit, (last layer) features learned during pretraining are essentially the same as those from random initialization (\cref{cor:dichotomyMain,{thm:main}}); this is verified empirically in Word2Vec in \cref{fig:UScitystateFL}.
As feature learning (e.g.\ Imagenet and BERT) lies at the core of deep learning's far-ranging impact so far \citep{he_deep_2016,devlin_bert_2019,brown_language_2020}, this insight amounts to a fatal weakness of the NTK theory as a model of neural networks in practice.

\enlargethispage{\baselineskip}

We seek to capture feature learning in overparametrized networks by considering other parametrizations and their infinite-width limits.
By slightly modifying the standard parametrization (SP), in fact, we can enable feature learning that is \emph{maximal} in a sense to be explained shortly.
We describe how to compute this limit exactly (and rigorously) via the \emph{Tensor Programs} technique developed in \cite{scaling,TP1,TP2,TP3}.

\begin{wrapfigure}{r}{0.35\textwidth}
    \vspace{-8pt}
    \begin{center}
        \includegraphics[width=0.35\textwidth]{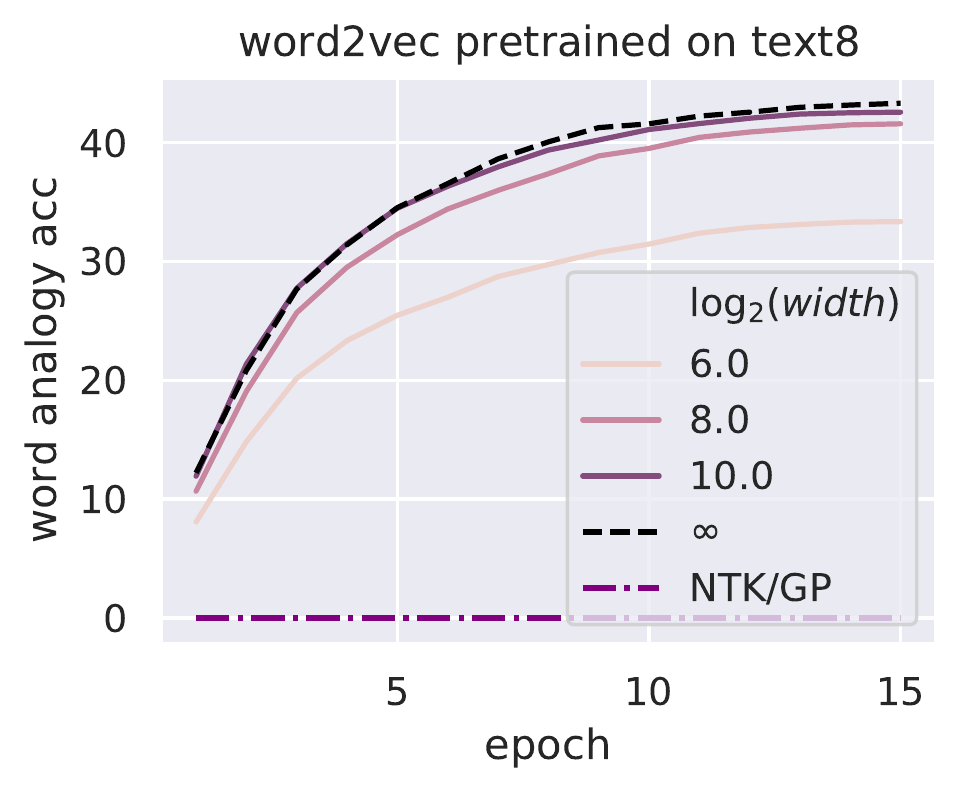}
    \end{center}
    \vspace{-20pt}
\end{wrapfigure}
\paragraph{Feature Learning Infinite-Width Networks on Real Tasks}
We explicitly calculate this limit for the tasks of Word2Vec \citep{mikolov_distributed_2013,mikolov_efficient_2013} and few-shot learning on Omniglot via MAML
\citep{finnMAML2017},%
\footnote{Short for \emph{Model Agnostic Meta-Learning}} two standard tasks relying crucially on feature learning.
In Word2Vec, an important early instance of large-scale language pretraining, we must learn, in an unsupervised manner, word embeddings so that similar words have close embeddings.
Then we test the learned embeddings on the word analogy task, which asks questions of the kind ``what to a queen is as a man to a woman?''
In few-shot learning, the model is asked to make predictions given only a handful (e.g.\ 5) of labeled examples.
Metalearning/MAML makes this possible by having the model learn good representations of typical examples that can \emph{adapt} quickly, via a small number of SGD steps, to new few-shot learning tasks.
On both tasks, we find our feature learning infinite-width networks outperform both NTK baselines and finite-width networks, with the latter approaching the infinite-width performance as width increases.
\textbf{\textcolor{purple}{Figure right}} shows this for one of our Word2Vec results.
See \cref{sec:experiments} for our other experiments.

\paragraph{abc-Parametrizations}
This paper studies a natural class of parametrizations, which we call the \emph{abc-Parametrization} and describe here.
Consider an $L$-hidden-layer perceptron: For weight matrices $W^{1}\in\R^{n\times d}$ and $W^{2},\ldots,W^{L}\in\R^{n\times n}$, and nonlinearity $\phi:\R\to\R$, such a neural network on input $\xi\in\R^{d}$ is given by $h^{1}(\xi)=W^{1}\xi\in\R^{n}$, and 
\begin{align}
x^{l}(\xi)=\phi(h^{l}(\xi))\in\R^{n},\quad h^{l+1}(\xi)=W^{l+1}x^{l}(\xi)\in\R^{n},\quad\text{for \ensuremath{l=1,\ldots,L-1},}\label{eqn:MLP}
\end{align}
and the network output (also called the \emph{logit(s)}) is $f(\xi)=W^{L+1}x^{L}(\xi)$ for $W^{L+1}\in\R^{1\times n}$. 
An \emph{abc-parametrization} is specified by a set of numbers $\{a_l, b_l\}_l \cup \{c\}$ such that 
\begin{enumerate}[\hspace{30pt} (a)]
    \item We parametrize each weight as $W^{l}=n^{-a_{l}}w^{l}$ for actual trainable parameter $w^{l}$
    \item We initialize each $w_{\alpha\beta}^{l}\sim\Gaus(0,n^{-2b_{l}})$, and
    \item The SGD learning rate is $\eta n^{-c}$ for some width-independent $\eta$.%
    \footnote{Observe that by changing $a_l$, $b_l$ while holding $a_l+b_l$ fixed, we effectively give layer $l$ its own learning rate.}
    \footnote{
        One can further include a set of constants in front of $n^{-a_l}$ and $n^{-b_l}$, for example powers of input dimension $d$, but we shall keep it simple here as we are only concerned with scaling behavior with $n$.}
\end{enumerate}

\emph{Examples:} 
The NTK parametrization (NTP) \citep{jacot_neural_2018} has $a_1=0$ and $a_l=1/2$ for $l\ge 2$; $b_l=0$ for all $l$; $c=0$.
When depth $L=1$, the Mean Field parametrization (MFP) \citep{sirignano_mean_2018,mei_mean_2018,chizat_global_2018,rotskoff_neural_2018} has $a_1=0$, $a_2=1$; $b_l=0$ for all $l$; $c=-1$.
The standard parametrization (SP)
available as the default setting in PyTorch \citep{pytorch}%
\footnote{This is also known as the ``fanin'' or ``Lecun'' initialization; ``Kaiming'' initialization is the same up to multiplicative constants.
The default in Tensorflow \citep{tensorflow} uses Glorot initialization, where the variance of an entry scales like $1/(fanin+fanout)$.
This causes the first layer preactivation to converge to 0 as $n\to\infty$, and thus yields pathological behavior in the limit.
} has $a_l=0$ for all $l$; $b_1=0$ and $b_l=1/2$ for $l\ge 2$; $c=0$.
However, we shall see that $c$ is too small (learning rate too large) in SP.
We can define abc-parametrization and generalize our results to arbitrary neural architectures (\cref{sec:abcAnyArch}), but we shall focus on MLPs in the main text.

\paragraph{Dynamical Dichotomy}

For any abc-parametrization, if $c$ is too small (i.e.\ learning rate too large), SGD can lead to blowup of preactivation and/or logits; we say this parametrization is \emph{unstable}.
In practice this translates to numerical issues.
If $c$ is too large (i.e.\ learning rate too small), then the function computed by the network does not change in finite time; we say this parametrization is \emph{trivial}.
We prove what we call the \emph{Dynamical Dichotomy theorem} (\cref{cor:dichotomyMain}):
\begin{center}
    \it
    \parbox{0.92\linewidth}{\centering
    Any \emph{nontrivial} \emph{stable} abc-parametrization yields a (discrete-time) infinite-width limit.
    \\
    This limit either 1) allows the embedding $x^L(\xi)$ to evolve nontrivially (\cref{defn:featurelearningMain}) or\\
    2) is described by kernel gradient descent in function space (\cref{defn:kernelregimeMain}), but not both.
    }
\end{center}
We call the former kind a \emph{feature learning limit} and the latter a \emph{kernel limit}.
For 1-hidden-layer MLPs, the former is exemplified by MFP, and the latter, NTP.
This dichotomy implies that certain functional dynamics, such as higher order generalizations of the NTK dynamics, are not valid infinite-width limits (see \cref{rem:invalidlimits}).
In addition, the neural network function $f$ (defined in \cref{eqn:MLP}) in any feature learning limit must be identically 0 at initialization (see \cref{cor:FLnoGP}).%
\footnote{We stress this is in the $n\to\infty$ limit, so does not contradict the feature learning seen in finite-width SP NN.}

\begin{wrapfigure}{r}{0.45\textwidth}
    \vspace{-10pt}
    \begin{center}
        \includegraphics[width=0.45\textwidth]{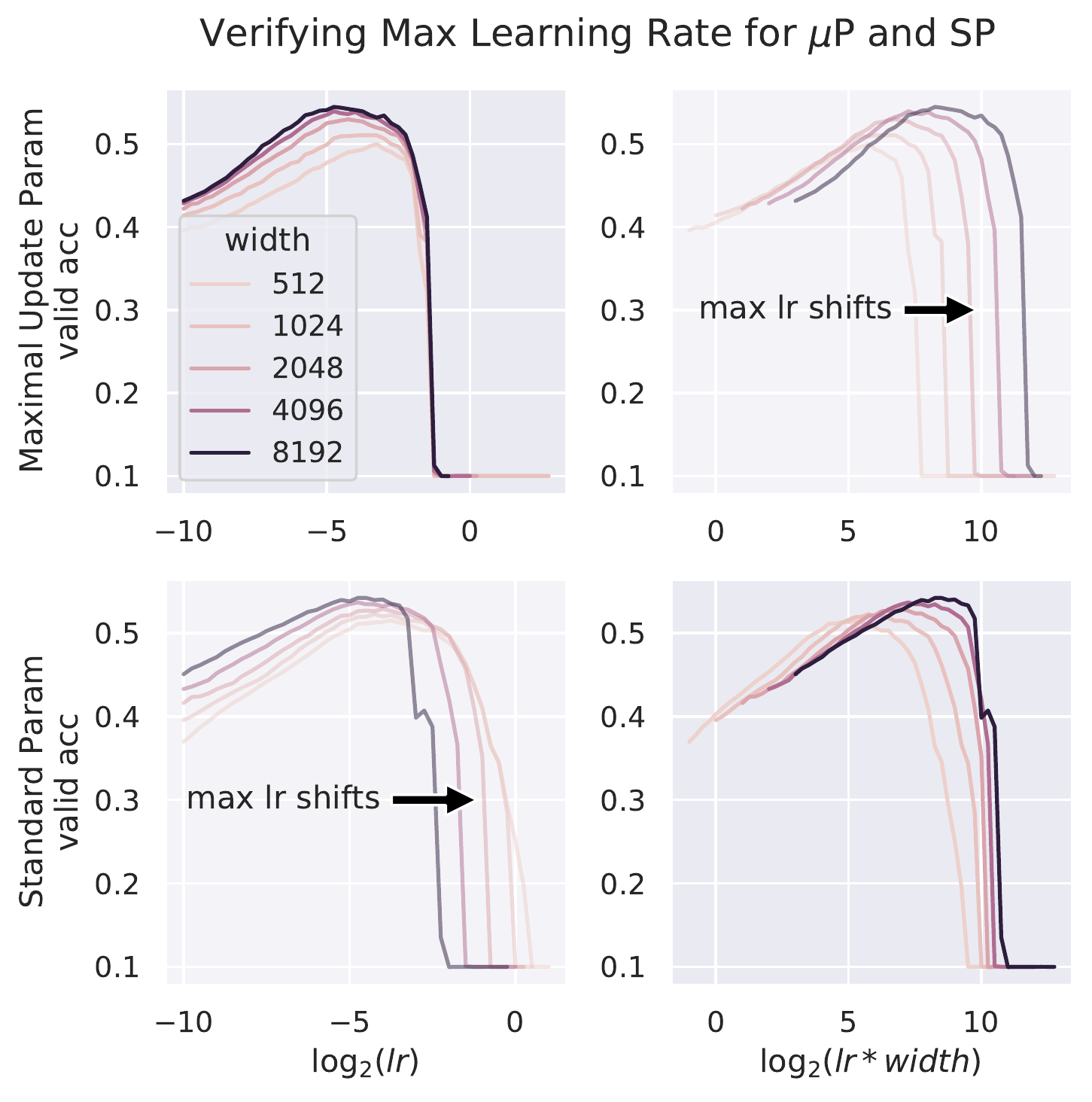}
    \end{center}
    \vspace{-20pt}
\end{wrapfigure}

\paragraph{Standard Param.\ Does Not Learn Features}
We show that the SP (resp.\ NTP) can only allow $O(1/width)$ (resp.\ $O(1)$) learning rate (i.e. $c=1$, resp. $c=0$), so as to avoid blowup, and yield kernel limits (\cref{sec:standardparam}).
Instead, we propose a parametrization that has $\Theta(1)$ max learning rate and admits feature learning \emph{maximally}:
it allows every parameter to be updated maximally (in terms of scaling with width) without leading to blowup (\cref{sec:MUP}).
We thus call it the \emph{Maximal Update Parametrization} (abbreviated MUP or $\mu$P).
It is given by $a_1=-1/2$, $a_{L+1}=1/2$, and $a_l=0$ for all $2\le l\le L$; $b_l = 1/2$ for all $l$; $c=0$.
In a 1-hidden-layer MLP, this specializes to MFP, up to symmetry (see \cref{stmt:SGDsymmetry}). 
The ``feature learning limits'' mentioned above in our main experiments are $\mu$P limits.
\textbf{\textcolor{purple}{{Figure to the right:}}}
We empirically verify our max learning rate predictions on relu MLP with 2 hidden layers, trained with square loss on CIFAR10.
We plot learning rate vs accuracy in each subplot.
Each curve represents MLP with a specific width.
The right edge of each curve indicates the max learning rate.
The diagonal subplots scale the x-axes (log learning rate) in the correct width-scaling for the corresponding parametrizations.
We see, indeed, max learning rate for SP scales like $1/width$ but is constant in $\mu$P.

\paragraph{Key Theoretical Idea: \emph{Tensor Programs}}
\begin{wrapfigure}{r}{0.5\textwidth}
    \vspace{-25pt}
    \begin{center}
        \includegraphics[width=0.5\textwidth]{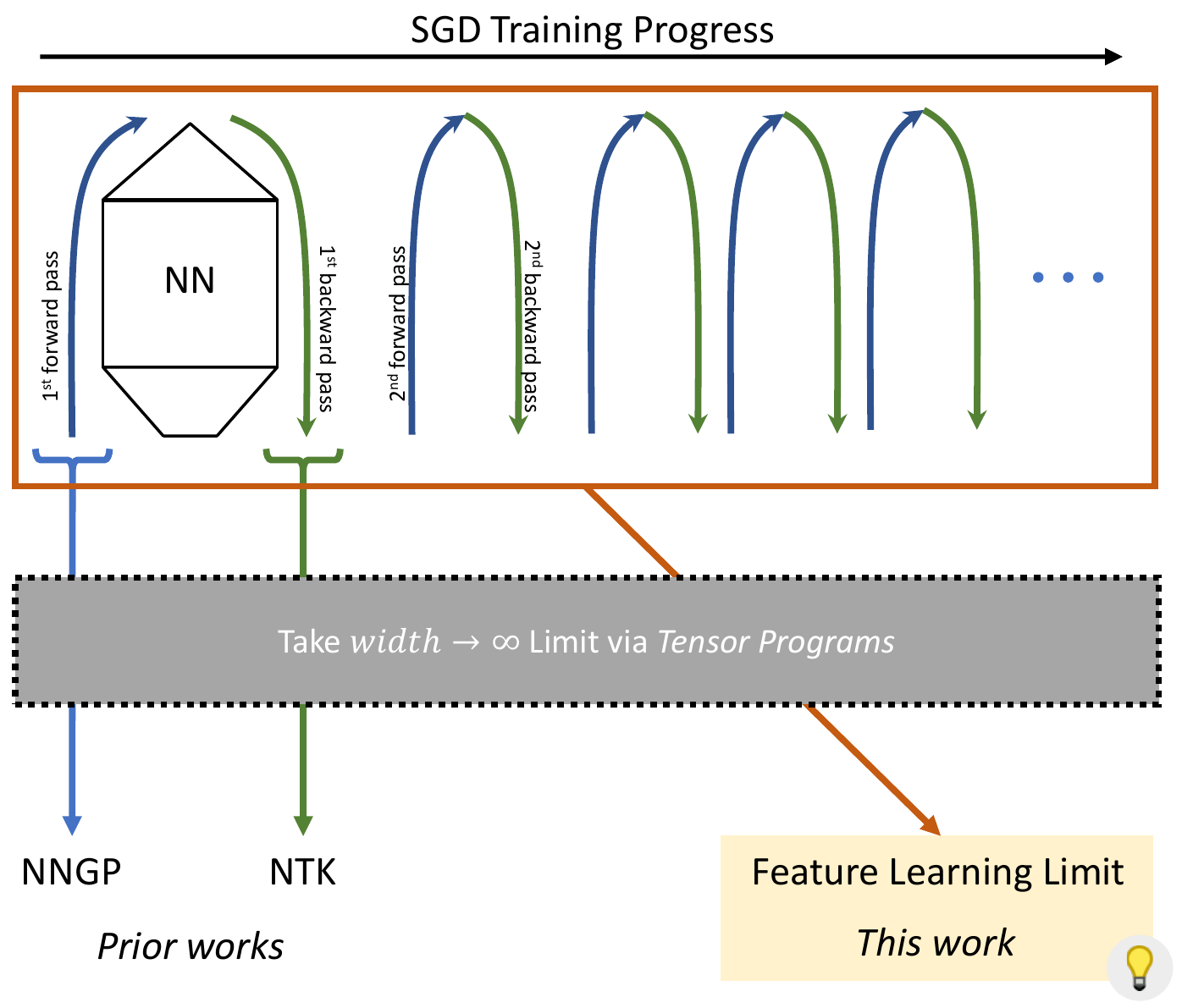}
    \end{center}
    \vspace{-100pt}
\end{wrapfigure}
In \cref{sec:TensorPrograms} and \cref{sec:ProgramConstruction},
we describe the \emph{Tensor Programs technique} for deriving (rigorously) the infinite-width training dynamics of any abc-parametrization.
The main insight of this approach is:
\begin{center}
    \it
    \parbox{0.9\linewidth}{
        When width is large, every activation vector has roughly iid coordinates, at \textbf{any time} during training.
        Using \emph{Tensor Programs}, we can recursively calculate such coordinate distributions, and consequently understand how the neural network function evolves.
    }
\end{center}
The Tensor Programs technique was developed in a series of papers \cite{scaling,TP1,TP2,TP3} that proved the architectural universality of the Neural Network-Gaussian Process (NNGP) Correspondence and the Neural Tangent Kernel (NTK) limits and showed how to compute the corresponding infinite-width kernels. 
In the \textcolor{purple}{\textbf{Figure above}}, the NNGP kernel can be thought of as the ``limit'' of the first forward pass of a randomly initialized model; the NTK can be similarly thought of as the ``limit'' of its first backward pass.
The mechanics of calculating such limits is 1) to write down the relevant neural network computation (e.g.\ the first forward pass in the NNGP case) as a principled composition of matrix multiplication and coordinatewise nonlinearities, called a \emph{Tensor Program}, and 2) to recursively calculate the distribution of coordinates of each vector via what's called the \emph{Master Theorem}.
In this paper, we follow the exact same recipe, where in 1) we just write down the \emph{entire SGD training} instead of only the first step.
More generally,
\begin{center}
    \it
    \parbox{0.9\linewidth}{\centering
        To derive the infinite-width limit of \textbf{any} neural computation (e.g.\ SGD training),\\
        1) express it as a Tensor Program, and 2) mechanically apply the Master Theorem.
    }
\end{center}

For example, we easily recover the (discrete-time) 1-hidden-layer mean field limit (\cref{thm:1LPMUPLimit}).
It readily applies to practically any neural architecture (e.g.\ ResNet and Transformers)%
\footnote{e.g.\ by extending the example programs of \citep{TP1,TP2}, which express only the first forward and backward passes, into the entire training computation.}
as well as many common variants of SGD;
however, in this paper, for pedagogical clarity, we only focus on multilayer perceptrons.
The generality of our approach allows us to easily adapt to settings outside the traditional (CIFAR10-style) supervised classification, such as the Word2Vec and few-shot learning tasks in this paper, or reinforcement learning and image generation outside of our scope.

\paragraph*{Our Contributions}
\begin{enumerate}
\item Formulate a natural space of NN parametrizations (\emph{abc-parametrizations}).
\item Prove \emph{Dynamical Dichotomy}: Any nontrivial stable abc-parametrization yields either feature learning or kernel limits, but not both. 
\item Show both NTK and standard parametrizations yield kernel limits and propose the \emph{Maximal Update Parametrization ($\mu$P)} , which admits maximal feature learning in a suitable sense.
\item Use Tensor Programs to derive the infinite-width limit of $\mu$P and, more generally, the limit of any abc-parametrization. We verify our theory using extensive experiments.
\item Show the $\mu$P limit outperforms both NNGP/NTK baselines and finite networks on 1) Word2Vec and 2) Omniglot few-shot learning, trained via first-order MAML.
\end{enumerate}

\paragraph{\emph{Tensor Programs} Series}
While this work is self-contained, it is positioned as the 4th paper in the series, following \citet{TP1,TP2,TP3}.
We do not extend the Tensor Programs machinery further here, but instead extract the first major payoff of the foundation laid in the earlier works.
In fact, this paper is the original motivation for this series; for a short history, see \cref{sec:history}.

\section{Related Works}

\paragraph*{Comparison with Mean Field Limits}

For 1-hidden-layer MLP, the mean field limit \citep{sirignano_mean_2018,mei_mean_2018,chizat_global_2018,rotskoff_neural_2018} is equivalent to the $\mu$P limit modulo the symmetry of \cref{stmt:SGDsymmetry} (see \cref{sec:motivating}).
Several works also proposed different versions of mean field frameworks for deeper MLPs \citep{araujo_mean-field_2019,sirignano_mean_2020,nguyen_nonrigorous_meanfieldMLP_2019,nguyen_rigorous_meanfieldMLP_2020,fang_modelingfromfeatures_2020}.
However, they did not consider the typical Gaussian $\Gaus(0, 1/n)$ random initialization (or the appropriately rescaled version in their respective parametrizations)%
\footnote{
In fact, empirically we observe such Gaussian random initialization to be crucial to performance compared to the mean-field-style initialization in this literature.
}, which has a Central-Limit effect as opposed to a Law-of-Large-Numbers effect.
For example, \citep{araujo_mean-field_2019,nguyen_rigorous_meanfieldMLP_2020} can cover the case of $\Gaus(0, 1/n^2)$, instead of $\Gaus(0, 1/n)$, initialization, which in fact causes the function to be stuck at initialization.
See \cref{sec:comparisonwithdeepmeanfield} for more explanations.
Of these works, the mean field limit of \citep{fang_modelingfromfeatures_2020} has the form most similar to what we derive here.
There, as we do here, the coordinate distribution of each (pre)activation vector is tracked recursively.
The main difference is, while \citep{fang_modelingfromfeatures_2020} has an atypical initialization involving $\ell_2$ regression, we consider the usual Gaussian $\Gaus(0, 1/n)$ scheme.
Such a (size $n\times n$) Gaussian matrix in the middle of the network has a distinctly different effect, more similar to that of a Gaussian matrix in the usual NNGP/NTK calculation,%
\footnote{Actually, it is more similar to the Gaussian matrix in asymmetric message passing \citep{bayati_dynamics_2011} in that care must be taken to keep track of correlation between $W$ and $W^\trsp$.}
than the ``mean field'' matrices considered in \citep{fang_modelingfromfeatures_2020} and previous works \citep{araujo_mean-field_2019,sirignano_mean_2020,nguyen_nonrigorous_meanfieldMLP_2019,nguyen_rigorous_meanfieldMLP_2020}, which has an ``integral kernel'' effect that is the straightforward generalization of matrices to function spaces.
Nevertheless, discrete time versions of the 1-hidden-layer mean field limit and of many of the multilayer limits (such as \citep{nguyen_rigorous_meanfieldMLP_2020,fang_modelingfromfeatures_2020}) can be derived directly by writing the corresponding initialization and training inside a Tensor Program and applying the Master Theorem (\cref{thm:PLNetsorT+MasterTheorem}).

\paragraph{Discrete- vs Continuous-Time Gradient Descent}
At a high level, there are two natural limits of neural networks training dynamics: large-width and continuous-time.
Most prior works on infinite-width limits of neural networks 
also took the continuous-time limit simultaneously, e.g. \citep{jacot_neural_2018,mei_mean_2018,chizat_global_2018,rotskoff_neural_2018,sirignano_mean_2018}.
In contrast, here we only take the large width limit, so that gradient descent stays discrete-time.
Then the results of these prior works can be recovered by taking another continuous-time limit.
From a practical perspective, the continuous-time limit is often unnatural, e.g.\ 1) because the step size is usually as large as possible to speed up training, 2) because of the task (such as reinforcement learning), or 3) because of the importance of hyperparameters like batch size that are hidden away in such limits.
On the theory side, taking the continuous-time limit can create issues with 1) well-posedness and 2) existence and uniqueness of the resulting ODE/PDE.
While they can sometimes be proved to hold, they are artifacts of the continuous-time limit, as the corresponding questions for the discrete time evolution are trivial, and thus not relevant to the behavior of real networks.

\paragraph{Technical Assumptions}
Earlier works on neural tangent or mean field limits (e.g. \citep{jacot_neural_2018,mei_mean_2018,chizat_global_2018,rotskoff_neural_2018,sirignano_mean_2018,nguyen_rigorous_meanfieldMLP_2020,fang_modelingfromfeatures_2020}) assume various forms of regularity conditions, such as 1) 0th, 1st, and/or 2nd order smoothness on the nonlinearity or other related functions, and 2) the support boundedness, subgaussianity, and/or PDF smoothness of initialization distributions.
These are often either unnatural or difficult to check.
In our work, the only assumption needed to rigorously obtain the infinite-width limit is that the nonlinearity $\phi$ has a polynomially bounded weak 2nd derivative and that the loss function has a continuous derivative w.r.t.\ the prediction (\cref{assm:phismooth}).
In particular, when we specialize to the 1-hidden-layer case and derive the discrete time version of the mean field limit, we cover the standard Gaussian initialization;
in fact, we can allow any heavy-tailed initialization that can be written as the image of a Gaussian under a pseudo-Lipschitz function, which include nonsmooth PDFs and singular distributions.%
\footnote{We won't expand further here, but it can be derived straightforwardly from the Master Theorem (\cref{thm:PLNetsorT+MasterTheorem}).}
This generosity of technical assumptions is due to that of the Tensor Programs Master Theorems proven in \citep{TP1,TP2,TP3}.

\paragraph{Training Time}
Many prior works (e.g.\ \citep{mei_mean_2018,allen-zhu_convergence_2018,huang2019dynamics}) derived explicit time dependence of the convergence to infinite-width limit, so that a larger width can allow the network to stay close to the limit for longer.
In this paper, our results only concern training time independent of width, since our primary objective is to investigate the limit itself and its feature learning capabilities.
Moreover, recent evidence suggests that, given a fixed computational budget, it's always better to train a larger model for a shorter amount of time \citep{train_large_then_compress_2020}, which validates the practical relevance of our limit mode.
Nevertheless, it is possible to prove a quantitative version of the Tensor Programs Master Theorem, by which one can straightforwardly allow training time to increase with width.

\paragraph{Classification of Parametrizations}
\citep{chizat_note_nodate} pointed out that the weights move very little in the NTK limit, so that linearization approximately holds around the initial parameters, in contrast to the mean field limit (for 1-hidden-layer networks) where the weights move substantially.
For this reason, they called the former ``lazy training'' and the latter ``active training,'' which are classified nonrigorously by a multiplicative scaling factor of the logit (similar to $n^{-a_{L+1}}$ in this paper).
While these terms are not formally defined, they intuitively correspond to the kernel and feature learning regimes in our paper.
From a different perspective, \citep{mei_mean-field-kernel_2019} observed that the NTK and mean field limit can be thought of as short and long time-scale regimes of the mean field evolution equations.
Neither of the above works attempted to formally classify natural parametrizations of neural networks.
In contrast, \citep{woodworth_kernel_2020} studied a toy class of neural networks in the context of implicit regularization due to the scale $\alpha$ of initialization (which is closely related to logit multiplier of \citep{chizat_note_nodate} noted above).
They identified the $\alpha \to \infty$ limit (of the scale $\alpha$, not of width) with the ``kernel regime'' and the $\alpha \to 0$ limit with what they call the ``rich regime''.
They showed that the former is implicitly minimizing an $\ell_2$ risk while the latter, an $\ell_1$ risk.
They claim width allows the toy model to enter the kernel regime more naturally, but as we see in this work, both kernel and feature learning regimes are admissible in the large width limit of a standard MLP.
Closer to our approach, \citep{golikov_dynamically_2020} studied what amounts to a 2-dimensional subspace of the space of stable abc-parametrizations for $L=1$.
They proposed a notion of stability which is similar to the combination of stability and nontriviality in this paper.
They characterized when the Neural Tangent Kernel, suitably generalized to any parametrization and playing a role similar to the feature kernel in this paper, evolves over time.
However, to simplify the proofs, they assumed that the gradients for the different weight matrices are estimated using different inputs, a very unnatural condition.
In contrast, here our results are for the usual SGD algorithm applied to MLPs of arbitrary depth.
In all of the above works and most of existing literature, not much attention is paid to the feature learning capabilities of neural networks in the right parametrization, as opposed to our focus here.
A notable exception is \citep{chizat_implicit_2020}, which showed that the mean field limit, but not the NTK limit, can learn low dimension linear structure of the input distribution resulting in ambient-dimension-independent generalization bounds.

\paragraph{Other Related Works}
\citep{lewkowycz_large_2020} proposed a toy model to study how large learning rate can induce a neural network to move out of the kernel regime in $\Omega(\log(width))$ time.
Since our dichotomy result only concerns training for $O(1)$ time (which, as we argue above, is more practically relevant), there is no contradiction.
\citep{sohl-dickstein_infinite_2020} also noted that standard parametrization leads to unstable training dynamics.
They then injected constants in the NTK parametrization, such as $\alpha/\sqrt n$ instead of $1/\sqrt n$ and tuned $\alpha$ in the resulting kernel.
\citep{aitchison_why_2020,aitchison_deep_2020} also observed the lack of feature learning in NNGP and NTK limits but, in contrast to taking the exact limit of SGD training as we do here, they proposed a deep kernel process as a way of loosely mimicking feature learning in finite-width networks.
\citep{gilboa_wider_2019} empirically observed that wider networks achieve better downstream performance with linear transfer learning, even though on the original pretraining task there can be little difference.
We fix the input dimension $d$ in this work, but one can also consider varying $d$ with width $n$, e.g. \citep{Oymak_2020,nguyen2020global}.
\citep{yuanzhi_beyond_NTK} proved a complexity separation between NTK and finite-width networks by showing the latter approximates a sort of infinite-width feature learning network.
In the literature surrounding NTK, often there are subtle differences in parametrization leading to subtle differences in conclusion (e.g.\ \citep{allen-zhu_convergence_2018,zou_stochastic_2018,du_gradient_2018}). Our abc framework encapsulates all such parametrizations, and can easily tell when two ostensibly different parametrizations (e.g.\ \citep{zou_stochastic_2018,du_gradient_2018}) are actually equivalent or when they are really different (e.g.\ \citep{allen-zhu_convergence_2018,du_gradient_2018}) via \cref{stmt:SGDsymmetry}.

\section{Feature Learning vs Kernel Behavior}
\label{sec:kernelvsfeature}

In this section, we give a characterization of training procedures that induce feature learning vs kernel behavior; we will elaborate on what we mean by these two kinds of behavior below. We first motivate this discussion by reviewing the well-known tangent kernel and mean field limits of a shallow neural network. 

\subsection{Motivating Examples: Neural Tangent Kernel and Mean Field Limits}
\label{sec:motivating}
For simplicity, define a shallow network $f(\xi)$ with input/output dimension 1 by 
\begin{equation}
f(\xi)=Vx(\xi)\in\R,\quad x(\xi)=\phi(h(\xi))\in\R^{n},\quad h(\xi)=U\xi\in\R^{n}.
\label{eqn:1LPIntro}
\end{equation}
As a specialization of \cref{eqn:MLP}, we parametrize weights $V=n^{-a_{v}}v\in\R^{1\times n}$ and $U=n^{-a_{u}}u\in\R^{n\times 1}$, where the width $n$ should be thought of as tending to $\infty$, and $v,u$ should be thought of as the actual trainable parameters. We will sample $v_{\alpha}\sim\Gaus(0,n^{-2b_{v}}),u_{\alpha}\sim\Gaus(0,n^{-2b_{u}})$ for $\alpha\in[n]$.
The learning rate is $\eta n^{-c}$ for some $\eta$ independent of $n$.

For example, in the \emph{Neural Tangent Parametrization} (abbreviated \emph{NTP}) \citep{jacot_neural_2018}, $a_{u}=b_{v}=b_{u}=0$, $a_{v}=1/2$, $c=0$.
The \emph{Mean Field Parametrization} (abbreviated \emph{MFP}) corresponds to $a_v = 1$, $a_{u} = b_{u}=b_v=0$, $c=-1$; however, as will be explained shortly, we will use the equivalent formulation $a_{u}=-1/2, a_v=b_{u}=b_v=1/2$, $c=0$ in this section so $c=0$ for both NTP and MFP.
We remark that the GP limit, i.e.\ training only the last layer of a infinite-wide, randomly initialized network, is a special case of the NTK limit where the first layer is not trained. Everything we discuss below about the NTK limit specializes to the GP limit appropriately.

Given an input $\xi,$ the gradient of $f$ can be calculated as 
\begin{align*}
dx(\xi) =V,\quad
dh(\xi) =dx(\xi)\odot\phi'(h(\xi)),\quad
dv(\xi) =n^{-a_{v}}x(\xi),\quad
du(\xi) =n^{-a_{u}}dh(\xi)\xi
\end{align*}
where $d\bullet(\xi)$ is shorthand for $\nabla_{\bullet}f(\xi)$ (however, note that later in \cref{sec:exactformulas}, $d\bullet(\xi)$ will stand for $n\nabla_{\bullet}f(\xi)$). For loss function $\loss:\R\times\R\to\R$, the loss gradient on a pair $(\xi,y)$ is then given by $\loss'(f(\xi),y)[dv(\xi),du(\xi)]$ (where $\loss'$ denotes derivative in first argument).

Note that one can keep the function $f$ invariant while changing the magnitude of the gradient $dv$ by changing $a_{v},b_{v}$, holding $a_{v}+b_{v}$ constant; likewise for $du$.
Thus, the trajectory of $f$ stays fixed if, for any $\theta \in \R$, we set $a_{u}\gets a_{u} +\theta, a_v \gets a_v +\theta, b_{u} \gets b_{u} - \theta, b_v \gets b_v - \theta, c \gets c-2\theta$ (also see \cref{stmt:SGDsymmetry}).
With $\theta=-1/2$, this explains why the two formulations of MFP above are equivalent.
Then, for both NTP and MFP, we will consider the dynamics of $f$ trained under stochastic gradient descent with learning rate $\eta=1$ and batch size 1, where the network is fed the pair $(\xi_{t},y_{t})$ at time $t$, starting with $t=0$. This simplicity is intended to intuitively illustrate our points below, but we shall state formal results regarding more common settings in \cref{sec:abcDichotomyMain}.

\vspace{-0.3em}

\paragraph*{Notation and Setup}
Below, when we say a (random) vector $v \in \R^n$ has \emph{coordinate size $O(n^a)$} (written $v=O(n^a)$),%
\footnote{Contrast this with a common semantics of $v=O(n^a)$ as $\|v\|=O(n^a)$.}
we mean $\sqrt{\|v\|^2/n} = O(n^a)$ with high probability for large $n$.
Intuitively, this means that each coordinate has a typical fluctuation of $O(n^a)$.
Likewise if $O(n^a)$ is replaced with $\Theta(n^a)$ or $\Omega(n^a)$.
See \cref{defn:BigO} for a formal definition.

Let $f_{t},h_{t},x_{t},U_{t},V_{t},dx_{t},dh_{t},dv_{t},du_{t}$ denote the corresponding objects at time $t$, with $t=0$ corresponding to random initialization. We also abuse notation and write $x_{t}=x_{t}(\xi_{t})$, i.e.\ applying the function $x_{t}$ specifically to $t$th input $\xi_{t}$; similarly for $f_{t},h_{t},dx_{t},dh_{t},dv_{t},du_{t}$. These symbols will never appear by themselves to denote the corresponding function, so this should cause no confusion. Then SGD effectively updates $U$ and $V$ by 
\begin{align*}
U_{t+1} =U_{t}-\chi_{t}n^{-a_{u}}du_{t},\quad
V_{t+1} =V_{t}-\chi_{t}n^{-a_{v}}dv_{t}.
\end{align*}
where $\chi_{t}\defeq \loss'(f_{t},y_{t})$. Finally, let $\Delta\bullet_{t}\defeq\bullet_{t}-\bullet_{0}$, for all $\bullet\in\{f,h,x,U,V,dx,dh,dv,du\}$. For example, after 1 SGD update, we have, for any $\xi\in\R$, 
\begin{align*}
\Delta h_{1}(\xi) & =h_{1}(\xi)-h_{0}(\xi)=-n^{-a_{u}}\chi_{0}\xi du_{0}=-n^{-2a_{u}}\chi_{0}\xi_{0}\xi dh_{0}\\
    &=
        -n^{-2a_{u}}\chi_{0}\xi_{0}\xi dx_{0}\odot\phi'(h_{0})
        \numberthis\label{eq:Deltah}\\
\Delta f_{1}(\xi) & =V_{0}\Delta x_{1}(\xi)+\Delta V_{1}x_{1}(\xi)=V_{0}\Delta x_{1}(\xi)-n^{-a_{v}}dv_{0}^{\trsp}x_{1}(\xi)\\
    & =V_{0}\Delta x_{1}(\xi)-n^{-2a_{v}}x_{0}^{\trsp}x_{1}(\xi)
    \numberthis\label{eq:deltaf}
\end{align*}

\subsubsection{Key Observations}

Let's list a few characteristics of the NTK and MF limits in the context of the shallow network in \cref{eqn:1LPIntro}, and then discuss them in the general setting of deep MLP. 
We will keep our discussion intuitive to carry across the key ideas.

\paragraph*{Feature Evolution}

For a generic $\xi\in\R$, its embedding vector $x_{0}(\xi)$ has coordinates of $\Theta(1)$ size in both NTP and MFP. 
However, for any $t\ge1$ independent of $n$, $\Delta x_{t}(\xi)$ generically has coordinate size $\Theta(1/\sqrt{n})$ in NTP but $\Theta(1)$ in MFP.

\emph{Example for $t=1$}: By \cref{eq:Deltah}, we have 
\[\Delta h_{1}(\xi)=n^{-2a_{u}}\chi_{0}\xi_{0}\xi dx_{0}\odot\phi'(h_{0}).\]
Plug in $a_{u}=0$ for NTP.
Observe that $\xi_{0},\xi,\chi_{0}=\Theta(1)$,%
\footnote{$\chi_0 = \loss'(f_0, y_0) = \Theta(1)$ because $f_0$ has variance $\Theta(1)$.}
so
\begin{equation}\Delta h_{1}(\xi)=\Theta(1) \cdot dx_{0}\odot\phi'(h_{0}).\tag{in NTP}
\end{equation}
In addition, $\phi'(h_{0})=\Theta(1)$ because $h_{0}=\Theta(1)$, so
\begin{equation}\Delta h_{1}(\xi)=\Theta(1) \cdot dx_{0}\odot \Theta(1).\tag{in NTP}\end{equation}
Finally, $dx_{0}=V_{0}=\Theta(1/\sqrt{n})$ in NTP. Altogether, this implies
\begin{gather}
    \Delta h_{1}(\xi)
        =
            \Theta(1/\sqrt{n})\nonumber\\
\implies \Delta x_{1}(\xi)
        \approx
            \phi'(h_{0}(\xi))\odot\Delta h_{1}(\xi) =\Theta(1/\sqrt n)\to 0,\quad \text{as $n\to\infty$}.
            \tag{in NTP}
\end{gather}
On the other hand, in MFP, the only thing different is $a_{u}=-1/2$ and $dx_{0}=\Theta(1/n)$, which implies
\begin{equation}
    \Delta h_{1}(\xi)=\Theta(n) \cdot \Theta(1/n) \odot \Theta(1) = \Theta(1)
    \implies \Delta x_{1}(\xi) = \Theta(1).\tag{in MFP}
\end{equation}

\paragraph*{Feature Kernel Evolution}

Therefore the \emph{feature kernel }$F_{t}(\xi,\zeta)\defeq x_{t}(\xi)^\trsp x_{t}(\zeta)/n$ does not change in the NTK limit but it does in the MF limit%
, i.e.\ for any fixed $t\ge1$,%
\footnote{here the limit should be construed as almost sure limits; see \cref{thm:PLNetsorT+MasterTheorem}.}
\begin{align*}
\lim_{n\to\infty}F_{t}(\xi,\zeta) & =\lim_{n\to\infty}F_{0}(\xi,\zeta),\quad\text{in NTP, but}\\
\lim_{n\to\infty}F_{t}(\xi,\zeta) & \ne\lim_{n\to\infty}F_{0}(\xi,\zeta),\quad\text{in MFP, in general.}
\end{align*}
Indeed, regardless of parametrization, we have
\[F_{t}(\xi,\zeta)=\frac{1}{n}\left[x_{0}(\xi)^{\trsp}x_{0}(\zeta)+\Delta x_{t}(\xi)^{\trsp}x_{0}(\zeta)+x_{0}(\xi)^{\trsp}\Delta x_{t}(\zeta)+\Delta x_{t}(\xi)^{\trsp}\Delta x_{t}(\zeta)\right].\]
In NTP, because $\Delta x_t(\xi) = \Theta(1/\sqrt n)$ as noted above,
\[\frac{1}{n}\Delta x_{t}(\xi)^{\trsp}x_{0}(\zeta)=\frac{1}{n}\sum_{\alpha=1}^{n}\Delta x_{t}(\xi)_{\alpha}x_{0}(\zeta)_{\alpha}=\frac{1}{n}\sum_{\alpha=1}^{n}O(n^{-1/2})=O(n^{-1/2}),\]
and likewise the other terms involving $\Delta x_{t}$ will vanish as $n\to\infty$. But in MFP, $\Delta x_{t}(\xi)=\Theta(1)$ will in general be correlated with $x_{0}(\zeta)$ such that $\frac{1}{n}\Delta x_{t}(\xi)^{\trsp}x_{0}(\zeta)=\frac{1}{n}\sum_{\alpha=1}^{n}\Theta(1) =\Theta(1)$.

It may seem somewhat puzzling how the NTK limit induces change in $f$ without feature or feature kernel evolution.
We give some intuition in \cref{sec:appendixmotivating}.

\paragraph*{Pretraining and Transfer Learning}

The simple fact above about the feature kernel $K$ implies that the NTK limit is unable to perform linear transfer learning.
By \emph{linear transfer learning, }we mean the popular style of transfer learning where one discards the pretrained linear classifier layer and train a new one on top of the features (e.g.\  $x$ in our example), which are fixed. Indeed, this is a linear problem and thus only depends on the kernel of the features. If this kernel is the same as the kernel at initialization, then the pretraining phase has had no effect on the outcome of this ``transfer'' learning.

In fact, a more sophisticated reasoning shows pretraining in the NTK limit is no better than random initialization for transfer learning even if finetuning is performed to the whole network, not just the classifier layer.
This remains true if we replace the linear classifier layer by a new deep neural network.
See \cref{{rem:kerneltrivializes},thm:kerneltrivializes}.
The Word2Vec experiment we do in this paper is a linear transfer task.

In some other settings, such as some settings of metalearning, like the few-shot learning task in this paper, the last layer of the pretrained network is not discarded.
This is called \emph{adaptation}.
Then the NTK limit does not automatically trivialize transfer learning.
However, as will be seen in our experiments, the NTK limit still vastly underperforms the feature learning limit, which is exemplified by the MF limit here.

\paragraph{Kernel Gradient Descent in Function Space}
In NTP, as $n\to\infty$, $\langle \nabla_{U,V} f_0(\xi), \nabla_{U,V} f_0(\zeta) \rangle$ converges to some deterministic value $K(\xi, \zeta)$ such that $K$ forms a kernel (the NTK).
Then, in this limit, if the learning rate is $\eta$, the function $f$ evolves according to kernel gradient descent $f_{t+1}(\xi) = f_t(\xi) - \eta K(\xi, \xi_t) \chi_t$.
However, this shouldn't be the case for the MF limit.
For example, if $\phi$ is identity, then intuitively $f_{t+1}(\xi) - f_t(\xi)$ should be quadratic in $\eta$, not linear, because two layers are updated at the same time.

\subsection{abc-Parametrizations and Dynamical Dichotomy}
\label{sec:abcDichotomyMain}
In this section, we broaden our scope to the abc-parametrizations of deeper MLPs, defined by \cref{eqn:MLP}, and their infinite-width limits.
In \cref{tab:summary}, we summarize the $\{a_l, b_l\}_l \cup \{c\}$ values of various abc-parametrizations in the literature.

\begin{assm}\label{assm:nonlin}
    Our main results in this section (and this section only) will assume $\phi$ is either tanh or a smooth version of relu called $\sigma$-gelu (see \cref{defn:gelu}), for sufficiently small $\sigma>0$ (which means $\sigma$-gelu approximates relu arbitrarily well).
\end{assm}
Note this assumption is only needed for the classification of abc-parametrizations.
For deriving the infinite-width limits, the much weaker \cref{assm:phismooth} suffices.
We believe our results here will hold for generic nonlinearities, but making this precise is outside our scope.
(See \cref{rem:nonlinspecificity} for an overview on how \cref{assm:nonlin} is used).

\begin{table}
    \caption{We summarize the abc values of SP (standard), NTP (Neural Tangent), MFP (Mean Field, for 1-hidden-layer nets), $\mu$P (Maximal Update, ours). We show the minimal value of $c$ such that the parametrization is stable (\cref{defn:stability}). We also list the quantities $r,2a_{L+1}+c,a_{L+1}+b_{L+1}+r$ involved in stability, feature learning, and kernel regime properties of the parametrizations.
    Here we only focus on scaling with $n$ and ignore dependence on input dimension.
    Recall the MLP definition:
    }
    \label{tab:summary}
    \[h^{1}=W^{1}\xi\in\R^{n}, x^{l}=\phi(h^{l})\in\R^{n},h^{l+1}=W^{l+1}x^{l}\in\R^{n},f(\xi)=W^{L+1}x^{L}\]
    \begin{tabular}{cccccc}
        \toprule 
         & Definition & SP (w/ LR $\f 1 n$) & NTP & MFP ($L=1$) & $\mu$P (ours)\tabularnewline
        \midrule
        \midrule 
        $a_{l}$ & $W^{l}=n^{-a_{l}}w^{l}$ & 0 & $\begin{cases}
        0 & l=1\\
        \nicefrac 1 2 & l\ge2
        \end{cases}$ & $\begin{cases}
        0 & l=1\\
        1 & l=2
        \end{cases}$ & $\begin{cases}
        -\nicefrac 1 2 & l=1\\
        0 & 2\le l\le L\\
        \nicefrac 1 2 & l=L+1
        \end{cases}$\tabularnewline
        $b_{l}$ & $w_{\alpha\beta}^{l}\sim\Gaus(0,n^{-2b_{l}})$ & $\begin{cases}
        0 & l=1\\
        \nicefrac 1 2 & l\ge2
        \end{cases}$ & 0 & $0$ & $\nicefrac 1 2$\tabularnewline
        $c$ & $LR=\eta n^{-c}$ & $1$ & $0$ & $-1$ & $0$\tabularnewline
        $r$ & \cref{def:rMain} & $\nicefrac 1 2$ & $\nicefrac 1 2$ & $0$ & $0$\tabularnewline
        \midrule 
        \multicolumn{2}{l}{$2a_{L+1}+c$} & 1 & 1 & 1 & 1\tabularnewline
        \multicolumn{2}{l}{$a_{L+1}+b_{L+1}+r$} & 1 & 1 & 1 & 1\tabularnewline
        \multicolumn{2}{l}{Nontrivial?} & $\checkmark$ & $\checkmark$ & $\checkmark$ & $\checkmark$\tabularnewline
        \multicolumn{2}{l}{Stable?} & $\checkmark$ & $\checkmark$ & $\checkmark$ & $\checkmark$\tabularnewline
        \multicolumn{2}{l}{Feature Learning?} &  &  & $\checkmark$ & $\checkmark$\tabularnewline
        \multicolumn{2}{l}{Kernel Regime?} & $\checkmark$ & $\checkmark$ &  & \tabularnewline
        \bottomrule
    \end{tabular}
\end{table}

\paragraph*{Symmetries of abc-Parametrizations}

As above, we can scale the parameter gradients $\nabla_{w^{l}}f$ arbitrarily while keeping $f$ fixed, if we vary $a_{l},b_{l}$ while fixing $a_{l}+b_{l}$: $\nabla_{w^{l}}f$ is scaled by $n^{-\theta}$ if $a_{l}\gets a_{l}+\theta,b_{l}\gets b_{l}-\theta$. 
In other words, changing $a_l, b_l$ this way effectively gives $w^l$ a per-layer learning rate.
If we apply this gradient with learning rate $\eta n^{-c}$, then the change in $W^{l}$ is scaled by $\eta n^{-c-2\theta}$. Consequently, if $c\gets c-2\theta$, then $W^{l}$ is not affected by the change in $a_{l},b_{l}$. In summary,
\begin{equation}
    \forall \theta \in \R: \text{\emph{$f_t(\xi)$ stays fixed for all $t$ and $\xi$ if we set $a_{l}\gets a_{l}+\theta,\ b_{l}\gets b_{l}-\theta,\ c\gets c-2\theta$.}}
    \label{stmt:SGDsymmetry}
\end{equation}

\paragraph{Stable abc-Parametrizations}

We will only consider abc-parametrizations such that, as $n\to\infty$, 1) the preactivations $\{h^{l}\}_{l}$ and activations $\{x^{l}\}_{l}$ have $\Theta(1)$ coordinates at initialization, and 2) their coordinates and the logit $f(\xi)$ all stay $O(1)$ throughout the course of SGD.%
\footnote{but they may depend on training time and $\eta$; in particular, it's possible that they diverge with time.}
Otherwise, they tend to $\infty$ with $n$, eventually going out of floating point range. Indeed, this is an acute and real problem common in modern deep learning, where float16 is necessary to train large models.
We call any such parametrization \emph{stable} (see \cref{defn:stability} for a formal definition).
Thus unstable parametrizations are of no practical interest.

It turns out stable abc-parametrizations can be characterized by a set of inequalities on $\{a_{l},b_{l}\}_{l}\cup\{c\}$ (so that the stable ones form a polyhedron).
To present these inequalities succinctly, it's useful to define
\begin{defn}
    \label{def:rMain}For any abc-parametrization, we write $r$ for the quantity
    \[
    r\defeq\min(a_{L+1}+b_{L+1},2a_{L+1}+c)+c-1+\min_{l=1}^{L}\left[2a_{l}+\ind(l=1)\right].
    \]
\end{defn}
For example, in NTP, $r=1/2$, while in MFP (when $L=1$), $r=0$. Intuitively, $r$ is the exponent such that $\Delta x_{t}^{L}(\xi)=\Theta(n^{-r})$. Thus, to avoid activation blowup, we want $r\ge0$; to perform feature learning, we want $r=0$.
\begin{thm}[Stability Characterization, c.f.\ \cref{thm:stabilityconditions}]
    \label{thm:stabilityconditionsMain}An abc-parametrization is stable iff all of the following are true (with intuitions in parentheses):
    
    \begin{enumerate}
    \item ((pre)activations $x^l_0, h^l_0$ at initialization are $\Theta(1)$ and logits $f_0$ are $O(1)$)\label{item:actlogitinitMain} 
    \begin{equation}
    a_{1}+b_{1}=0;\quad a_{l}+b_{l}=1/2,\ \forall l\in[2,L];\quad a_{L+1}+b_{L+1}\ge1/2.\label{eq:actlogitinitMain}
    \end{equation}
    \item (features don't blowup, i.e. $\Delta x_{t}^{l}=O(1)$ for all $l$)
    \begin{equation}
    r\ge0.\label{eqn:DeltaWNotTooBigMain}
    \end{equation}
    \item (logits don't blow up during training, i.e. $\Delta W_{t}^{L+1}x_{t}^{L},W_{0}^{L+1}\Delta x_{t}^{L}=O(1)$) \label{item:logitblowuptrainMain} 
    \begin{equation}
    2a_{L+1}+c\ge1;\quad a_{L+1}+b_{L+1}+r\ge1.\label{eq:logitblowuptrainMain}
    \end{equation}
    \end{enumerate}    
\end{thm}

\paragraph{Nontrivial abc-Parametrizations}
Among stable abc-parametrizations, there are also those where $f$ does not change throughout training in the infinite-width limit.
We say such parametrizations are \emph{trivial}.
Our dichotomy result will only apply to nontrivial stable abc-parametrizations.%
\footnote{In particular, it's possible for the function $f$ to stay fixed with time, but for the features to change.}

Nontrivial abc-parametrizations can also be described by a disjunction of equations on $\{a_{l},b_{l}\}_{l}\cup\{c\}$ (geometrically, they correspond to the union of two faces on the polyhedron of stable abc-parametrizations).
\begin{thm}\label{thm:nontriviality}
    A stable abc-parametrization is nontrivial iff $a_{L+1}+b_{L+1}+r=1$ or $2a_{L+1}+c=1$.
\end{thm}

\paragraph{Feature Learning}
Below, for brevity, we say \emph{training routine} to mean the package of learning rate $\eta n^{-c}$, training sequence $\{(\xi_{t},y_{t})\}_{t\ge0}$,%
\footnote{For simplicity, we only consider batch size 1; it's straightforward to generalize to larger batch sizes.}
and a loss function $\loss(f(\xi),y)$ that is continuously differentiable in the prediction of the model $f(\xi)$.
As above, we use $\bullet_{t}$ to denote the object $\bullet$ after $t$ steps of SGD.
\begin{defn}[c.f.\ \cref{defn:featurelearning,{defn:featurekernelevolve}}]
    \label{defn:featurelearningMain}
    We say an abc-parametrization \emph{admits feature learning} (resp.\ \emph{evolves the feature kernel}) if, as $n\to\infty$, $\Delta x^L_t(\xi)$ has $\Omega(1)$ coordinates (resp.\ $\f 1 n (x^L_t(\xi)^\trsp x^L_t(\zeta) - x^L_0(\xi)^\trsp x^L_0(\zeta))=\Omega(1)$) for some training routine, time $t\ge1$, and input $\xi$ (resp.\ $\xi,\zeta$).%
    \footnote{For the sake of streamlining the main text presentation, we defined feature learning and feature kernel evolution slightly differently than in \cref{defn:featurelearning}, but ultimately they are equivalent as a result of our theorems.}%
\end{defn}
MFP, in the 1-hidden-layer case, is an example of feature learning parametrization.

Intuitively, feature kernel evolution implies feature learning, but \emph{a priori} it seems possible that the latter can occur without the former (akin to some kind of rotation of features).
If so, then, e.g.\ in terms of linear transfer learning, the pretraining ultimately had no benefit.
But, in fact,
\begin{thm}\label{thm:featurelearningMain}
    A nontrivial stable abc-parametrization admits feature learning iff it evolves the feature kernel iff $r=0$.
\end{thm}

\paragraph{Kernel Regime}
While feature learning here is defined by looking at the embedding of an input $\xi$, we can also look at the dynamics of the \emph{function} represented by the neural network.
\begin{defn}[c.f.\ \cref{defn:kernelregime}]
    \label{defn:kernelregimeMain}
    We say an abc-parametrization \emph{is in kernel regime} if there exists a positive semidefinite kernel $K$ such that, for any training routine, time $t\ge0$, and input $\xi$, in the $n\to\infty$ limit,
    \begin{equation}
        f_{t+1}(\xi)=f_{t}(\xi)-\eta K(\xi,\xi_{t})\loss'({f}_{t}(\xi_{t}),y_{t}),\quad \forall t\ge0.
        \label{eqn:kerneleqn}
    \end{equation}
    In other words, SGD reduces to kernel gradient descent in the large $n$ limit.
\end{defn}
\begin{thm}\label{thm:kernelregimeMain}
    A nontrivial stable abc-parametrization is in kernel regime iff $r>0$.
\end{thm}
NTP is a typical example of this, where $r=1/2$ and $K$ is given by the NTK.

\begin{wrapfigure}{r}{0.35\textwidth}
    \vspace{-27pt}
    \begin{center}
        \includegraphics[width=0.35\textwidth]{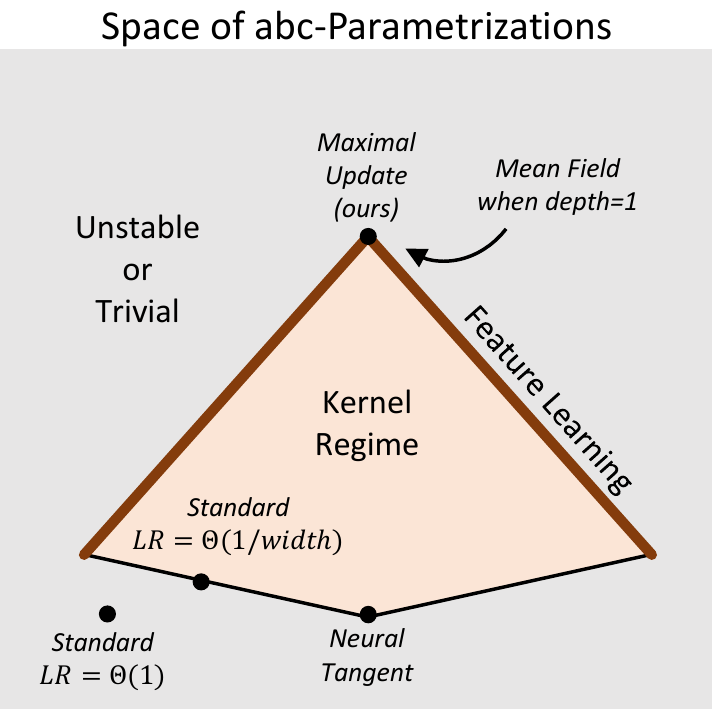}
    \end{center}
    \caption{\textbf{A Caricature of abc-Parametrizations.}
    The nontrivial stable parametrizations form a high dimensional polyhedron.
    Those on a part of its boundary admit feature learning, while all others are in kernel regime.
    $\mu$P is a vertex in the former, while NTP, latter.
    See \cref{fig:abcgeometry} for a more geometrically accurate depiction.
    }
    \label{fig:abcparamspace}
    \vspace{-40pt}
\end{wrapfigure}
\paragraph{Dynamical Dichotomy}
Since a stable abc-parametrization has either $r=0$ or $r>0$ by \cref{eqn:DeltaWNotTooBigMain}:
\begin{cor}\label{cor:dichotomyMain}
    A nontrivial stable abc-parametrization either admits feature learning or is in kernel regime, but not both.
\end{cor}

Note that \emph{kernel regime} (\cref{defn:kernelregimeMain}) is not defined as \emph{lack of feature learning}, so \cref{cor:dichotomyMain} is not a trivial statement.
In addition, \cref{assm:nonlin} is necessary.
For example, if $\phi$ is linear, then this dichotomy doesn't hold, as a 1-hidden-layer linear network where only the first layer is trained would both admit feature learning and is in kernel regime.

An interesting consequence of Dynamical Dichotomy is
\begin{cor}\label{cor:FLnoGP}
    Any nontrivial stable \emph{feature learning} abc-parametrization must have $\lim_{n\to\infty} f_0(\xi) = 0$ for all $\xi$, where the limit is almost sure.
\end{cor}

\cref{thm:featurelearningMain,{thm:kernelregimeMain},cor:FLnoGP} are consequences of the more general classification theorem \cref{thm:main}, which in addition shows: 1) feature learning in layer $l$ would imply the same for layers $l,\ldots,L$; 2) in any feature learning parametrization, $f_t$ in the large $n$ limit becomes deterministic, and thus is incompatible with any Bayesian perspective (in contrast to the NNGP limit).

Dynamical Dichotomy in the shallow perceptron case is illustrated by the NTK and MF limits, as presented in \cref{sec:motivating}, which shows the NTK limit exemplifies \cref{thm:kernelregimeMain} while the MF limit exemplifies \cref{thm:featurelearningMain}.
We present a simplified picture of abc-parametrizations in \cref{fig:abcparamspace},
but see \cref{fig:abcgeometry} for a more geometrically accurate depiction.

The paragraph above \cref{sec:Stability-at-Initialization} gives a quick outline of the proof of Dynamical Dichotomy, and the beginning of each succeeding section outlines the logic of that section.

\begin{rem}[Function Space Picture]
    A kernel regime limit resides solely in the \emph{function space picture}, i.e. the evolution of $f$ at any time being solely determined by the function values $\{\lim f_t(\zeta)\}_\zeta$ themselves (as opposed to the internal activations of $f$ as well) along with $\eta$, $\loss$, and $(\xi_t, y_t)$.
    Intuitively, this cannot be true for the feature learning limit, and therefore, at least informally, Dynamical Dichotomy is also a dichotomy over the sufficiency of the function space picture for determining the training evolution:
    We can construct two settings where $\{\lim f_t(\zeta)\}_\zeta$, $\eta$, $\loss$, and $(\xi_t, y_t)$ are the same but $f_{t+1}$ are different.
    1) The first setting is at $t=0$, where $\lim f_t(\zeta)=0$ for all input $\zeta$ by \cref{cor:FLnoGP}.
    Here a typical SGD will change $f$.
    2) In the second setting, suppose $\phi$ is relu.
    Design a sequence of inputs such that training the MLP on them with very large learning rate will make all relu neurons saturated in the 0 region.
    Then $f$ is everywhere 0, and an SGD step will not change that.

\end{rem}

\begin{rem}[Not All Dynamics are Infinite-Width Limits]
    \label{rem:invalidlimits}
    Accordingly, a nonlinear function space dynamics cannot be a valid infinite-width limit of some abc-parametrization.
    By \emph{nonlinear}, we mean $f_{t+1}(\xi) - f_t(\xi)$ is nonlinear in $\loss'({f}_{t}(\xi_{t}),y_{t})$.
    For example, any natural higher-order generalization of \cref{eqn:kerneleqn} (perhaps derived from a Taylor expansion at initialization) is not a valid limit.%
    \footnote{
        It may seem that Neural Tangent Hierarchy \citep{huang2019dynamics}, which allow some kind of higher order dynamics in the function space, violates our observation.
        But their infinite-width limit is identical to NTK in the constant time $t=O(1)$ regime, which is what \cref{rem:invalidlimits} (and this paper) concerns.
        Moreover, here we are talking about functional dynamics that doesn't depend on $n$ (because we are already at the $n\to\infty$ limit) whereas their functional dynamics does.
    }
\end{rem}

\paragraph{Pretraining and Transfer Learning}
As in the shallow examples, \cref{cor:dichotomyMain} says that any kernel regime parametrization (including NTP) trivializes pretraining and transfer learning%
\footnote{linear and nonlinear; see \cref{thm:kerneltrivializes}.}
in the infinite-width limit.

By calculating $r$ for the standard parametrization (SP), we can easily see that it cannot admit feature learning in the sense here without becoming unstable.
However, in the next section, we will manually analyze the training dynamics in an SP MLP to give an intuition why this is the case.
In turn, we then propose a simple modification of SP, the Maximal Update Parametrization (MUP or $\mu$P), which \emph{does} admit feature learning and, in fact, does so \emph{maximally} in a suitable sense.
In the pedagogical spirit, we will focus on the key insights and stress the right heuristics without dwelling on formal aspects.

\section{Standard Parametrization}
\label{sec:standardparam}

In this section, we give intuition for why gradient descent of neural network in standard parametrization (SP) will lead to logits blowup after 1 step, if the learning rate is $\omega(1/n)$, where $n$ is the width.
In addition, we will see why, with learning rate $O(1/n)$, SP is in kernel regime.
We first consider the simplest example and then state the general result at the end of the section.

To demonstrate the general principle in deep networks, it is necessary to consider the behavior of an $n\times n$ matrix in the middle of the network. Thus, the simplest case is a 2-hidden-layer linear MLP, i.e.\ \cref{eqn:MLP} with $L=2$ and $\phi=id$. The standard parametrization is given by 
\begin{equation}
a_{l}=0\ \forall l,\quad b_{1}=0,\quad b_{l}=1/2\ \forall l\ge2.\tag{SP}\label{eq:standardparam}
\end{equation}
We consider 1 step of SGD with learning rate $n^{-c}$ on a single data pair $(\xi,y)$. Then we can without ambiguity suppress explicit dependence on $\xi$ and write 
\begin{equation}
f=V\bar{h},\quad \bar{h}=Wh,\quad h=U\xi,\label{eq:standardparam2LP}
\end{equation}
where $U_{\alpha\beta}\sim\Gaus(0,1)$ and $W_{\alpha\beta},V_{\alpha\beta}\sim\Gaus(0,1/n)$ are the trainable parameters (simplifying the notation in \cref{sec:kernelvsfeature}).
As in \cref{{sec:kernelvsfeature}}, we use $\bullet_t$ to denote the quantity $\bullet$ after $t$ step of SGD.
Because we only focus on the 1st step of SGD, we lighten notation and write $\bullet = \bullet_0$.

\paragraph*{Initialization}

Since $U,W,V$ are independently sampled, a standard Central Limit argument would show that $h,\bar{h},f$ all have roughly iid Gaussian coordinates of variance $\Theta(1)$.

\paragraph*{First Gradient}

Now let's consider the gradients of $f$ on the data pair $(\xi,y)$, which are given by 
\begin{align}
d\bar{h} & =V^\trsp ,&dh&=W^\trsp d\bar{h},\nonumber \\
dV & =\bar{h},&dW&=d\bar{h}\ h^\trsp =V^\trsp h^\trsp ,& dU&=dh\ \xi^{\trsp}.\label{eq:backprop}
\end{align}
For simplicity, suppose we only update $W$ by learning rate $n^{-c}$ (and leave $U,V$ unchanged);
our conclusion will not change in the general case where we train all layers.
Then with $\chi$ denoting the loss derivative $\loss'(f,y)$, we can write 
\[
W_{1}=W-n^{-c}\chi\ dW.
\]
We shall show now that $c\ge1$ or else $f_{1}$ blows up with the width $n$ after this SGD step.

\paragraph*{After First SGD Step}

At $t=1$, $h_{1}=h$ since we did not update $U$, but 
\begin{align}
\bar{h}_{1} & =W_{1}h=\bar{h}-n^{-c}\chi\ dWh=\bar{h}-n^{-c}\chi \cdot V^\trsp h^\trsp h\label{eq:h2(1)}\\
f_{1} & =V\bar{h}_{1}=f-n^{-c}\chi\ VV^\trsp h^\trsp h.\label{eqn:SP2LPf}
\end{align}
Now, as noted above, $h$ has iid $\Theta(1)$ coordinates, so $h^\trsp h=\Theta(n)\in\R$. Similarly, $V\in\R^{1\times n}$ has Gaussian coordinates of variance $\Theta(1/n)$, so $VV^\trsp =\Theta(1)\in\R$. Finally, for typical loss function $\loss$ like MSE or cross entropy, $\chi=\loss'(f,y)$ is of order $\Theta(1)$ because $f$ fluctuates on the order $\Theta(1)$. Altogether, 
\[
f_{1}=f-\Theta(n^{1-c}).
\]
Therefore, for $f_{1}$ to remain $O(1)$, we must have $c\ge1$, i.e.\ the learning rate is $O(1/n)$.

\paragraph*{Kernel Regime and Lack of Feature Learning}

Consequently, the network cannot learn features in the large width limit if we would like the logits to not blow up.
Indeed, this version of SGD where only $W$ is updated can be seen to correspond to the limit where 
\[
a_{1}=\theta,\quad b_{1}=-\theta,\quad a_{2}=0,\quad b_{2}=1/2,\quad a_{3}=\theta,\quad b_{3}=-\theta+1/2,\quad \theta\to\infty.
\]
With $c=1$ as derived above, the parametrization is stable and nontrivial, as can be checked from \cref{thm:stabilityconditionsMain,{thm:nontriviality}}.
Then we get $r=1/2>0$, so by \cref{cor:dichotomyMain}, this parametrization is in kernel regime and does not admit feature learning.
We can also see this directly from \cref{eq:h2(1)}: from our calculations above, 
\[
\bar{h}_{1}-\bar{h}=O(n^{1-c})\ V^\trsp  = O(1)\ V^\trsp 
\]
whose coordinates have size $O(n^{-1/2})$ since $V$'s coordinates do, so there's no feature learning (at least in the first step).
Finally, from \cref{eqn:SP2LPf}, because $V V^\trsp  \to 1$ and $n^{-c} h^\trsp  h = n^{-1} h^\trsp  h \to \|\xi\|^2$, we get%
\footnote{Formally, these are almost sure convergences, but we suppress these details to emphasize on intuition.}
\[
    f_1 - f \to -\chi K(\xi, \xi)\defeq -\chi \|\xi\|^2,
\]
i.e.\ $f$ evolves by kernel gradient descent with the linear kernel.
Our derivations here only illustrate the first SGD step, but we can get the same conclusion from all steps of SGD similarly.

We summarize the general case below, which follows trivially from \cref{thm:stabilityconditionsMain,cor:dichotomyMain}.
\begin{thm}
An $L$-hidden-layer MLP in standard parametrization (see \cref{eq:standardparam} and \cref{tab:summary}) can only allow SGD learning rate of order $O(1/n)$ if we require $\lim_{n\to\infty}\EV f_t(\xi)^2 <\infty$ for all training routine, time $t$, and input $\xi$.
In this case, it is in kernel regime and does not admit feature learning.
\end{thm}

\section{Maximal Update Parametrization}
\label{sec:MUP}

As shown in the last section, the standard parametrization does not admit a feature learning infinite-width limit without blowing up logits. Here we propose simple modifications of the standard parametrization to make this possible while maintaining stability: 1) To enable feature learning, it suffices to divide the logits by $\sqrt{n}$ and use $\Theta(1)$ learning rate, i.e.\ set $a_{L+1}=1/2,c=0$ on top of \cref{eq:standardparam}; 2) to allow \emph{every layer} to perform feature learning, we should furthermore set $a_{1}=-1/2,b_{1}=1/2$.
We will see that this essentially means we update each weight matrix as much as possible without blowing up the logits or activations, so we call this the Maximal Update Parametrization (abbreviated MUP or $\mu$P).

\subsection{Dividing Logits by \texorpdfstring{$\sqrt{n}$}{sqrt(n)}}
\label{sec:logitsqrtn}

For example, in the 2-hidden-layer linear MLP example above, the network would compute 
\begin{equation}
f(\xi)=\frac{1}{\sqrt{n}}v\bar{h}(\xi),\quad \bar{h}(\xi)=Wh(\xi),\quad h(\xi)=U\xi,\label{eq:2LPMUP}
\end{equation}
where $U_{\alpha\beta}\sim\Gaus(0,1)$ and $W_{\alpha\beta},v_{\alpha\beta}\sim\Gaus(0,1/n)$ are the trainable parameters. Compared to SP (\cref{eq:standardparam2LP}), $h(\xi),\bar{h}(\xi)$ stays the same; only the logit $f(\xi)$ is scaled down.
Again, to simplify notation, we abbreviate $\bullet = \bullet_0$ and suppress explicit dependence on $\xi$.
This has two consequences

\paragraph*{Logits at Initialization Converge to 0}
since $f$ has variance $\Theta(1/n)$ (compare to the GP limit of MLP in SP at initialization).

\paragraph*{$\Theta(1)$ Learning Rate and Feature Learning}

Even though $f\to0$, the loss derivative $\chi=\loss'(f,y)$ stays $\Theta(1)$ if $y\ne0$. When we redo the calculation in \cref{eq:h2(1)}, we see 
\begin{align}
\bar{h}_{1} & =\bar{h}-n^{-c-1/2}\chi\ v^\trsp h^\trsp h=\bar{h}-\Theta(n^{-c+1/2})v^\trsp \label{eq:h2(1)MUP}\\
f_{1} & =f-n^{-c-1}\chi\ vv^\trsp h^\trsp h=f-\Theta(n^{-c}).\nonumber 
\end{align}
Because $v$ has coordinates of size $\Theta(n^{-1/2})$, we see that $\bar{h}$ and $f$ both change by $\Theta(1)$ coordinatewise if $c=0$ (i.e.\ learning rate is $\Theta(1)$). This directly illustrates feature learning after just 1 step of SGD.
For general MLPs, we can also check $a_{L+1}=1/2,c=0$ on top of \cref{eq:standardparam} implies $r=0$ and thus admits feature learning by \cref{thm:featurelearningMain}.

\paragraph{Kernel Behavior or Lack Thereof}
The example we have here, where we only train the middle layer in a linear MLP, actually \emph{is} in kernel regime.
This does not violate \cref{cor:dichotomyMain}, however, which assumes \cref{assm:nonlin}.
If, for example, we have tanh nonlinearity, then it is easy to see the $\mu$P SGD dynamics does not have a kernel limit:
If so, then $f_1 - f$ is linear in the learning rate $\eta$.
But note $\bar{h}_1 - \bar{h}$ is $\Theta(1)$ as $n\to\infty$ and linear in $\eta$, as can be derived similarly to \cref{eq:h2(1)MUP}.
Because tanh is bounded, this cannot happen.
Contrast this with SP or NTP, where $\bar{h}_1 - \bar{h}$ is $\Theta(1/\sqrt n)$ and thus ``resides in the linear regime of tanh'', allowing perfect scaling with $\eta$.

In addition, even in an linear MLP, if we train the middle layer \emph{and} the last layer, then the dynamics intuitively will become quadratic in the weights, so will not have a kernel limit.
Contrast this with SP or NTP, which suppress these higher order interactions because the learning rate is small, and a first order Taylor expansion heuristic holds.

\paragraph*{How is this different from standard parametrization with learning rate $1/\sqrt{n}$?}

As shown above, the logit $f$ blows up like $\Theta(\sqrt{n})$ after 1 step of SGD with learning rate $\Theta(1/\sqrt{n})$ in the standard parametrization, but remains $\Theta(1)$ in our parametrization here. The reason these two parametrizations seem similar is because in the 1st step, the weights receive the same updates modulo the loss derivative $\chi=\loss'(f,y)$. Consequently, $x_{1}^{L}-x^{L}$ and $h_{1}^{L}-h^{L}$ are $\Theta(1)$ coordinatewise in both cases. However, this update makes $x_{1}^{L}$ correlated with $W_{1}^{L+1},$ so that $W_{1}^{L+1}x_{1}^{L}$ (and $f_{1}$) scales like $\Theta(n^{1-a_{L+1}-b_{L+1}})$ due to Law of Large Numbers. Thus only in our parametrization here ($a_{L+1} = b_{L+1}=1/2$) is it $\Theta(1)$, while in standard parametrization ($a_{L+1} = 0, b_{L+1}=1/2$) it blows up like $\Theta(\sqrt{n})$. Contrast this with the behavior at initialization, where $W^{L+1}$ and $x^{L}$ are independent and zero-mean, so $W^{L+1}x^{L}$ scales like $\Theta(n^{1/2-a_{L+1}-b_{L+1}})$ by Central Limit Theorem.

\subsection{First Layer Parametrization}
\label{sec:Firstlayer}
While this now enables feature learning, the first layer preactivation $h$ effectively stays fixed throughout training even if we were to train $U$. For example, if we update $U$ in the linear MLP example \cref{eq:2LPMUP}, then by \cref{eq:backprop}, 
\begin{align*}
U_{1} & =U-n^{-c}\chi\ dU=U-n^{-c}\chi\ dh\xi^{\trsp}\\
h_{1} & =U_{1}\xi=h-n^{-c}\chi\ dh\xi^{\trsp}\xi=h-\Theta(n^{-c})dh
\end{align*}
since $\xi^{\trsp}\xi,\chi=\Theta(1)$. Now $dh=W^\trsp d\bar{h}=W^\trsp \frac{1}{\sqrt{n}}v^\trsp $ has roughly iid Gaussian coordinates, each of size $\Theta(1/n)$, since $\frac{1}{\sqrt{n}}v^\trsp $ has coordinates of the same size. Therefore, even with $c=0$, $h$ changes by at most $O(1/n)$ coordinatewise, which is dominated by its value at initialization. This $O(1/n)$ change also induces a $O(1/n)$ change in $f$, which would be dominated by the $\Theta(1)$ change due to $W$'s evolution, as seen in \cref{eq:h2(1)MUP}.

We therefore propose to set $a_{1}=-1/2,b_{1}=1/2$ on top of \cref{sec:logitsqrtn}'s parametrization. This implies the forward pass of $f$ remains the same but $U$'s gradient is scaled up by $n$, so that $h$ now changes by $\Theta(1)$ coordinatewise. In summary, we define 
\begin{defn}\label{defn:MUP}
The \emph{Maximal Update Parametrization (abbreviated MUP, or $\mu$P), }in the context of an $L$-hidden-layer MLP (\cref{eqn:MLP}), is given by 
\begin{equation*}
c=0,\quad b_l = 1/2\ \forall l, \quad a_{l}=\begin{cases}
-1/2 & l=1\\
0 & 2\le l\le L\\
1/2 & l=L+1.
\end{cases}
\end{equation*}
\end{defn}

Notice that $\mu$P for a 1-hidden-layer perceptron is equivalent to the mean field parametrization by \cref{stmt:SGDsymmetry}.
We also describe $\mu$P for any architecture in \cref{sec:muPAnyArch}.

\subsection{What is \texorpdfstring{$\mu$P}{MUP} Maximal In?}

For technical reasons, we adopt \cref{assm:nonlin} again for the formal results of this section.

In an abc-parametrization, the change in weight $W=W^l_t$ for any $l\ge 2$ due to learning rate $n^{-c}$ is $\del W \defeq - n^{-c} \cdot n^{-2a} dh\ x^{\trsp}$ where we abbreviated $x=x^{l-1}_t, h=h^{l}_t, a=a_l$.
(We will use $\del$ to denote 1-step change, but $\Delta$ to denote lifetime change).
In the next forward pass, $\del W$ contributes $\del W \bar x = -n^{1-c-2a} (x^{\trsp} \bar x/n) dh$, where $\bar x$ is the new activation due to change in previous layers' weights.
In general, $x$ and $\bar x$ are strongly correlated.
Then $x^{\trsp} \bar x/n\to R$ for some $R\ne 0$ by Law of Large Numbers (as they both have $\Theta(1)$ coordinates in a stable parametrization).
One can heuristically see that $dh$ has the same size as the last layer weights, which is $\Theta(n^{-(a_{L+1} + b_{L+1})} + n^{-(2a_{L+1}+c)})$ (where the first summand is from $W_0^{L+1}$ and the other from $\Delta W_t^{L+1}$).
Thus, $\del W \bar x$ is a vector with $\Theta(n^{-r_l})\defeq \Theta((n^{-(a_{L+1} + b_{L+1})} + n^{-(2a_{L+1}+c)})n^{1-c-2a})$ coordinates.
If $r_l > 0$, then $\del W \bar x$ contributes vanishingly; if $r_l <0$, then $\del W \bar x$ blows up.
For $l=1$, we get similar insights after accounting for the finite dimensionality of $\xi$.
\begin{defn}\label{defn:maximalupdate}
    For $l \in [L]$, we say $W^l$ \emph{is updated maximally} if $\Delta W^l_t x^{l-1}_t(\xi)$ has $\Theta(1)$ coordinates for some training routine%
    \footnote{Recall that \emph{training routine} means a package of learning rate $\eta n^{-c}$, training sequence $\{(\xi_{t},y_{t})\}_{t\ge0}$, and a loss function $\loss(f(\xi),y)$ that is continuously differentiable in the prediction of the model $f(\xi)$.}%
    , time $t\ge1$, and input $\xi$.
\end{defn}
\begin{prop}\label{prop:maximalupdate}
In a stable abc-parametrization, for any $l\in [L]$, $W^l$ is updated maximally iff 
\[r_l \defeq \min(a_{L+1}+b_{L+1},2a_{L+1}+c)+c-1+2a_l+\ind(l=1)=0.\]
\end{prop}
Note that $r$ (\cref{def:rMain}) is the minimum of $r_l$ over all $l$.
In $\mu$P, we can calculate that $r_l=0$ for all $l\in [L]$, so \emph{all $W^l, l \in [L],$ are updated maximally}.
Put another way, the final embedding $x^L(\xi)$ will have nonvanishing (nonlinear) contributions from $\Delta W^l$ of all $l$.
These contributions cause the logit $f(\xi)$ to change via interactions with $W^{L+1}_0$ and $\Delta W^{L+1}_t$.
If both $W^{L+1}_0$ and $\Delta W^{L+1}_t$ are too small, then the logit is fixed to its initial value, so all of the feature learning would have been useless.%
\footnote{It is indeed possible to perform feature learning in a trivial parametrization, e.g.\ $b_l=1/2\ \forall l, a_1=-1/2, a_2=100+1/2, c=-100$ in a 2-hidden-layer MLP.}
It's also possible for one to contribute vanishingly but not the other.%
\footnote{e.g.\ take $a_{L+1}=100+1/2,b_{L+1}=-100+1/2$, then $\Delta W^{L+1}$ is negligible.}
But both contribute in $\mu$P.
\begin{defn}\label{defn:lastlayermaximal}
    We say $W^{L+1}$ \emph{is updated maximally} (resp.\ \emph{initialized maximally}) if $\Delta W^{L+1}_t x^L_t(\xi) = \Theta(1)$ (resp.\ $W^{L+1}_0 \Delta x^L_t(\xi)=\Theta(1)$) for some training routine, time $t\ge 1$, and input $\xi$.
\end{defn}
Note \cref{defn:lastlayermaximal} is similar to \cref{defn:maximalupdate} except $\Delta W^{L+1}_t x^L_t(\xi) \in \R$ but $\Delta W^l_t x^{l-1}_t(\xi) \in \R^n$.
\begin{prop}\label{prop:lastlayermaximal}
    In a stable abc-parametrization, $W^{L+1}$ is 1) updated maximally iff $2a_{L+1}+c=1$, and 2) initialized maximally iff $a_{L+1}+b_{L+1}+r=1$.
\end{prop}
We remark that, by \cref{thm:nontriviality}, a parametrization is nontrivial iff $W^{L+1}$ is maximally updated or initialized.
Using \cref{prop:maximalupdate,prop:lastlayermaximal,thm:stabilityconditionsMain}, we can now easily conclude
\begin{thm}\label{thm:MUPvertex}
    In $\mu$P, $W^l$ is updated maximally for every $l \in [L+1]$, and $W^{L+1}$ is also initialized maximally.
    $\mu$P is the unique stable abc-parametrization with this property.
\end{thm}

\section{Deriving Feature Learning Infinite-Width Limit: Intuition and Examples}
\label{sec:exactformulas}

We propose \emph{the Tensor Programs technique} for deriving the infinite-width limit of any abc-parametrization.
This ultimately just requires the researcher to mechanically apply a set of rules to the computation graph underlying SGD.
However, while operationally simple, this procedure would seem ``too magical'' at first.
In this section, through a series of examples, we seek to build intuition for what is being automated by this procedure.
Then, in the next section, we formally describe the Tensor Programs framework.

\paragraph*{Setup and Notation}
For pedagogical simplicity, we only consider input dimension $d=1$ and learning rate $\eta=1$ here, but generalization to $d>1,\eta \ne 1$ is straightforward. 
We consider SGD with a singleton minibatch $\{(\xi_{t},y_{t})\}$ at time $t=0,1,2,\ldots$, where $\xi_{t}$ is the network input and $y_{t}$ is the label. We write $W_{t}^{l}$ for the matrix $W^{l}$ after $t$ steps of such training. For any network input $\xi\in\R$, we write $x_{t}^{l}(\xi)$ (resp. $h_{t}^{l}(\xi)$, $f_{t}(\xi)$) for the activation $x^{l}$ (resp. preactivation $h^{l}$, logits $f$) of the network after $t$ steps of SGD. We denote the scaled gradient $n\nabla_{x_{t}^{l}}f_{t}(\xi)$ (resp. $n\nabla_{h_{t}^{l}}f_{t}(\xi)$) by $dx_{t}^{l}(\xi)$ (resp. $dh_{t}^{l}(\xi)$). For brevity, we abuse notation and use $x_{t}^{l}$ (without being applied to $\xi$) to also denote the vector $x_{t}^{l}(\xi_{t})$ (applied specifically to $\xi_{t}$); likewise for $h_{t}^{l},dh_{t}^{l},dx_{t}^{l},f_{t}$. We will not use $x_{t}^{l}$ on its own to denote the function $\xi\mapsto x_{t}^{l}(\xi)$ so this should not cause confusion. The loss function is denoted $\loss$ and the loss derivative $\loss'(logit,target)$ is in the first argument. We write $\chi_{t}\defeq\loss'(f_{t},y_{t})$.

\subsection{1-Hidden-Layer MLP}
\label{sec:shallowMLP}

As mentioned above, for 1 hidden layer, the infinite-width $\mu$P limit is the same as the mean field limit of \citep{sirignano_mean_2018,mei_mean_2018,chizat_global_2018,rotskoff_neural_2018}. Nevertheless, we present a slightly different derivation of this that is more consistent with the philosophy of Tensor Programs. Such a network on input $\xi\in\R$ is given by 
\begin{equation}
f(\xi)=Vx(\xi),\quad x(\xi)=\phi(h(\xi)),\quad h(\xi)=U\xi,\label{eqn:UV1LP}
\end{equation}
for $U\in\R^{n\times1},V\in\R^{1\times n}$ parametrized like $U=\sqrt{n}u,V=\frac{1}{\sqrt{n}}v$ and with initialization $u_{\alpha\beta},v_{\alpha\beta}\sim\Gaus(0,1/n)$.%
\footnote{Again, more generally, we can insert constants in this parametrization, like $U=\frac{\sqrt{n}}{\sqrt{d}}u$, but we omit them here for simplicity.}
Then $U_0$ (the initial value of $U$) has iid $\Gaus(0,1)$ coordinates. It will turn out to be convenient to represent each such coordinate distribution as a random variable $Z^{U_0}\defeq\Gaus(0,1)$. Likewise, let $Z^{nV_0}\defeq\Gaus(0,1)$, independent from $Z^{U_0}$, represent the coordinate distribution of $nV_0$ (we do $nV_0$ instead of $V_0$ so that the $Z$ random variable is always independent of $n$).
We derive the $\mu$P limits of the first forward and backward passes manually before stating the general case.
To lighten notation, we suppress the $t=0$ subscript (e.g.\ $U = U_0, h=h_0, f=f_0$, etc), as we will spend some time on the first SGD step.

\paragraph*{First Forward Pass}

After randomly initialization, the preactivation $h=h(\xi)$ (where $\xi=\xi_0 \in \R $ is the first input) has iid coordinates, each a sample from $Z^{h}\defeq\xi Z^{U} \in \R$. Naturally, $x=x(\xi)$ has iid coordinates as well, each a sample from $Z^{x}\defeq\phi(Z^{h})$. Finally, $f=Vx=\frac{1}{n}\sum_{\alpha=1}^{n}(nV)_{\alpha}x_{\alpha}\to \mathring f \defeq \EV Z^{nV}Z^{x}$ by Law of Large Numbers as $n\to\infty$.\footnote{All convergence in this section will be almost sure, but to focus on the intuition here and less on the formalities, we do not explicitly write this down.}
In particular, $f$ becomes deterministically 0 in this limit because $V$ and $U$ are independent. For a typical loss function $\loss$, the loss derivative $\chi \defeq\loss'(f,y)$ then also become deterministic, $\chi \to \mathring \chi \defeq \loss'(\mathring f, y)$.

\paragraph*{First Backward Pass}

Similarly, $dx=nV^{\trsp}$ (recall $dx_t \defeq n\nabla_{x_t} f_t$) has coordinates distributed like $Z^{dx} \defeq Z^{nV}$ and $dh=dx\odot\phi'(h)$ has coordinates distributed like $Z^{dh} \defeq Z^{dx}\phi'(Z^{h})=Z^{nV}\phi'(Z^{h})$. Then SGD with learning rate $1$ makes the following updates: 
\begin{align*}
v_1 & =v-\chi x/\sqrt n& \implies&& V_1&=V-\chi x/n\\
u_1 & =u-\chi \xi\ dh/\sqrt n& \implies&& U_1&=U-\chi \xi \ dh.
\end{align*}
Since $\chi $ converges to a deterministic limit $\mathring \chi$, the coordinates of these updates are roughly iid, corresponding to an update of $Z$ random variables: %
\[
Z^{nV_1}=Z^{nV}-\mathring \chi Z^{x},\quad Z^{U_1}=Z^{U}-\mathring\chi \xi Z^{dh}.
\]

\paragraph*{Second Forward Pass}

Thus $V_1$ and $U_1$ still have roughly iid coordinates after 1 SGD step. Then, in the second forward pass,
    $h_1$ has coordinates 
    \[Z^{h_1}\defeq\xi_{1}Z^{U_1}=\xi_{1}Z^{U}-\xi_{1}\mathring\chi \xi Z^{dh}=\xi_{1}Z^{U}-\xi_{1}\mathring\chi \xi Z^{nV}\phi'(Z^{h}),\]
    $x_1$ has coordinates $Z^{x_1}\defeq\phi(Z^{h_1})$, 
    and the output is
    \begin{equation}
        f_{1}=\frac{1}{n}\sum_{\alpha=1}^{n}(nV_1)_{\alpha}x_{\alpha}\to \mathring f_1 \defeq \EV Z^{nV_1}Z^{x_1} = \EV (Z^{nV}-\mathring \chi Z^{x}) Z^{x_1}
        \label{eqn:1LPf1}
    \end{equation} as $n\to\infty$.
Then $\chi_{1}\defeq\loss'(f_{1},y_{1}) \to \mathring \chi_1 \defeq \loss'(\mathring f_{1},y_{1})$ becomes deterministic.
The gradient vectors have roughly iid coordinates by a similar logic.

\newcommand{\dout}{{d_{\mathrm{o}}}}

\paragraph*{$t$th Iteration}

Repeating the above reasoning shows that at any time $t$ (independent of $n$), we obtain
\begin{thm}\label{thm:1LPMUPLimit}
    Consider a 1-hidden-layer MLP in $\mu$P (\cref{eqn:UV1LP}) and any training routine with learning rate $1$.
    Suppose $\phi'$ is pseudo-Lipschitz.%
    \footnote{This roughly means that $\phi'$ has a polynomially bounded weak derivative; see \cref{defn:pseudoLipschitz}.}
    As $n\to\infty$, for every input $\xi$, $f_t(\xi)$ converges almost surely to $\mathring f_t(\xi)$ defined as follows:
    \begin{gather}
        f_{t}(\xi)\asto\mathring{f}_{t}(\xi)\defeq\EV Z^{nV_t}Z^{x_t(\xi)},\quad Z^{x_t(\xi)}\defeq \phi(Z^{h_t(\xi)}),\quad Z^{h_t(\xi)}\defeq\xi Z^{U_t},
        \label{eqn:MUP1LPforward}\\
        \mathring \chi_{t}\defeq\loss'(\mathring f_{t},y_{t}),\quad
        Z^{nV_{t+1}}\defeq Z^{nV_t}-\mathring \chi_{t}Z^{x_t},\quad Z^{U_{t+1}}\defeq Z^{U_t}-\mathring\chi_{t}\xi_{t}Z^{nV_t}\phi'(Z^{h_t}),
        \label{eqn:MUP1LPSGD}
    \end{gather}
    with, as initial conditions, $Z^{U_0}$ and $Z^{nV_0}$ being independent standard Gaussians, where in \cref{eqn:MUP1LPSGD} we abbreviated $\mathring f_t = \mathring f_t(\xi_t), x_t = x_t(\xi_t), h_t = h_t(\xi_t)$.
\end{thm}
As aforementioned, this is a discrete time, minibatched version of the mean field limit of \citep{sirignano_mean_2018,mei_mean_2018,chizat_global_2018,rotskoff_neural_2018}.%
\footnote{\cite{sirignano_mean_2018,mei_mean_2018,chizat_global_2018,rotskoff_neural_2018} present the equations in terms of the PDF of $Z$ random variables. Formally, the PDF limit can be obtained by taking the continous-time limit of \cref{eqn:MUP1LPforward,eqn:MUP1LPSGD} and then applying Focker-Planck.
Note our derivation, when formalized using the Tensor Programs framework below, does not require smoothness and support assumptions on the initialization of $U,V$ in those works:
The initialization distribution here can be replaced with any image of Gaussians under pseudo-Lipschitz functions, which includes nonsmooth and singular distributions.}
When $\phi$ is identity, it's easy to see that $Z^{nV_t}$ and $Z^{U_t}$ are always (deterministic) linear combinations of $Z^{nV_0}$ and $Z^{U_0}$, say $Z^{nV_t} = A_t Z^{nV_0} + B_t Z^{U_0}$ and $Z^{U_t} = C_t Z^{nV_0} + D_t Z^{U_0}$.
Then the limit $\mathring f_t$ depends solely on $A_t, B_t, C_t, D_t$.
By tracking their evolution, we get the following greatly simplified formula for an infinite-width $\mu$P linear network.
\begin{cor}\label{cor:lin1LP}
    Consider a 1-hidden-layer linear MLP in $\mu$P (\cref{eqn:UV1LP}) and any training routine with learning rate $1$.
    As $n\to\infty$, for every input $\xi$, $f_t(\xi)$ converges almost surely to $\mathring f_t(\xi)$ defined as follows:
    \begin{align*}
        \mathring f_t(\xi) &= (A_t C_t + B_t D_t)\xi,\quad
        \mathring \chi_{t}=\loss'(\mathring f_{t},y_{t}),\\
        (A_{t+1}, B_{t+1}) &= (A_t, B_t) - \mathring \chi_t \xi_t (C_t, D_t),\\
        (C_{t+1}, D_{t+1}) &= (C_t, D_t) - \mathring \chi_t \xi_t (A_t, B_t),\quad
    \end{align*}
    with initial condition $A_0 = D_0 = 1, B_0 = C_0 = 0$.
\end{cor}
This can be easily generalized to larger input and output dimenions (see \cref{sec:mamldetails}).
In a gist, such an infinite-width $\mu$P linear network with input dimension $d$ and output dimension $\dout$ is equivalent to a width-$(d+\dout)$ linear network with the same input/output dimensions but an ``diagonal'', instead of random, initialization.
Our Word2Vec and MAML experiments will crucially rely on this simplifying observation.
We remark that, in contrast to our approach, such an observation would be obscured by the PDE perspective of prior works \citep{sirignano_mean_2018,mei_mean_2018,chizat_global_2018,rotskoff_neural_2018}.

\subsection{2-Hidden-Layer MLP: SGD with Partially Decoupled Backpropagation}
\label{sec:decoupledSGD}
A 2-hidden-layer MLP is given by
\[
f(\xi)=V\bar x(\xi),\quad \bar x(\xi) = \phi (\bar h(\xi)),\quad \bar h(\xi) = W x(\xi),\quad x(\xi)=\phi(h(\xi)),\quad h(\xi)=U\xi,
\]
for $U\in\R^{n\times1},W\in \R^{n\times n},V\in\R^{1\times n}$ parametrized like $U=\sqrt{n}u,V=\frac{1}{\sqrt{n}}v$ and with initialization $u_{\alpha\beta},W_{\alpha\beta},v_{\alpha\beta}\sim\Gaus(0,1/n)$.
The presence of the $n\times n$ Gaussian matrix $W$ (``$\infty \times \infty$'' as opposed to ``$\infty \times$ finite'' like $U$ or ``finite $\times \infty$'' like $V$) is new and
has two major effects on the infinite-width training dynamics:
1) A Central Limit effect from the random Gaussian nature of $W$ and 2) a correlation effect between $W$ and its transpose $W^\trsp$.
We isolate the first effect here by analyzing a slightly different version of backpropagation (which has a different limit than normal backpropagation), and then discuss the second effect in the next section.
We abuse notation and abbreviate $W=W_0$.

\paragraph*{Partially Decoupled Backpropagation}

In this section, we analyze a version of SGD where the backpropagation weights are partially decoupled from the forward propagation weights. Here, we think of $\Delta W_t$ as the trainable weights, initialized at 0, and think of the Gaussian $W$ as untrainable ``constants''. The forward pass proceeds normally%
\footnote{i.e.\ $f_t=V_t \bar{x}_t,\bar{x}_t=\phi(\bar{h}_t),\bar{h}_t=(W+\Delta W_t)x_t,x_t=\phi(h_t),h_t=U\xi_t.$}
with $W_{t}=W+\Delta W_t$. But we sample and fix an iid copy $\widetilde{W}$ of $W^\trsp $ before training, and in the backward pass compute 
\begin{equation}
dx_{t}=(\widetilde{W}+\Delta W_t^{\trsp})d\bar{h}_{t}\quad\text{instead of}\quad dx_{t}=(W^\trsp +\Delta W_t^{\trsp})d\bar{h}_{t}=W_{t}^{\trsp}d\bar{h}_{t}.\label{eqn:decoupledbackprop}
\end{equation}
In particular, at initialization, we would have $dx_{0}=\widetilde{W} d\bar{h}_{0}$ instead of $dx_{0}=W^\trsp d\bar{h}_{0}$. Everything else stays the same in the backward pass%
\footnote{i.e.\ $d\bar{x}_{t}=nV_{t}^{\trsp},d\bar{h}_{t}=\phi'(\bar{h}_{t})\odot d\bar{x}_{t},dh_{t}=\phi'(h_{t})\odot dx_{t}$}.
Finally, each weight is still updated by SGD via the usual outer products: with $\chi_{t}\defeq\loss'(f_{t},y_{t})$,
\begin{equation}
    v_{t+1}=v_{t}-\chi_{t}\bar{x}_{t}^\trsp/\sqrt{n},
    \quad\Delta w_{t+1}=\Delta w_t-\chi_{t}d\bar{h}_{t}x_{t}^\trsp/n,
    \quad u_{t+1}=u_t-\chi_{t}\xi_{t}dh_{t}^\trsp/\sqrt n.
\label{eq:wupdateDecoupledsgd}
\end{equation}
Since $V = v/\sqrt n, W = w, U = \sqrt n u$ per $\mu$P, this causes the following changes in $W$s:
\begin{equation}
V_{t+1}=V_{t}-\chi_{t}\bar{x}_{t}^\trsp/n,
\quad\Delta W_{t+1}=\Delta W_t-\chi_{t}d\bar{h}_{t}x_{t}^\trsp/n,
\quad 
U_{t+1}=U_{t}-\chi_{t}\xi_{t}dh_{t}^\trsp\label{eq:WupdateDecoupledsgd}
\end{equation}
Note here we update $\Delta w$ and $\Delta W$ instead of $w$ and $W$.

\paragraph{Why This Decoupled SGD?}
The reasons we talk about this version of SGD is that it isolates the effect of having a Gaussian $n\times n$ matrix $\widetilde{W}$ in the backward pass, and we can derive its infinite-width limit relatively easily using Central Limit heuristics. In the normal version of SGD, $\widetilde{W}$ would equal $W^{\trsp}$, and its correlation with $W$ creates additional terms in the infinite-width dynamics, that are better explained on their own.

Again, we walk through the first few forward and backward passes to gain some intuition for the infinite-width limit, before stating the general case.

\paragraph*{First Forward Pass}
is similar to that in \cref{sec:shallowMLP} and follows the usual calculations involved in deriving the NNGP%
\footnote{
1) $h_{0}$ is iid Gaussian with coordinates drawn from $Z^{h_{0}}=\xi_{0}Z^{U_{0}}$; 2) $x_{0}$ has coordinates $Z^{x_{0}}=\phi(Z^{h_{0}})$; 3) $\bar{h}_{0}=Wx_{0}$ has roughly iid coordinates drawn from a zero-mean Gaussian $Z^{\bar{h}_{0}}$ by a Central Limit heuristic, where $Z^{\bar{h}_{0}}$ is correlated with $Z^{\bar{h}_{0}(\xi)}$ for any $\xi$ (including $\xi=\xi_{0}$) with covariance $\Cov(Z^{\bar{h}_{0}},Z^{\bar{h}_{0}(\xi)})=\lim_{n\to\infty}\frac{1}{n}x_{0}^{\trsp}x_{0}(\xi)=\EV Z^{x_{0}}Z^{x_{0}(\xi)}$; 4) $\bar{x}_{0}$ has coordinates $Z^{\bar{x}_{0}}=\phi(Z^{\bar{h}_{0}})$; 5) $f_0=\frac{1}{n}\sum_{\alpha=1}^{n}(nV_{0})_{\alpha}\bar{x}_{0\alpha}\to\mathring{f}_{0}\defeq\EV Z^{nV_{0}}Z^{\bar{x}_{0}}$ by a Law of Large Number heuristic.
}.

\paragraph*{First Backward Pass}
is similar to that in \cref{sec:shallowMLP} and to calculations involved in deriving Neural Tangent Kernel, except swapping $W^{\trsp}$ with $\widetilde{W}$ (which at this point has no visible effect, because of the Gradient Independence Phenomenon \cite{TP2}; but the effect will become clear in the second forward pass)%
\footnote{
1) $d\bar{x}_{0}=nV_{0}^{\trsp}$so $Z^{d\bar{x}_{0}}=Z^{nV_{0}}$; 2) $Z^{d\bar{h}_{0}}=\phi'(Z^{\bar{h}_{0}})\odot Z^{d\bar{x}_{0}}$; 3) $Z^{dx_{0}}=Z^{\widetilde{W}d\bar{h}_{0}}$ is Gaussian with covariance $\Cov(Z^{dx_{0}},Z^{dx_{0}(\xi)})=\lim_{n\to\infty}\frac{1}{n}d\bar h_{0}^{\trsp}d \bar h_{0}(\xi)=\EV Z^{d\bar h_{0}}Z^{d\bar h_{0}(\xi)}$ for any input $\xi$; 4) $Z^{dh_{0}}=\phi'(Z^{h_{0}})\odot Z^{dx_{0}}$.
Since $f$ converges to a deterministic number $\mathring{f}_{0}$, we also generically have $\loss'(f,y_{0})\to\mathring{\chi}_{0}\defeq\loss'(\mathring{f}_{0},y_{0})$. Finally, the weights are updated like \cref{eq:WupdateDecoupledsgd}.
}.
We end up with $\Delta W_1 = -\chi_0 d\bar{h}_0 x^{\trsp}_0$, as usual.

\paragraph*{Second Forward Pass}

As usual, we have $Z^{h_{1}}=\xi_{1}Z^{U_{1}}=\xi_{1}Z^{U_{0}}-\mathring{\chi}_{0}\xi_{1}\xi_{0}Z^{dh_{0}}$ and $Z^{x_{1}}=\phi(Z^{h_{1}})$, reflecting the coordinate distributions of $h_{1}$ and $x_{1}$%
\footnote{Recall they abbreviate $h_1(\xi_1)$ and $x_1(\xi_1)$}.
Next,
\begin{equation}
    \bar{h}_{1}=Wx_{1}+\Delta W_1x_{1}=Wx_{1}-\chi_{0}d\bar{h}_{0}\frac{x_{0}^{\trsp}x_{1}}{n}.
    \label{eqn:2LPh12}
\end{equation}
On one hand, 1) $\frac{x_{0}^{\trsp}x_{1}}{n}\to\EV Z^{x_{1}}Z^{x_{0}}$ by a Law of Large Numbers heuristic. On the other hand, 2) by a Central Limit heuristic, $Wx_{1}$ should roughly have Gaussian coordinates $Z^{Wx_{1}}$ correlated with $Z^{\bar{h}_{0}}=Z^{Wx_{0}}$ with $\Cov(Z^{Wx_{1}},Z^{Wx_{0}})=\lim \frac{x_{0}^{\trsp}x_{1}}{n}=\EV Z^{x_{1}}Z^{x_{0}}$. However, \emph{very importantly}, this Central Limit heuristic is correct only because we used $\widetilde{W}$ in backprop instead of $W^{\trsp}$; otherwise, $h_{1}$ has a strong correlation with $W$ through $dh_{0} = \phi'(h_0) \odot (W^{\trsp} d\bar{h}_0)$, and thus so does $x_{1}$, so that $Wx_{1}$ no longer has Gaussian coordinates. 
This is the ``second major effect'' referred to in the beginning of this section.
See \cref{sec:2LPSGD} for how to handle this correlation.

In any case, in our scenario here,
\begin{equation*}
    Z^{\bar{h}_{1}}\defeq Z^{Wx_{1}}-cZ^{d\bar{h}_{0}},\quad \text{where}\quad c=\mathring{\chi}_{0}\EV Z^{x_{1}}Z^{x_{0}},
\end{equation*}
is a linear combination of a Gaussian variable and the gradient $d\bar{h}_{0}$'s coordinate random variable. Finally, $Z^{\bar{x}_{1}}=\phi(Z^{\bar{h}_{1}})$ and the logit is $f_{1}=\frac{1}{n}\sum_{\alpha=1}^{n}(nV_{1})_{\alpha}\bar{x}_{1\alpha}\to\mathring{f}_{1}\defeq\EV Z^{nV_{1}}Z^{\bar{x}_{1}}=\EV Z^{nV_{0}}Z^{\bar{x}_{1}}-\mathring{\chi}_{0}\EV Z^{\bar{x}_{0}}Z^{\bar{x}_{1}}$.

\paragraph*{Second Backward Pass}

Everything proceeds just like in the 1-hidden-layer case%
\footnote{$d\bar{x}_{1}=nV_{1}^{\trsp},$ $d\bar{h}_{1}=d\bar{x}_{1}\odot\phi'(\bar{h}_{1})$, $dh_{1}=dx_{1}\odot\phi'(h_{1})$}
except for the computation of 
\[dx_{1}=\widetilde{W}d\bar{h}_{1}-\Delta W_1^\trsp d\bar{h}_{1}=\widetilde{W}d\bar{h}_{1}-\chi_{0}x_{0}\frac{d\bar{h}_{0}^\trsp d\bar{h}_{1}}{n}.\]
Like in the computation of $\bar{h}_{1}$ in \cref{eqn:2LPh12}, $\frac{d\bar{h}_{0}^\trsp d\bar{h}_{1}}{n}\to\EV Z^{d\bar{h}_{0}}Z^{d\bar{h}_{1}}$ and $\widetilde{W}d\bar{h}_{1}$ is roughly Gaussian (and correlated with $\widetilde{W}d\bar{h}_{0}$ in the natural way).
But again, for this Gaussian intuition to be correct, it is crucial that we use $\widetilde{W}$ here instead of $W^\trsp $, or else $d\bar{x}_{1}$ (and thus $d\bar{h}_{1}$) is strongly correlated with $W^\trsp $ (through $\bar{x}_{0}=\phi(Wx_{0})$ inside $n\Delta V_{1} = -\chi_0 \bar{x}^\trsp_0$).

In any case, we have
\begin{equation*}
    Z^{dx_{1}}=Z^{\widetilde Wd\bar{h}_{1}}-cZ^{x_{0}}, \quad \text{where} \quad c=\mathring{\chi}_{0}\EV Z^{d\bar{h}_{0}}Z^{d\bar{h}_{1}},
\end{equation*}
is a sum of Gaussian $Z^{\widetilde Wd\bar{h}_{1}}$ and a multiple of $Z^{x_{0}}$. Then weights are updated according to \cref{eq:WupdateDecoupledsgd}.

\paragraph*{$t$th Iteration}

For general $t$, we always have (true in normal SGD as well)
\[
\Delta W_t=-\frac{1}{n}\sum_{s=0}^{t-1}\chi_{s}d\bar{h}_{s}x_{s}^{\trsp}
\]
so that in the forward pass 
\begin{align*}
\bar{h}_{t} & =Wx_{t}+\Delta W_tx_{t}=Wx_{t}-\sum_{s=0}^{t-1}\chi_{s}d\bar{h}_{s}\frac{x_{s}^{\trsp}x_{t}}{n}
\numberthis\label{eqn:h2t}
\\
Z^{\bar{h}_{t}} &\defeq Z^{Wx_{t}}-\sum_{s=0}^{t-1}\mathring{\chi}_{s}Z^{d\bar{h}_{s}}\EV Z^{x_{s}}Z^{x_{t}}.
\end{align*}
Here $Z^{Wx_{t}}$ is Gaussian with covariance $\Cov(Z^{Wx_{t}},Z^{Wx_{s}})=\EV Z^{x_{t}}Z^{x_{s}}$ for any $s$. This means that $Z^{\bar{h}_{t}}$ and $Z^{\bar{h}_{s}}$ are correlated through $Z^{Wx_{t}},Z^{Wx_{s}}$ (but also through $Z^{d\bar{h}_{r}},r\le\min(t,s)$). Likewise, in the backward pass, 
\begin{align*}
dx_{t} & =\widetilde{W}d\bar{h}_{t}-\Delta W^{\trsp}d\bar{h}_{t}=\widetilde{W}d\bar{h}_{t}-\sum_{s=0}^{t-1}\chi_{s}x_{s}\frac{d\bar{h}_{s}^\trsp d\bar{h}_{t}}{n}\\
Z^{dx_{t}} &\defeq Z^{\widetilde{W}d\bar{h}_{t}}-\sum_{s=0}^{t-1}\mathring{\chi}_{s}Z^{x_{s}}\EV Z^{d\bar{h}_{s}}Z^{d\bar{h}_{t}}
\end{align*}
Here, $Z^{\widetilde{W}d\bar{h}_{t}}$ is Gaussian with covariance $\Cov(Z^{\widetilde{W}d\bar{h}_{t}},Z^{\widetilde{W}d\bar{h}_{s}})=\EV Z^{d\bar{h}_{t}}Z^{d\bar{h}_{s}}$ for any $s$. Thus, $Z^{dx_{t}}$ and $Z^{dx_{s}}$ are correlated through $Z^{\widetilde{W}d\bar{h}_{t}},Z^{\widetilde{W}d\bar{h}_{s}}$ (but also through $Z^{x_{r}},r\le\min(t,s)$). Again, the Gaussianity of $Z^{Wx_{t}}$ and $Z^{\widetilde{W}d\bar{h}_{t}}$ depend crucially on the fact that we use $\widetilde{W}$ instead of $W^\trsp $ in backpropagation.

Other parts of the forward and backward propagations are similar to before.
Our reasoning can be formalized via Tensor Programs to prove the following
\begin{thm}\label{thm:decoupledSGD2LP}
    Consider a 2-hidden-layer MLP in $\mu$P with partially decoupled backpropagation as in \cref{eqn:decoupledbackprop} and any training routine with learning rate $1$.
    Suppose $\phi'$ is pseudo-Lipschitz.%
    \footnote{This roughly means that $\phi'$ has a polynomially bounded weak derivative; see \cref{defn:pseudoLipschitz}.}
    As $n\to\infty$, for every input $\xi$,
    \[
        f_{t}(\xi)\asto\mathring{f}_{t}(\xi),\quad\text{where $\mathring f_t(\xi)$ is defined as follows:}
        \]
    (forward pass)
    \begin{gather*}
        \mathring{f}_{t}(\xi)\defeq\EV Z^{nV_t}Z^{\bar x_t(\xi)},\quad Z^{\bar x_t(\xi)}\defeq \phi(Z^{\bar h_t(\xi)}),\quad
        Z^{x_t(\xi)}\defeq \phi(Z^{h_t(\xi)}),\quad
        Z^{h_t(\xi)}\defeq\xi Z^{U_t}\\
        Z^{\bar{h}_{t}(\xi)} \defeq Z^{Wx_{t}(\xi)}-\sum_{s=0}^{t-1}\mathring{\chi}_{s}Z^{d\bar{h}_{s}}\EV Z^{x_{s}}Z^{x_{t}(\xi)}\numberthis\label{eq:Zh2}\\
        \{Z^{Wx_{t}(\xi)}\}_{\xi, t} \ \ \text{centered, jointly Gaussian with}\ \
        \Cov(Z^{Wx_{t}(\xi)},Z^{Wx_{s}(\zeta)})=\EV Z^{x_{t}(\xi)}Z^{x_{s}(\zeta)}
    \end{gather*}
    (backward pass)
    \begin{gather*}
        \chi_{t}\defeq\loss'(\mathring f_{t},y_{t}),\quad
        Z^{d\bar{x}_{t}} \defeq Z^{nV_{t}},\quad
        Z^{d\bar h_t} \defeq \phi'(Z^{\bar h_t}) Z^{d\bar x_t}\quad
        Z^{d h_t} \defeq \phi'(Z^{h_t}) Z^{dx_t}\\
        Z^{dx_{t}} \defeq Z^{\widetilde{W}d\bar{h}_{t}}-\sum_{s=0}^{t-1}\mathring{\chi}_{s}Z^{x_{s}}\EV Z^{d\bar{h}_{s}}Z^{d\bar{h}_{t}}\numberthis\label{eq:Zdx1}\\
        \{Z^{\widetilde{W}d\bar{h}_{t}}\}_{t} \ \ \text{centered, jointly Gaussian with}\ \
        \Cov(Z^{\widetilde{W}d\bar{h}_{t}},Z^{\widetilde{W}d\bar{h}_{s}})=\EV Z^{d\bar{h}_{t}}Z^{d\bar{h}_{s}}
    \end{gather*}
    ($U, V$ updates)
    \begin{gather*}
        Z^{nV_{t+1}} \defeq Z^{nV_{t}}-\mathring{\chi}_{t}Z^{\bar{x}_{t}}
        \quad
        Z^{U_{t+1}} \defeq Z^{U_{t}}-\mathring{\chi}_{t}\xi_{t}Z^{dh_{t}}        
    \end{gather*}
    with $Z^{U_0}$ and $Z^{nV_0}$ being independent standard Gaussians as initial conditions, and by definition,
    $\{Z^{Wx_{t}(\xi)}\}_{\xi, t}$, $\{Z^{\widetilde{W}d\bar{h}_{t}}\}_{t}$, $Z^{U_0}$, and $Z^{nV_0}$ are mutually independent sets of random variables.
    Here, if $h_t$ appears without argument, it means $h_t(\xi_t)$; likewise for $\bar h_t, x_t, \bar x_t, dh_t, d\bar h_t, dx_t, d\bar x_t, \mathring f_t$.
\end{thm}

\subsection{2-Hidden-Layer MLP: Normal SGD}
\label{sec:2LPSGD}

Finally, we dicuss normal SGD for 2-hidden-layer MLP, i.e.\ in backprop we compute 
\[
dx_{t}=W_{t}^{\trsp}d\bar{h}_{t}=(W^\trsp +\Delta W^{\trsp})d\bar{h}_{t}.
\]
The first forward and backward passes are essentially the same as in the last section. However, as mentioned there, in the second forward pass, $W x_{1}$ (a part of $\bar{h}_{1}=W x_{1}+\Delta W_1x_{1}$) will no longer be approximately Gaussian because of the correlation between $x_{1}$ and $W $. Let's first get some intuition for why this is before stating the infinite-width limit formally.%

\paragraph*{Warmup: $\phi=\protect\id$}

First, as warmup, suppose $\phi=\id$. In this case, $W x_{1}$ will actually still be Gaussian, but its variance will be different than what's predicted in the previous section. To lighten notation, we write $x=x_1$ in this section. Then unwinding the definition of $x$, we have 
\[
x=h+aW^{\trsp}z
\]
where we abbreviated $h=\xi_{1}U_{0},z=d\bar{h}_{0},a=-\chi_{0}\xi_{0}\xi_{1}$.
Then $Wx$ has coordinates 
\[
(Wx)_{\alpha}=(Wh)_{\alpha}+a(WW^{\trsp}z)_{\alpha}.
\]
As derived in the first forward pass in \cref{sec:decoupledSGD}, $(Wh)_{\alpha}$ is approximately Gaussian (particularly because $W ,U_{0}$ are independent). 
This is true for $(WW^{\trsp}z)_{\alpha}$ as well here because we assumed $\phi=\id$, but not true generally.
Indeed, 
\[
(WW^{\trsp}z)_{\alpha}=\sum_{\beta,\gamma}W_{\alpha\beta}W_{\gamma\beta}z_{\gamma}=z_{\alpha}\sum_{\beta}(W_{\alpha\beta})^{2}+\sum_{\beta}\sum_{\gamma\ne\alpha}W_{\alpha\beta}W_{\gamma\beta}z_{\gamma}.
\]
We will soon see the derivations of \cref{sec:decoupledSGD} correspond to ignoring the first term: %
 In the second term, there are $n$ summands of the form $\sum_{\gamma\ne\alpha}W_{\alpha\beta}W_{\gamma\beta}z_{\gamma}$ that are approximately iid with variance $\approx\|z\|^{2}/n^{2}$. Thus, the second term itself, by a Central Limit heuristic, should converge to $\Gaus(0,\lim_{n\to\infty}\|z\|^{2}/n)$. On the other hand, the first term $z_{\alpha}\sum_{\beta}(W_{\alpha\beta})^{2}\to z_{\alpha}$ by Law of Large Numbers. Tying it all together, $(Wx)_{\alpha}$ is a linear combination of two Gaussian terms $(Wh)_{\alpha}$ and $\sum_{\beta}\sum_{\gamma\ne\alpha}W_{\alpha\beta}W_{\gamma\beta}z_{\gamma}$, as well as as $z_{\alpha}$ (which is Gaussian in the case of $\phi=\id$, but not generally).

Note that, if we did $(W\widetilde{W}z)_{\alpha}$ instead of $(WW^{\trsp}z)_{\alpha}$, as in the last section, then the same analysis would show the first term is $z_{\alpha}\sum_{\beta}W_{\alpha\beta}\widetilde{W}_{\beta\alpha}\to0$, while the second term converge in distribution to the same Gaussian. Thus, the effect of decoupling in \cref{sec:decoupledSGD} is killing the copy of $z$ in $(Wx)_{\alpha}$.

We can summarize our derivation here in terms of $Z$: 
\begin{align}
\text{For \ensuremath{\phi=\id}:}\quad
Z^{Wx} & \defeq Z^{Wh}+aZ^{WW^{\trsp}z}=Z^{Wh}+a(\hat{Z}^{WW^{\trsp}z}+Z^{z}),\label{eq:ZW2x1}\\
 & \text{where}\quad\hat{Z}^{WW^{\trsp}z}\defeq\Gaus(0,\EV(Z^{z})^{2}).\nonumber 
\end{align}
Note the Central Limit heuristic in the derivation of $\hat{Z}^{WW^{\trsp}z}$ also shows $\hat{Z}^{WW^{\trsp}z}$ is jointly Gaussian with $Z^{Wh}$ with $\Cov(\hat{Z}^{WW^{\trsp}z},Z^{Wh})=\EV Z^{W^{\trsp}z}Z^{h}$. So, to put \cref{eq:ZW2x1} in a form more suggestive of the general case, we will write 
\begin{equation}
Z^{Wx}=\hat{Z}^{Wx}+aZ^{z},\quad
\text{where}\quad
\hat{Z}^{Wx}=Z^{Wh}+a\hat{Z}^{WW^{\trsp}z}\disteq\Gaus(0,\EV(Z^{x})^{2}).
\label{eqn:linZhat}
\end{equation}

\paragraph*{General $\phi$}

Unwinding the definition of $x,$ we have 
\begin{align}
x & =\phi(h+aW^{\trsp}z\odot\phi'(h_{0})).\numberthis\label{eqn:x11}
\end{align}
By Taylor-expanding $\phi$, we can apply a similar (though more tedious) argument as above to derive 
\begin{equation}
Z^{Wx}=\hat{Z}^{Wx}+cZ^{z}\label{eq:ZWx}
\end{equation}
where $c=a\EV\phi'(Z^{h_{1}})\phi'(Z^{h_{0}})$ and $\hat{Z}^{Wx}\disteq\Gaus(0,\EV(Z^{x})^{2})$. In the case of $\phi=\id$, $c$ reduces to $a$ as above, recovering \cref{eqn:linZhat}. For general $\phi$, we can immediately see that $Z^{Wx}$ is not Gaussian because $Z^{z}=Z^{d\bar{x}_{0}}\phi'(Z^{\bar{h}_{0}})$ is not.
In the Tensor Programs framework formalized in \cref{sec:TensorPrograms}, $cZ^{z}$ is denoted $\dot Z^{Wx}$.

Similarly, coordinates distribution of $dx_{1}=W_{1}^{\trsp}d\bar{h}_{1}$ will also change in the backward pass.

\paragraph*{General $t$}

For general $t$, we obtain dynamical equations in $Z$ identical to those in \cref{thm:decoupledSGD2LP} except that \cref{eq:Zh2} and \cref{eq:Zdx1} need to be modified.
We state the general result below.

\begin{thm}\label{thm:SGD2LP}
    Consider a 2-hidden-layer MLP in $\mu$P and any training routine with learning rate $1$.
    Suppose $\phi'$ is pseudo-Lipschitz.%
    \footnote{This roughly means that $\phi'$ has a polynomially bounded weak derivative; see \cref{defn:pseudoLipschitz}.}
    As $n\to\infty$, for every input $\xi$, $f_{t}(\xi)\asto\mathring{f}_{t}(\xi)$ where $\mathring f_t(\xi)$ is defined the same way as in \cref{{thm:decoupledSGD2LP}} except that \cref{{eq:Zh2}} should be replaced with
    \begin{gather*}
        Z^{\bar{h}_{t}(\xi)} \defeq \hat Z^{Wx_{t}(\xi)} + \dot Z^{Wx_{t}(\xi)} -\sum_{s=0}^{t-1}\mathring{\chi}_{s}Z^{d\bar{h}_{s}}\EV Z^{x_{s}}Z^{x_{t}(\xi)}\\
        \{\hat Z^{Wx_{t}(\xi)}\}_{\xi, t} \ \ \text{centered, jointly Gaussian with}\ \
        \Cov(\hat Z^{Wx_{t}(\xi)},\hat Z^{Wx_{s}(\zeta)})=\EV Z^{x_{t}(\xi)}Z^{x_{s}(\zeta)}
    \end{gather*}
    and \cref{{eq:Zdx1}} should be replaced with
    \begin{gather*}
        Z^{dx_{t}} \defeq \hat Z^{W^\trsp d\bar{h}_{t}} + \dot Z^{W^\trsp d\bar{h}_{t}} -\sum_{s=0}^{t-1}\mathring{\chi}_{s}Z^{x_{s}}\EV Z^{d\bar{h}_{s}}Z^{d\bar{h}_{t}}\\
        \{\hat Z^{W^\trsp d\bar{h}_{t}}\}_{t} \ \ \text{centered, jointly Gaussian with}\ \
        \Cov(\hat Z^{W^\trsp d\bar{h}_{t}},\hat Z^{W^\trsp d\bar{h}_{s}})=\EV Z^{d\bar{h}_{t}}Z^{d\bar{h}_{s}}.
    \end{gather*}
    Like in \cref{thm:decoupledSGD2LP}, by definition,
    $\{\hat Z^{Wx_{t}(\xi)}\}_{\xi, t}$, $\{\hat Z^{W^\trsp d\bar{h}_{t}}\}_{t}$, $Z^{U_0}$, and $Z^{nV_0}$ are mutually independent sets of random variables.

    Here, $\dot Z^{Wx_t(\xi)} \defeq \sum_{r=0}^{t-1} \theta_r Z^{d\bar h_r}$ where $\theta_r$ is calculated like so:
    $Z^{x_t(\xi)}$ by definition is constructed as
    \[Z^{x_t(\xi)} = \Phi(\hat Z^{W^\trsp d\bar{h}_{0}}, \ldots, \hat Z^{W^\trsp d\bar{h}_{t-1}}, Z^{U_0})\]
    for some function%
    \footnote{that may depend on various scalars such as $\mathring \chi_s$, $\EV Z^{x_{s}}Z^{x_{s'}(\xi)}$, and $\EV Z^{d\bar{h}_{s}}Z^{d\bar{h}_{s'}}$}
    $\Phi: \R^{t+1} \to \R$.
    Then
    \[\theta_r \defeq \EV \partial \Phi(\hat Z^{W^\trsp d\bar{h}_{0}}, \ldots, \hat Z^{W^\trsp d\bar{h}_{t-1}}, Z^{U_0}) / \partial \hat Z^{W^\trsp d\bar{h}_{r}}.\]

    Likewise, $\dot Z^{W^\trsp d\bar{h}_{t}} \defeq \sum_{r=0}^{t-1} \theta_r Z^{x_r}$ where $\theta_r$ is calculated as follows:
    $Z^{d\bar{h}_{t}}$ by definition is constructed as
    \[Z^{d\bar{h}_{t}}=\Psi(\hat Z^{W x_{0}}, \ldots, \hat Z^{W x_{t-1}}, Z^{V_0})\]
    for some function\footnotemark[\value{footnote}]
    $\Psi: \R^{t+1} \to \R$.
    Then
    \[\theta_r \defeq \EV \partial \Psi(\hat Z^{W x_{0}}, \ldots, \hat Z^{W x_{t-1}}, Z^{V_0}) / \partial \hat Z^{W x_{r}}.\]

\end{thm}

For example, generalizing \cref{eqn:x11}, for any input $\xi$, we have
\[Z^{x_1(\xi)} = \Phi(Z^{W^\trsp d\bar h_0}, Z^{U_0}),\quad\text{where}\quad \Phi(z, u)\defeq \phi(\xi u - \mathring \chi_0 \xi_0 \xi \phi'(\xi_0 u) z).\]
Then $\theta_0 = \EV \partial_z \Phi(Z^{W^\trsp d\bar h_0}, Z^{U_0}) = -\mathring \chi_0 \xi_0 \xi \EV \phi'(Z^{h_1(\xi)}) \phi'(Z^{h_0})$, which specializes to $c$ in \cref{eq:ZWx}.
Altogether, $\dot Z^{Wx_1(\xi)} = -\mathring \chi_0 \xi_0 \xi Z^{d\bar h_0}\EV \phi'(Z^{h_1(\xi)}) \phi'(Z^{h_0})$.

Note that $\hat{Z}^{W x_{t}}$ here does not equal $Z^{W x_{t}}$ in \cref{eq:Zh2} in general, because the covariance $\Cov(\hat{Z}^{W x_{t}},\hat{Z}^{W x_{s}})=\EV Z^{x_{t}}Z^{x_{s}}$ is affected by the presence of $\dot{Z}^{W x_{r}}$ for all $r\le\max(s,t)$.

\subsection{MLP of Arbitrary Depth}

The $\mu$P limit of deeper MLPs can be derived along similar logic; see \cref{sec:Program-Setup,sec:ProgramConstruction,sec:infwidthlimit} for a rigorous treatment within the Tensor Programs framework, which also covers all stable abc-parametrizations.

\paragraph*{What happens in other feature learning parametrizations}

If we are in the feature learning regime, then any $W^l$ that is not maximally updated (\cref{defn:maximalupdate}) will be effectively fixed (to its initialized value) in the infinite-width limit (i.e.\ no learning occurs).

\subsection{Summary of Main Intuitions for Deriving the \texorpdfstring{$\mu$P}{MUP} Limit}

\begin{description}
\item[Law of Large Numbers] Any vector $z$ has roughly iid coordinates given by $Z^{z}$. For any two vectors $z,z'\in\R^{n}$, $\frac{1}{n}\sum_{\alpha=1}^{n}z_{\alpha}z'_{\alpha}\to\EV Z^{z}Z^{z'}$. 
    \begin{enumerate}
        \item This is all we needed to derive the 1-hidden-layer dynamics of \cref{sec:shallowMLP}, since all the matrices there are size-$n$ vectors.
        \item In \cref{{sec:decoupledSGD},sec:2LPSGD}, this is also used in calculating the limit of $\Delta W_t x_t$.
    \end{enumerate}
\item[Central Limit] If the underlying computation graph never involves the transpose $W^{\trsp}$ of a $n\times n$ Gaussian matrix $W$ in a matrix multiplication, then $Wz$ is roughly iid Gaussian with coordinate $Z^{Wz}\disteq\Gaus(0,\EV(Z^{z})^{2})$ (if $W_{\alpha\beta}\sim\Gaus(0,1/n)$) 
    \begin{enumerate}
        \item This along with the last intuition are all we used to derive the 2-hidden-layer decoupled dynamics of \cref{sec:decoupledSGD}, where $W$ is the middle layer weight matrix.
    \end{enumerate}
\item[($W$, $W^\trsp$) Correlation] If $W^\trsp$ is involved, then $Wz$ has coordinates distributed like random variable $\hat{Z}^{Wz}+\dot{Z}^{Wz}$ where $\hat{Z}^{Wz}$ is the Gaussian obtained by pretending $W$ is independent from $W^{\trsp}$, and $\dot{Z}^{Wz}$ results from the correlation between $W$ and $W^{\trsp}$. $\dot{Z}^{Wz}$ is purely a linear combination of $Z^{z'}$ for previously defined vectors $z'$ such that $z$ depends on $W^\trsp z'$.
\begin{enumerate}
    \item All three intuitions above are needed to derive the 2-hidden-layer dynamics of normal SGD (\cref{sec:2LPSGD}), where $W^\trsp$ is used in backpropagation.
    \item The calculation of $\dot Z^{Wx}$ is quite intricate, which is why we first discussed decoupled SGD in \cref{sec:decoupledSGD}, which doesn't need $\dot Z^{Wx}$ calculation, before discussing normal SGD in \cref{sec:2LPSGD}.
\end{enumerate}

\end{description}

\begin{figure}
    \centering
    \includegraphics[width=\textwidth]{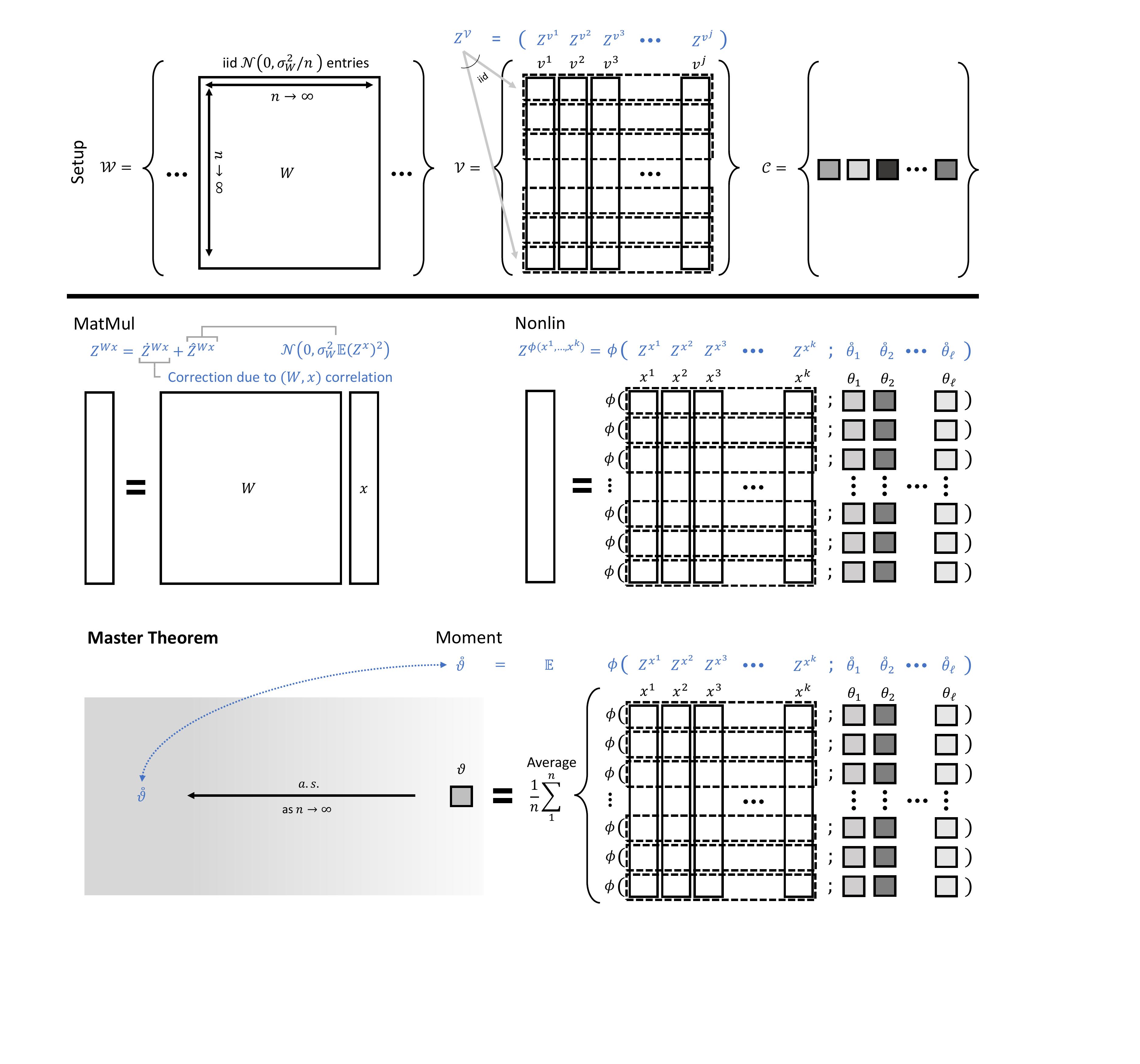}
    \caption{\textbf{Graphical overview of the Tensor Programs framework.}
    For the Master Theorem, we illustrate \cref{thm:PLNetsorT+MasterTheorem}(2) since \cref{thm:PLNetsorT+MasterTheorem}(1) is a corollary of \cref{thm:PLNetsorT+MasterTheorem}(2) for a larger program.}
    \label{fig:TP}
\end{figure}

\section{Tensor Programs Framework}
\label{sec:TensorPrograms}

While the previous section demonstrates the intuition of how to derive the $\mu$P limit, it also lays bare 1) the increasing complexity of a manual derivation as the training goes on, as well as 2) the mounting uncertainty for whether the intuition still holds after many steps of SGD.
This is a perfect call for the Tensor Programs framework, which automates (and makes rigorous) the limit derivation for any ``computation graph'' --- including the computation graph underlying SGD.
Here we review this framework (developed in \citet{scaling,TP1,TP2,TP3}) in the context of $\mu$P limit.
\cref{fig:TP} graphically overviews the content of this section.

As seen abundantly in \cref{sec:exactformulas}, the computation underlying SGD can be expressed purely via three instructions: matrix multiplication (by a Gaussian matrix, e.g.\ $W_0 x_0$), coordinatewise nonlinearities (e.g.\ $\phi$), and taking coordinatewise average (e.g.\ $\frac{1}{n}\sum_{\alpha=1}^{n}(nV_1)_{\alpha}x_{1\alpha}$).
In deriving the $\mu$P SGD limit, we focused mostly on keeping track of $\R^n$ vectors (e.g.\ $\bar x_t$ or $dh_t$), but importantly we also computed scalars $f_t$ and $\chi_t$ by (what amounts to) taking coordinatewise average (e.g.\ $f_1 = \frac{1}{n}\sum_{\alpha=1}^{n}(nV_{1})_{\alpha}x_{1\alpha}$).
We implicitly compute scalars as well inside $\Delta W_t x_t$.
This motivates the following notion of a \emph{program}, which can be thought of as a low-level symbolic representation of a computation graph common in deep learning (e.g.\ underlying \texttt{Tensorflow} and \texttt{Pytorch}).
\begin{defn}\label{defn:TensorProgram}
    A \emph{Tensor Program}%
    \footnote{What we refer to as Tensor Program is the same as \netsortplus{} in \citet{TP3}; we will not talk about other languages (like \netsort{}) so this should not cause any confusion} 
    is a sequence of $\R^{n}$-vectors and $\R$-scalars inductively generated via one of the following ways from an initial set $\mathcal C$ of random scalars, $\mathcal{V}$ of random $\R^{n}$ vectors, and a set $\mathcal{W}$ of random $\R^{n\times n}$ matrices (which will be sampled with iid Gaussian entries in \cref{setup:netsortplus})%
    \begin{description}
    \item [\texttt{MatMul\label{instr:matmul}}] Given $W\in\R^{n\times n}$ and $x\in\R^{n}$, we can generate $Wx\in\R^{n}$ or $W^{\trsp}x\in\R^{n}$ 
    \item [\texttt{Nonlin\label{instr:nonlin+}}] Given $\phi:\R^{k}\times\R^{l}\to\R$, previous scalars $\theta_{1},\ldots,\theta_{l}\in\R$ and vectors $x^{1},\ldots,x^{k}\in\R^{n}$, we can generate a new vector 
    \[
    \phi(x^{1},\ldots,x^{k};\theta_{1},\ldots,\theta_{l})\in\R^{n}
    \]
    where $\phi(-;\theta_{1},\ldots,\theta_{l})$ applies coordinatewise to each ``$\alpha$-slice'' $(x_{\alpha}^{1},\ldots,x_{\alpha}^{k})$. 
    \item [\texttt{Moment\label{instr:moment}}] Given same setup as above, we can also generate a new scalar 
    \[
    \f 1n\sum_{\alpha=1}^{n}\phi(x_{\alpha}^{1},\ldots,x_{\alpha}^{k};\theta_{1},\ldots,\theta_{l})\in\R.
    \]
    \end{description}
\end{defn}

\paragraph{Explanation of \cref{defn:TensorProgram}}
The \emph{vectors} mentioned in \cref{defn:TensorProgram} are exemplified by $h_t, x_t, dh_t, dx_t$ in \cref{sec:exactformulas}.
The \emph{scalars} mentioned are exemplified by $f_t, \chi_t$ as well as e.g.\ ${x_{s}^{\trsp}x_{t}}/{n}$ inside the calculating of $h_t$ (\cref{eqn:h2t}).
The $\theta_i$s in \refNonlinPlus{} and \refMoment{} rules may appear cryptic at first.
These scalars are not needed in the first forward and backward passes.
But in the second forward pass, for example for the 1-hidden-layer MLP (\cref{sec:shallowMLP}), $x_1 = \phi(h_1) = \phi(\xi_1 U_0 - \chi_0 \xi_1 \xi_0 nV_0 \phi'(h_0))$ depends on the scalar $\chi_0, \xi_0, \xi_1$, and can be written in the form of \refNonlinPlus{} as $\bar\phi(U_0, nV_0, h_0; \chi_0)$ for some $\bar\phi$ appropriately defined.

The \emph{initial set of scalars} $\mathcal C$ is the training sequence $\{\xi_t, y_t\}_t$ for all three examples of \cref{sec:exactformulas}.
In our 2-hidden-layer MLP examples, the \emph{initial set of matrices} $\mathcal W$ is $\{W\}$ (\cref{sec:2LPSGD}) or $\{W, \widetilde W\}$ (\cref{sec:decoupledSGD}), i.e.\ the random $\R^{n\times n}$ Gaussian matrices.
On the other hand, in the 1-hidden-layer MLP example (\cref{sec:shallowMLP}), $\mathcal{W}$ is empty.
The \emph{initial set of vectors} $\mathcal V$ in all three examples are $\mathcal V = \{U_0, nV_0\}$.\footnote{Here we write $nV_0$ instead of $V_0$ because we want all vectors to have $\Theta(1)$ coordinates; see \cref{setup:netsortplus}.}%
\footnote{In \cref{sec:exactformulas} we assumed input dimension is 1. In general, each column of $U_0$ would be a separate initial vector. Likewise, if the output dimension is greater than 1, then each row of $V_0$ would be a separate initial vector.}
Notice how the vectors of these $\mathcal V$ are sampled with iid standard Gaussian coordinates.
We formalize a more general setup for arbitrary Tensor Programs:
\begin{setup}\label{setup:netsortplus}
    1) For each initial $W\in\mathcal{W}$, we sample iid $W_{\alpha\beta}\sim\Gaus(0,\sigma_{W}^{2}/n)$ for some variance $\sigma_{W}^{2}$ associated to $W$, independent of other $W' \in \mathcal {W}$; 2) for some multivariate Gaussian $Z^{\mathcal{V}}=\left\{ Z^{h}:h\in\mathcal{V}\right\} \in\R^{\mathcal{V}}$, we sample the initial set of vectors $\mathcal{V}$ like $\left\{ h_{\alpha}:h\in\mathcal{V}\right\} \sim Z^{\mathcal{V}}$ iid for each $\alpha\in[n]$.
    3) For each initial scalar $\theta \in \mathcal C$, we require $\theta \asto \mathring \theta$ for some deterministic $\mathring \theta \in \R$.
\end{setup}
In all of our examples, we took $\sigma_W^2 = 1$ for simplicity, but \cref{setup:netsortplus} allows for other initializations (e.g.\ a typical initialization for relu networks is $\sigma_W^2 = 2$); additionally, $Z^h, h\in \mathcal V$, are all standard Gaussians, independent from one another, since $U_0, nV_0$ are sampled this way; and our initial scalars $\{\xi_t, y_t\}_t$ are fixed with $n$, so they are their own limits.%
\footnote{Since $\{\xi_t, y_t\}_t$ are fixed with $n$, we can WLOG absorb them into any nonlinearities in \refNonlinPlus{} that they are involved in, and set $\mathcal C = \emptyset{}$.
But, in kernel regime or nonmaximal feature learning parametrization, we usually have initial scalars, such as $n^{-2a_{L+1}-c}$, that tend to 0 with $n$; see \cref{sec:ProgramConstruction}.}

\paragraph{What Does a Tensor Program Vector Look Like?}
Recall that we represented the coordinate distribution of each vector $h$ with a random variable $Z^h$ in \cref{sec:exactformulas} and kept track of how different $Z$s are correlated with each other.
We also calculated scalar limits like $f_t \to \mathring f_t, \chi_t \to \mathring \chi_t$.
These calculations led to a set of formulas for the $\mu$P limit (e.g.\ \cref{thm:1LPMUPLimit,thm:decoupledSGD2LP,thm:SGD2LP}).
We can also construct such $Z^h$ and $\mathring \theta$ for vectors $h$ and scalars $\theta$ in any Tensor Program.
They intuitively capture the coordinate distribution of vector $h$ and the deterministic limit of $\theta$.
The following definition formally defines $Z^h$ and $\mathring \theta$, but the connection between $Z^h$ (resp. $\mathring \theta$) and the coordinates of $h$ (resp. $\theta$) is not made rigorously until \cref{thm:PLNetsorT+MasterTheorem} later.
The \refZMatMul{} rule below perhaps asks for some discussion, and we shall do so after the definition.

\begin{defn}[$Z^h$ and $\mathring \theta$]
    \label{defn:netsortplusKeyIntuit}
Given a Tensor Program, we recursively define $Z^{h}$ for each vector $h$ and $\mathring{\theta}$ for each scalar $\theta$ as follows.
\begin{description}
\item [\texttt{ZInit}]
If $h\in\mathcal{V}$, then $Z^{h}$ is defined as in \cref{setup:netsortplus}. We also set $\hat{Z}^{h}\defeq Z^{h}$ and $\Zdot^{h}\defeq 0$.
\item [\texttt{ZNonlin$^+$}] Given $\phi:\R^{k}\times\R^{l}\to\R$, previous scalars $\theta_{1},\ldots,\theta_{l}\in\R$ and vectors $x^{1},\ldots,x^{k}\in\R^{n}$, we have
\[
Z^{\phi(x^{1},\ldots,x^{k};\theta_{1},\ldots,\theta_{l})}\defeq\phi(Z^{x^{1}},\ldots,Z^{x^{k}};\mathring{\theta}_{1},\ldots,\mathring{\theta}_{l}).
\]
\item [\texttt{ZMoment}] Given same setup as above and scalar $\theta = \f 1n\sum_{\alpha=1}^{n}\phi(x_{\alpha}^{1},\ldots,x_{\alpha}^{k};\theta_{1},\ldots,\theta_{l})$, then
\[
\mathring{\theta}\defeq\EV\phi(Z^{x^{1}},\ldots,Z^{x^{k}};\mathring{\theta}_{1},\ldots,\mathring{\theta}_{l}).
\]
Here $\mathring{\theta}_{1},\ldots,\mathring{\theta}_{l}$ are deterministic, so the expectation is taken over $Z^{x^{1}},\ldots,Z^{x^{k}}$.

\item [\texttt{ZMatMul}\label{Z:MatMul}] $Z^{Wx}\defeq \hat{Z}^{Wx}+\Zdot^{Wx}$
    for every matrix $W$ (with $\Gaus(0,\sigma_{W}^{2}/n)$ entries) and vector $x$, where
    \begin{description}
    \item [\texttt{ZHat}\label{Z:Zhat}] {$\hat{Z}^{Wx}$} is a Gaussian variable with zero mean. Let $\mathcal V_W$ denote the set of all vectors in the program of the form $W y$ for some $y$.
    Then $\{\hat Z^{W y}: W y \in \mathcal V_W\}$ is defined to be jointly Gaussian with zero mean and covariance
    \[
    \Cov\left(\hat{Z}^{Wx},\hat{Z}^{W y}\right)\defeq \sigma_{W}^{2}\EV Z^{x}Z^{y},\quad\text{for any \ensuremath{Wx,W y\in\mathcal{V}_W}.}
    \]
    Furthermore, $\{\hat Z^{W y}: W y \in \mathcal V_W\}$ is mutually independent from $\{\hat Z^{v }: v \in \mathcal V \cup \bigcup_{\bar W \ne W} \mathcal V_{\bar W}\}$, where $\bar{W}$ ranges over $\mathcal{W}\cup\{A^{\trsp}:A\in\mathcal{W}\}$.
    \item [\texttt{ZDot}\label{Z:Zdot}]
    We can always unwind $Z^x = \Phi(\cdots)$, for some arguments $(\cdots)=(\{\hat Z^{W^\trsp y^i}\}_{i=1}^k, \{\hat Z^{z^i}\}_{i=1}^j; \{\mathring \theta_i\}_{i=1}^l)$, $z^i \not\in \mathcal V_{W^\trsp}$ (where $\mathcal V_{W^\trsp}$ is defined in \refZhat{}), and deterministic function $\Phi: \R^{k+j+l} \to \R$.
    Define $\partial Z^x / \partial \hat Z^{W^\trsp y^i} \defeq \partial_i \Phi(\cdots)$. 
    Then we set
    \begin{equation}
        \Zdot^{Wx}\defeq \sigma_{W}^{2}\sum_{i=1}^k Z^{y^i}\EV \f{\partial Z^x} {\partial \hat Z^{W^\trsp y^i}},\label{eqn:Zdot}
    \end{equation}
    There is some nuance in this definition, so see \cref{rem:PartialDer} and \ref{rem:ExpectationPartialDer}.
    \end{description}
\end{description}
\end{defn}

\paragraph{Explanation of \cref{defn:netsortplusKeyIntuit}}
\refNonlinPlus{} and \refMoment{} should appear only natural.
However, we pause to digest the meaning of \refZMatMul{} by relating back to our examples in \cref{sec:exactformulas}.
First notice that $\dot Z^{Wx} = 0$ if $W^\trsp$ is not used in the program, so that $Z^{Wx} = \hat Z^{Wx}$.
This is the case in \cref{sec:decoupledSGD}, where $\widetilde W$ is used in backprop instead of $W^{\trsp}$.
There (in \cref{eq:Zh2}), $Z^{W x_{t}}$ is Gaussian with covariance $\Cov(Z^{W x_{t}},Z^{W x_{s}})=\EV Z^{x_{t}}Z^{x_{s}}$ for any $s$, consistent with \refZhat{}.
In \cref{sec:2LPSGD}, however, $\dot Z^{Wx} \ne 0$ in general.
The \refZdot{} rule is a direct generalization of the calculation of $\dot Z$ in \cref{thm:SGD2LP}.

\paragraph{$\dot Z^{W x_t}$ and $\dot Z^{W^{\trsp} d\bar h_t}$ of \cref{sec:2LPSGD} for general $t$}
will all be nonzero but have no easy expression.
Here we seek to convey the complexity of computing them; this is optional reading for the first time reader.
To calculate $\dot Z^{W x _t}$ ($\dot Z^{W^{\trsp} d\bar h_t}$ is similar), we need to express $Z^{x_t}$ as a function of purely $\hat Z^{W^{\trsp} d\bar h_s}, s<t,$ and $Z^{U_0} = \hat Z^{U_0}$.
Then we symbolically differentiate $Z^{x_t}$ by $\hat Z^{W^{\trsp} d\bar h_s}$ and take expectation to obtain the coefficient of $Z^{d\bar h_s}$ in $\dot Z^{W x _t}$.
For $t=1$ as in the examples in \cref{sec:2LPSGD}, this task is easy because $\hat Z^{W^{\trsp} d\bar h_0} = \hat Z^{dx_0} = Z^{dx_0}$.
But in general, the calculation can balloon quickly.
Indeed, note $Z^{x_t} = \phi(Z^{h_t})$ and
\[Z^{h_{t}}=\xi_{t}Z^{U_{t}} = \xi_t Z^{U_0} - \xi_t \sum_{s=0}^{t-1} \mathring \chi_s \xi_s Z^{dh_s}
= \xi_t Z^{U_0} - \xi_t \sum_{s=0}^{t-1} \mathring \chi_s \xi_s \phi'(Z^{h_s})Z^{dx_s}.\]
However, each $Z^{dx_s}$ is a linear combination of $Z^{W^{\trsp} d\bar h_s} = \hat Z^{W^{\trsp} d\bar h_s} + \dot Z^{W^{\trsp} d\bar h_s}$ and $Z^{x_r}, r < s$ (coming from $\Delta W_t^\trsp d\bar h_s$).
Each of $\dot Z^{W^{\trsp} d\bar h_s}$ and $Z^{x_r}$ then needs to be recursively expanded in terms of $\hat Z$ before we can calculate the symbolic partial derivative $\partial Z^{x_t} / \partial \hat Z^{W^{\trsp} d\bar h_s}$.

\paragraph{Master Theorem}
Finally, we relate the \emph{symbolic} nature of a Tensor Program given in \cref{defn:netsortplusKeyIntuit} to the \emph{analytic} limit of its computation, in the following \emph{Master Theorem}.
Pseudo-Lipschitz functions are, roughly speaking, functions whose (weak) derivatives are polynomially bounded.
We state the theorem assuming mild regularity conditions (\cref{assm:MasterTheoremSmoothness}) that roughly says most nonlinearities in the program should be pseudo-Lipschitz.
\begin{thm}[Tensor Program Master Theorem, c.f.\ Theorem E.15 of \citep{TP3}]
    \label{thm:PLNetsorT+MasterTheorem} Fix a Tensor Program initialized accordingly to \cref{setup:netsortplus}.
    Adopt \cref{assm:MasterTheoremSmoothness}.
    Then 
    \begin{enumerate}
    \item For any fixed $k$ and any pseudo-Lipschitz $\psi:\R^{k}\to\R$, as $n\to\infty$,
    \begin{equation}
    \f 1n\sum_{\alpha=1}^{n}\psi(h_{\alpha}^{1},\ldots,h_{\alpha}^{k})\asto\EV\psi(Z^{h^{1}},\ldots,Z^{h^{k}}),
    \label{eqn:mastertheorem}
    \end{equation}
    for any vectors $h^{1},\ldots,h^{k}$ in the program, where $Z^{h^{i}}$ are as defined in \cref{{defn:netsortplusKeyIntuit}}.
    \item Any scalar $\theta$ in the program tends to $\mathring{\theta}$ almost surely, where $\mathring{\theta}$ is as defined in \cref{{defn:netsortplusKeyIntuit}}.
    \end{enumerate}
\end{thm}

Intuitively, \cref{thm:PLNetsorT+MasterTheorem}(1) says that each ``coordinate slice'' $(h^1_\alpha, \ldots, h^k_\alpha)$ can be thought of as an iid copy of $(Z^{h^1},\ldots, Z^{h^k})$.%
\footnote{This implies an explicit convergence in distribution (see \cite{TP3}), but this convergence in distribution is strictly weaker than the formulation in \cref{thm:PLNetsorT+MasterTheorem}, which is in general much more useful.}
This intuition is consistent with our heuristic derivation in \cref{sec:exactformulas}, and \cref{thm:PLNetsorT+MasterTheorem} underlies the proof of \cref{thm:1LPMUPLimit,thm:decoupledSGD2LP,thm:SGD2LP}.
\cref{thm:PLNetsorT+MasterTheorem}(2) allows us to directly obtain the function learned at the end of training:
For example, for a 1-hidden-layer MLP, it shows that the network's output on any input $\xi$ at time $t$ converges to $\mathring f_t(\xi)$ given in \cref{thm:1LPMUPLimit}.

\cref{alg:infwidthlimit} summarizes how to compute the infinite-width limit of any network in any abc-parametrization and for any task, using the Tensor Programs framework laid out in this section.
It generalizes the manual derivations of \cref{sec:exactformulas}.
We carry out \cref{alg:infwidthlimit} for MLPs in all of our experiments.

\begin{algorithm}[tb]
    \caption{Compute the infinite-width limit of an NN in any abc-parametrization and any task}
    \label{alg:infwidthlimit}
    \begin{algorithmic}[1]
      \State Write the computation graph underlying training and inference in a Tensor Program (akin to writing low level PyTorch or Tensorflow code).
      \State Calculate $Z^h$ for each vector $h$ and $\mathring \theta$ for each scalar $\theta$ in the program, according to \cref{{defn:netsortplusKeyIntuit}}.
      \State The logits $f_t(\xi)$ of the neural network at any time $t$ should be written as a collection of scalars, so $\mathring f_t(\xi)$ is calculated in the previous step. For $t$ being inference time, $\mathring f_t(\xi)$ is the output of the infinite-width network after training.
    \end{algorithmic}
\end{algorithm}

\paragraph{Architectural and algorithmic universality}
Given that Tensor Programs can express the first forward and backward computation of practically any architecture \cite{TP1,TP2}, it should perhaps come as no surprise that they can also express practically any training and inference procedure --- or just any computation --- involving any such architecture.
This includes both feature learning and kernel limits.
We leverage this flexibility to derive and compute the $\mu$P and kernel limits for metalearning and Word2Vec; see \cref{sec:experiments}.

\paragraph{Extensions}
We focused on programs whose vectors all have the same dimension $n$ here.
But it's easy to generalize to the case where vectors have different dimensions, which corresponds to e.g.\ when a network's widths are non-uniform.
See \citep{TP3}.

\section{Computational Considerations}
\label{sec:comp}

While the TP framework is very general, computing the feature learning limits analytically is inherently computationally intensive aside from special cases like the linear 1-hidden-layer MLP (\cref{cor:lin1LP}). Here we explain why, so as to motivate our experimental choices below.

\paragraph*{No closed-form formula for evaluating the expectations (e.g.\ in \cref{eqn:mastertheorem}) involving general nonlinearities except in special cases}

For example, for a 1-hidden-layer MLP (\cref{sec:shallowMLP}), after 1 step of SGD, the logit is of the form $\EV (Z_1 + b \phi(Z_2))\phi(Z_3 + c Z_1\phi'(Z_2))$ where $Z_i$s denote different (correlated) Gaussians (\cref{eqn:1LPf1}).
While one can still evaluate this via Monte-Carlo, the error will compound quickly with training time.
On the other hand, because of the nesting of $\phi'$ inside $\phi$, there is no closed-form formula for this expectation in general.

\emph{Notable Exception}: If the nonlinearity $\phi$ is polynomial, then the expectation is a polynomial moment of a multivariate Gaussian and can be evaluated analytically, e.g.\ using Isserlis' theorem from the covariance matrix.%

\paragraph*{Even with nonlinear polynomial $\phi$, there is exponential computational bottleneck}

As training time $t$ increases, due to the nesting of $\phi$ and $\phi'$ in the preactivations, the integrand of the expectation, e.g.\ $\EV Z^{\bar x_{t}}Z^{nV_{t}}$, will turn out to be a polynomial in $\Omega(1)$ Gaussian variables with degree $\Omega(2^{t})$. The covariance matrix of the Gaussian variables will in general be nontrivial, so evaluating the expectation, e.g.\  using Isserlis' theorem, requires super-exponential time. This is because we would need to expand the polynomial integrand into monomials, and there would be $\Omega(2^{t})$ monomials, each of which require $\Omega(2^{t})$ time to evaluate using Isserlis' theorem.

\paragraph*{$n\times n$ Gaussian matrices}

Both points above apply to 1-hidden-layer MLPs. Additional difficulties with deeper networks is caused by the $n\times n$ initial Gaussian matrix $W_{0}^{l},2\le l\le L$, in the middle of the network. 1) In general, due to the nonlinearities, $x_{t}^{l-1}$ would be linearly independent from $x_{s}^{l-1}$ for all $s<t$. Therefore, in calculating $W_{t}^{l}x_{t}^{l-1}=W_{0}^{l}x_{t}^{l-1}+\Delta W_{t}^{l}x_{t}^{l-1}$, we create a new Gaussian variable $\hat{Z}^{W_{0}^{l}x_{t}^{l-1}}$ linearly independent from all previous $\hat{Z}^{W_{0}^{l}x_{s}^{l-1}},s<t$. This then requires us to compute and store the covariance between them. Thus, $t$ steps of SGD costs $\Omega(t^{2})$ space and time (not mentioning that the computation of each covariance entry can require exponential time, as discussed above).
2) In addition, due to the interaction between $W_{t}^{l}$ in the forward pass and $W_{t}^{l\trsp}$ in the backward pass, there is nonzero $\dot{Z}$, as demonstrated in \cref{eq:ZWx}. %
This $\dot{Z}$ is generally a linear combination of $\Omega(t)$ terms, and the coefficients of this combination require evaluation of some expectations that typically run into the exponential bottleneck discussed above.%

\paragraph{Summary}
From easiest to hardest in terms of $\mu$P limit's computational cost, we have 
1) 1-hidden-layer linear networks;
2) $L$-hidden-layer linear MLP, $L \ge 2$;
3) nonlinear MLP with polynomial activations;
4) nonlinear MLP with nonpolynomial activations.
Nevertheless, 1-hidden-layer linear networks are more than sufficient to demonstrate feature learning in Word2Vec and few-shot learning with MAML, as we show below.

\section{Experiments}
\label{sec:experiments}
In light of the computational difficulties discussed above, we divide our experiments into two groups: 1) Verifying our theory; 2) Scaling up to realistic datasets to demonstrate feature learning. The experiments in group 1 focus on stress-testing our theory in many scenarios to show that it describes empirical phenomena accurately. They will run into the discussed computational difficulties (\cref{sec:comp}), so we cannot train the infinite-width $\mu$P networks for very long, but nevertheless long enough to verify the theory. Those in group 2 focus on real datasets (metalearning and Word2Vec) where feature learning is critical, and demonstrate that the GP and NTK limits are inadequate for those tasks. Necessarily, we adopt simpler neural architectures for this purpose so we can scale up.

\subsection{Verifying the Theory}

In \cref{fig:toyTP}, we analytically computed the $\mu$P limits derived in \cref{sec:exactformulas} for quadratic and linear activations, and verified them against finite width networks.

\begin{figure}
    \centering
    \includegraphics[width=0.8\textwidth]{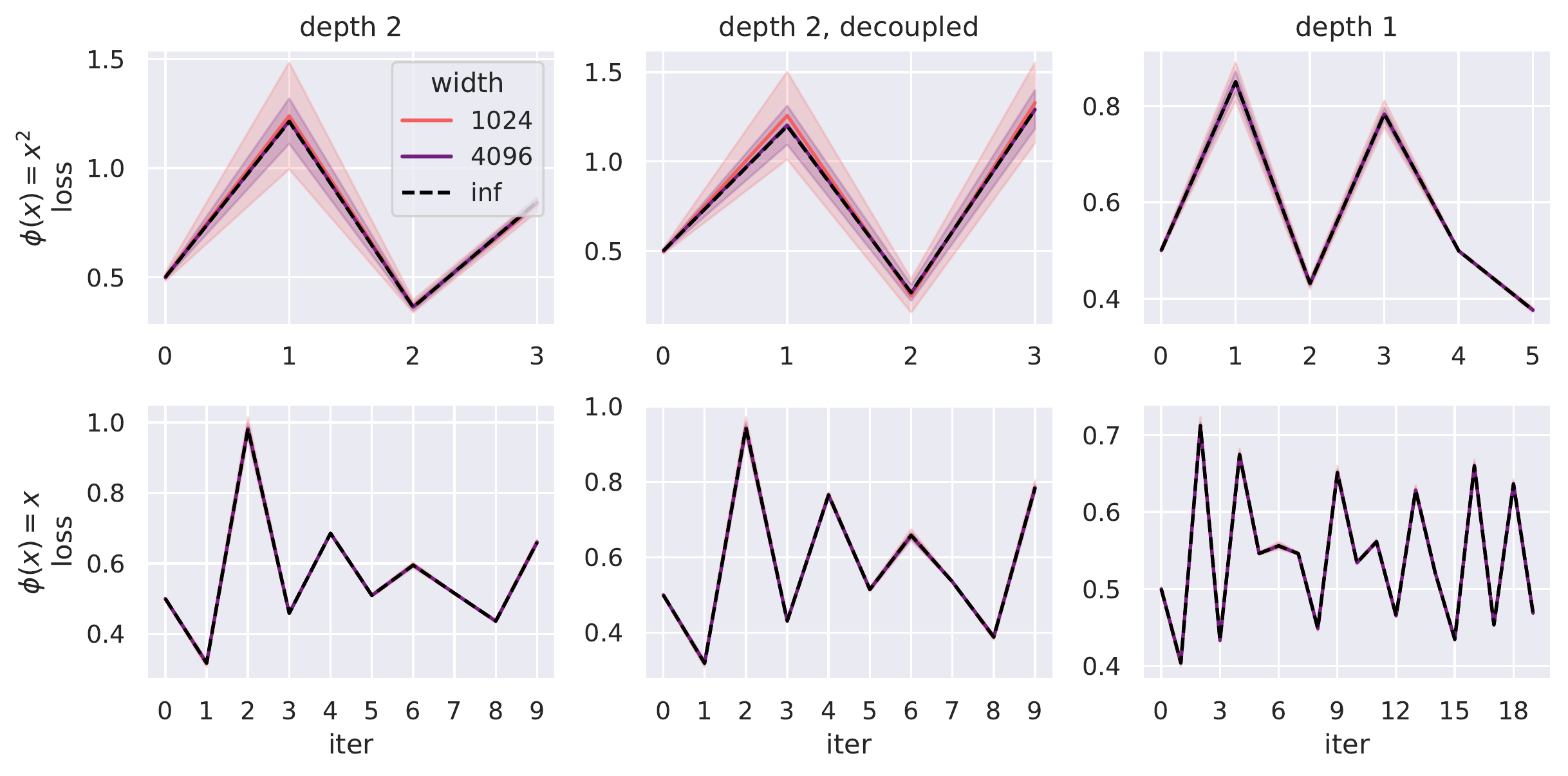}
    \caption{\textbf{Empirical Simulation Agrees with Theory.}
    We analytically compute the infinite-width $\mu$P limit for the three kinds of networks (depth 1, depth 2 decoupled, depth 2) described in \cref{sec:exactformulas}, with either quadratic $\phi(x)=x^2$ or linear $\phi(x)=x$ activation.
    The training set is random $\xi_t \in \{\pm 1\}, y_t \in \{\pm 1\}$, so that the deviation of finite width from infinite width losses are accentuated.
    We compare against finite width $\mu$P networks with width 1024 or 4096.
    For each width, we randomly initialize with 100 different seeds and aggregate the loss curves.
    The mean across these seeds is plotted as solid curves, and the standard deviation represented by the shade.
    As discussed in \cref{sec:comp}, nonlinear activation functions and higher depth face computational difficulties exponential with training time.
    Thus here we only train for a few steps.
    We observe that the quadratic network converges slower to the limit with width.
    This is expected since the tail of $Z^{x_t}$ is fatter for a quadratic activation than a linear activation.}
    \label{fig:toyTP}
\end{figure}

\subsection{Few-Shot Learning on Omniglot via First Order MAML}

In few-shot learning, the model is given only a small number of labeled examples before asking to make predictions on unseen data.
Therefore, this tests whether a model contains a good \emph{prior} that can adapt quickly to the small amount of data at hand.

\paragraph{MAML}
In Model Agnostic Meta-Learning (MAML), the model performs few-shot learning by one or more SGD steps on the given training data; this is called \emph{adaptation}.
In a pretraining (also called \emph{meta-training}) phase, MAML learns a \emph{good initialization} of the model parameters for this adaptation.
The training objective is to minimize the loss on a random task's test set after the model has adapted to its training set.
More precisely, the basic \emph{First Order} MAML at training time goes as follows:
With $f_\theta$ denoting the model with parameters $\theta$, and with step sizes $\epsilon, \eta$, we do
\begin{enumerate}
    \item At each time point, sample a few-shot task $\mathcal T$
    \item From $\mathcal T$, sample a training set $\mathcal D$
    \item Adapt $\theta' \gets \theta - \epsilon \nabla_\theta \loss_{\mathcal D}(f_\theta)$, where $\loss_{\mathcal D}(f_\theta)$ is the loss of $f_\theta$ over $\mathcal D$ \label{item:FOMAMLadapt}
    \item Sample a test set $\mathcal D'$ from $\mathcal T$
    \item Update $\theta \gets \theta - \eta \nabla_{\theta'} \loss_{\mathcal D'}(f_{\theta'})$, where $\loss_{\mathcal D'}(f_{\theta'})$ is the loss of $f_{\theta'}$ over $\mathcal D'$ \label{item:FOMAMLUpdate}
    \item Repeat
\end{enumerate}
In practice, we batch the tasks, just like batches in SGD, so that we accumulate all the gradients from Step \ref{item:FOMAMLUpdate} and update $\theta$ only at the end of the batch.

During \emph{meta-test} time, we are tested on random unseen few-shot tasks, where each task $\mathcal T$ provides a training set $\mathcal D$ and a test set $\mathcal D'$ as during meta-training.
We adapt to $\mathcal D$ as in Step 3 above (or more generally we can take  multiple gradient steps to adapt better) to obtain adapted parameters $\theta'$.
Finally, we calculate the accuracy of $\theta'$ on the test set $\mathcal D$.
We average this accuracy over many tasks $\mathcal T$, which we report as the \emph{meta-test accuracy}.

\paragraph{First Order vs Second Order MAML}
Notice in Step \ref{item:FOMAMLUpdate}, we take the gradient of $\loss_{\mathcal D'}(f_{\theta'})$ with respect to the adapted parameters $\theta'$.
In \emph{Second Order} MAML, we would instead take the gradient against the unadapted parameters $\theta$, which would involve the Hessian $\nabla_\theta \nabla_\theta \loss_{\mathcal D}(f_\theta)$.
Second Order MAML generally achieves performance slightly better than First Order MAML, but at the cost of significantly slower updates \citep{reptile2018}.
In order to scale up, we will focus on First Order MAML, hereafter referred to as just MAML.

\paragraph{Few-Shot Learning Terminologies}
An $N$-way classification task asks the model to predict a class from $N$ possiblities.
A $K$-shot classification task provides $K$ input/output pairs per class, for a total of $NK$ training points for $N$-way classification.

\paragraph{Omniglot}
Omniglot is a standard few-shot learning benchmark.
It consists of 20 instances of 1623 characters from 50 different alphabets, each handwritten by a different person.
We test our models on 1-shot 5-way classification: We draw 5 random characters, along with 1 training instance and 1 test instance for each character.
After the model adapts to the training instances, it's asked to predict the character of the test instances (choosing among the 5 characters).

\paragraph{Models}
Our main model is the $\mu$P limit of a 1-hidden-layer linear MLP.
We compare against:
1) finite width versions of the same;%
\footnote{Because we will tune initialization variances, our results also represent finite-width SP networks.}
2) the NNGP and NTK limits of the same;
3) the NNGP and NTK limits of a 1-hidden-layer relu MLP.
Note 2) is equivalent to a 0-hidden-layer perceptron, because the NNGP and NTK there are both linear kernels.
In addition, the infinite-width SP limit of a 1-hidden-layer network is the same as the NNGP limit.
Both 2) and 3) are equivalent to linear models with fixed (not learned) features, so MAML's adaptation only applies to the linear weights.
On the other hand, the $\mu$P limit and the finite $\mu$P networks will learn new representations of the data over time that can quickly adapt to new tasks.%
\footnote{
    Note that the transfer learning comment in \cref{sec:motivating} does not apply directly to the few-shot setting here, because the readout weights of the network carry over from the pretraining phase.
    Nevertheless, we will see a large performance gap between the kernel limits (2,3) and the $\mu$P limit.
}

\begin{table}
    \centering %
    \caption{\textbf{Omniglot Meta-Test Accuracies after Pretraining with First Order MAML.}}
    \label{tab:omniglottest}
    \begin{tabular}{cccccccccccc}
    \toprule 
    \multicolumn{2}{c}{$\phi=$ relu} & \multicolumn{9}{c}{$\phi=$ identity ; number = $\log_{2}width$} & \tabularnewline
    \midrule 
    GP  & NTK  & 1  & 3  & 5  & 7  & 9  & 11  & 13  & $\mu$P & GP/NTK & \tabularnewline
    \midrule 
    47.60  & 47.82  & 55.34  & 64.54  & 66.21  & 66.31  & 66.43  & 66.36  & 66.41  & 66.42  & 41.68 & \\
    $\pm$.02  & $\pm$.04  & $\pm$1.24  & $\pm$0.70  & $\pm$.15  & $\pm$.16  & $\pm$.23  & $\pm$.22  & $\pm$.18  & $\pm$.19  & $\pm$.09 & \tabularnewline
    \bottomrule
    \end{tabular}

    \end{table}

\paragraph{Hyperparameters}
We use (task) batch size 32 and adaptation step size 0.4 ($\epsilon$ in Step \ref{item:FOMAMLadapt}).
We also clip the gradient in Step \ref{item:FOMAMLUpdate} if the gradient has norm $\ge 0.5$.%
\footnote{One can write down gradient clipping easily in a Tensor Program, so the its infinite-width limit can be computed straightforwardly via \cref{thm:PLNetsorT+MasterTheorem}; see \cref{sec:expdetails}.}
For each model, we tune its weight initializaton variances and the meta learning rate ($\eta$ in Step \ref{item:FOMAMLUpdate}).
During meta-test time, we take 20 gradient steps during adaptation (i.e.\ we loop Step 3 above 20 times to obtain $\theta'$).
See \cref{sec:mamldetails} for more details.

\begin{wrapfigure}{r}{0.37\textwidth}
    \vspace{-34pt}
    \begin{center}
        \includegraphics[width=0.37\textwidth]{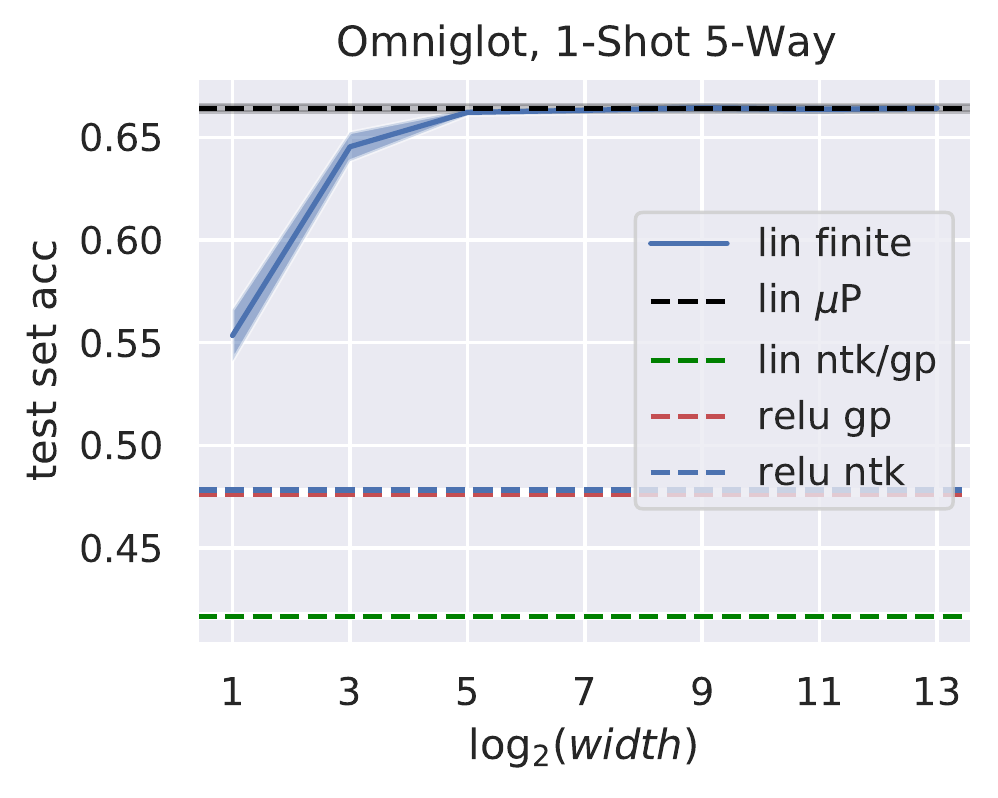}
    \end{center}
    \vspace{-45pt}
\end{wrapfigure}

\paragraph{Findings}
Our results are summarized in the \textcolor{purple}{\textbf{Figure to the right}} and \cref{tab:omniglottest}, where curves indicate means and shades indicate standard deviations.
There are three key takeaways:
1) The feature learning $\mu$P limit significantly outperforms the kernel limits.
2) The benefit of feature learning dominates the benefit of having nonlinearities.
3) As width increases, the finite $\mu$P networks approach the performance of the $\mu$P limit from below.

\subsection{Word2Vec}

\newcommand{\VV}{\mathcal{V}}

Word2Vec \citep{mikolov_distributed_2013,mikolov_efficient_2013} is an early example of large-scale pretraining and transfer learning in natural
language processing, where one learns a feature vector $h(\xi)$ for
every word $\xi$ based on the principle of distributional semantics.
For simplicity, we focus on a specific scheme of Word2Vec using context
as a bag-of-word (CBOW), negative example sampling, and Sigmoid loss
function.

\paragraph{Word2Vec Pretraining}
Consider training on a corpus with vocabulary $\VV$. 
At each time step, we sample a sentence for the corpus and choose a word $i \in \VV$.
This word's context $J \sbe \VV$ is a window of words around it in the sentence, thought of as a bag of words.
Let $\xi^i \in \R^{|\VV|}$ be the one-hot vector corresponding to word $i$.
We pass the averaged context $\xi^{J}\defeq\frac{1}{|J|}\sum_{j\in J}^{n}\xi^{j}$
through a 1-hidden-layer MLP with hidden size $n$ and identity activation:
\begin{equation}
f(\xi^{J})=V h(\xi^{J})\in \R^{|\VV|},\quad h(\xi^{J})=U\xi^{J} \in \R^{n},\label{eq:W2Vfwd}
\end{equation}
where $V\in\mathbb{R}^{|\VV|\times n}, U\in\mathbb{R}^{n\times|\VV|}$ factor as $V = n^{-a_v} v, U = n^{-a_u} u$ with initialization $v_\alpha \sim \Gaus(0, n^{-2b_v}), u_\alpha \sim \Gaus(0, n^{-2b_u})$, where $\{a_v, b_v, a_u, b_u\}$ specify the parametrization of the network.
After each forward pass, we sample a target word $\tau$ from $\VV$: with probability $p$, we take $\tau = i$; with probability $1-p$, we sample $\tau$ uniformly from $\VV \setminus \{i\}$.
Following \citep{mikolov_distributed_2013,mikolov_efficient_2013}, we take $p=1/21 \approx 4.76\%$.
The loss is then calculated with the Sigmoid function $\sigma(\cdot):$
\begin{equation}
\loss(f(\xi^{J}),\xi^{\tau})=\begin{cases}
\log(1-\sigma(f(\xi^{J})^{\top}\xi^{\tau})) & \tau=i\\
\log\sigma(f(\xi^{J})^{\top}\xi^{\tau}) & \tau\neq i
\end{cases}\label{eq:W2Vloss}
\end{equation}
Then $v$ and $u$ are updated via SGD as usual (causing $V$ and $U$ to update).
Conventionally, $h(\xi)\in\mathbb{R}^{n}$ is taken as the Word2Vec
embedding for a word $\xi$ after many iterations of forward-backward
updates.

\paragraph{Word Analogy Evaluation}
We evaluate the word embeddings $h(\xi)$ with the word analogy task.
This task asks the question of the kind: \emph{What to a `queen' is as a `man' to a `woman'?} (answer is `king').
The Word2Vec model answers this question by computing
\begin{equation}
    \argmax_i h(\xi^i)^\trsp (h(\xi^{\text{`man'}}) - h(\xi^{\text{`woman'}}) + h(\xi^{\text{`queen'}}))
    \label{eqn:wordanalogy}
\end{equation}
where $i$ ranges over $\VV \setminus \{\text{`man', `woman', `queen'}\}$.
If the argmax here is $i = \text{`king'}$, then the model answers correctly; otherwise, it's incorrect.
The accuracy score is the percentage of such questions answered correctly.

\paragraph{Dataset}

We train the models on \texttt{text8},%
\footnote{\url{http://mattmahoney.net/dc/textdata.html}}
a clean dataset consisting of the first 100 million characters of a 2006 Wikipedia dump.
The dataset has been featured in the original Word2Vec codebase and the Hutter Prize.
\texttt{text8} contains the first 100 million characters of \texttt{fil9}, a larger dataset obtained by filtering the first 1 billion characters in the aforementioned Wikipedia dump.
We space-separate the datasets into tokens and keep ones that appear no less than 5 times in the entire dataset for \texttt{text8} and 10 times for \texttt{fil9}.
The resulting datasets have 71,291 and 142,276 unique vocabulary items.

\paragraph{Models}
Our main model is the $\mu$P limit of \cref{eq:W2Vfwd}.
We compare against the baselines of 1) finite-width versions of the same, and 2) the NTK and GP limits of \cref{eq:W2Vfwd}.
As shown in \cref{cor:dichotomyMain}, the features of the NTK limit are fixed at initialization as $n\to\infty$ (and so are those of the GP limit, by definition), so its answer to \cref{eqn:wordanalogy} is uniformly selected from the whole vocabulary.%
\footnote{There is some nuance here because $h(\xi)^\trsp h(\bar \xi)$ is actually $\Theta(\sqrt n)$ instead of $\Theta(n)$ because $\xi, \bar\xi$ are one-hot, but the conclusion is the same; see \cref{sec:Word2Vecdetails}.}
Its accuracy is thus $\f{1}{|\mathcal V| - 3}$.
Since $|\mathcal V|$ is 71,291 for \texttt{text8} and 142,276 for \texttt{fil9}, this number is practically 0.
We compute the $\mu$P limit according to \cref{alg:infwidthlimit}, but we relate more implementation details in \cref{sec:Word2Vecdetails}.

\begin{wrapfigure}{r}{0.7\textwidth}
    \vspace{-25pt}
    \begin{center}
        \includegraphics[width=0.70\textwidth]{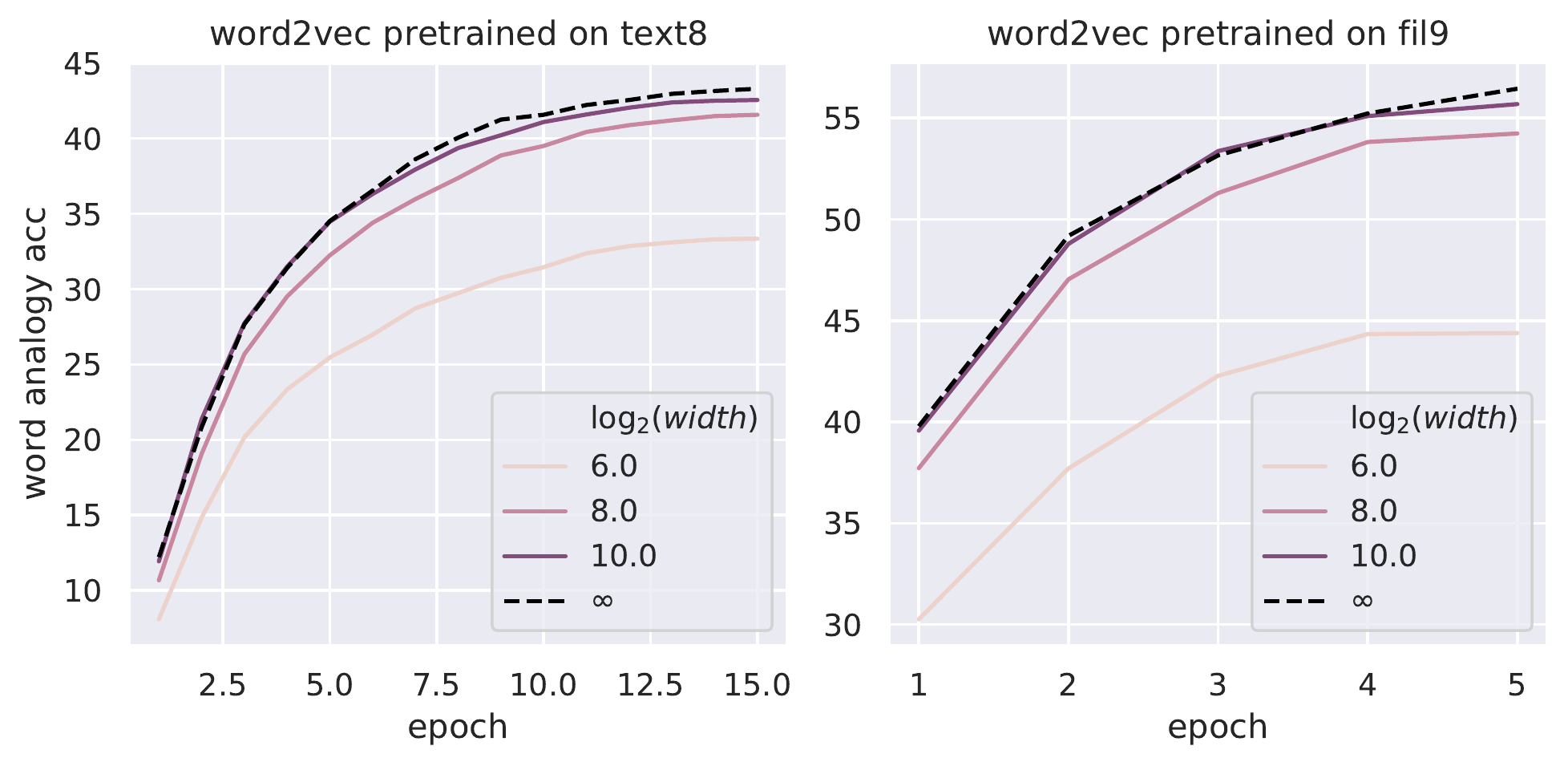}
    \end{center}
    \vspace{-30pt}
\end{wrapfigure}

\paragraph{Findings}
We show our results in \cref{tab:w2vresult} and \textcolor{purple}{\textbf{Figure to the right}}.
As expected, the infinite-width and finite-width $\mu$P networks significantly outperform the NTK limit.
In addition, we observe the finite width $\mu$P networks converge to the performance of the $\mu$P limit from below, as width increases.

\begin{table}
    \centering %
    \caption{\textbf{Test Accuracies on Word Analogy after Pretraining with CBOW Word2Vec.}}
    \label{tab:w2vresult} 
    \begin{tabular}{ccccccc}
    \toprule 
                    & \multicolumn{5}{c}{number = $\log_{2}width$} & \tabularnewline
    \midrule 
    Dataset         & 6      & 8      & 10     & $\mu$P & GP/NTK & \tabularnewline
    \midrule 
    \texttt{text8}  & 33.35  & 41.58  & 42.56  & \textbf{43.31}  & 0.0    & \tabularnewline
    \texttt{fil9}   & 44.39  & 54.24  & 55.69  & \textbf{56.45}  & 0.0    & \tabularnewline
    \bottomrule
    \end{tabular}
    
\end{table}

\section{Conclusion}

In this paper, we presented a framework, based on the notion of \emph{abc-parametrizations} and \emph{Tensor Programs} technique, that unifies the Neural Tangent Kernel (NTK) and Mean Field limits of large width neural networks (NNs).
In the Dynamical Dichotomy theorem, we classified the abc-parametrizations into feature learning and kernel regimes.
We identified the lack of feature learning as a fatal weakness of NTK as a model for real NN.
In fact, we showed the standard parametrization suffers from the same problem.
As a solution, we proposed the Maximal Update Parametrization ($\mu$P) and derived its infinite-width limit, which admits feature learning.
Through experiments on Word2Vec and few-shot learning, we demonstrated that $\mu$P is a good model for feature learning behavior in neural networks.

More generally, this paper showcased the power of the \emph{Tensor Programs} technique: Any computation expressable in a Tensor Program has a ``infinite-width'' limit we can derive.
Because of the universality of Tensor Programs for expressing deep learning computation \citep{TP1,TP2}, this technique systematically solves the mathematical problem of taking infinite-width limits which has been dealt with haphazardly in prior literature.
Its immense flexibility means that the theory of reinforcement learning, self-supervised learning, deep generative models, etc with overparametrized neural networks in the feature learning regime are now ripe for the picking.

\section*{Acknowledgements}

In alphabetical order, we thank
Sina Alemohammad,
Zeyuan Allen-Zhu,
Francis Bach,
Yasaman Bahri,
Lenaic Chizat,
Jeremy Cohen,
Yarin Gal,
Quanquan Gu,
Bobby He,
Di He,
Jiaoyang Huang,
Arthur Jacot,
Jaehoon Lee,
Jason Lee,
Zhiyuan Li,
Etai Littwin,
Yiping Lu,
Song Mei,
Roman Novak,
Vinay Rao,
Michael Santacroce,
Sam Schoenholz,
Lisa Schut,
Jascha Sohl-Dickstein,
Alessandro Sordoni,
Denny Wu,
Huishuai Zhang, and
Pengchuan Zhang
for discusson and feedback.

\bibliography{ref}
\bibliographystyle{plainnat}

\newpage
\appendix

\section{A Short Origin Story of the \emph{Tensor Programs} Paper Series}
\label{sec:history}
The Tensor Programs framework was initially proposed in \citep{scaling} in February 2019, and was mainly applied to extend the NNGP and NTK limits to arbitrary architectures (and to make rigorous the signal propagation literature \citep{poole_exponential_2016,schoenholz_deep_2017,yang_mean_2017,yangVarianceVariation,hanin_which_2018,hanin_how_2018,chen_dynamical_2018,yang_mean_2019,pennington_resurrecting_2017,hayou_selection_2018,philipp_nonlinearity_2018,gilboa_dynamical_2019,yang2019finegrained}).
While NNGP and NTK amount to taking limits of neural networks at initialization, it was soon, in April 2019, realized that Tensor Programs could 1) also trivially take limits of the entire training procedure of neural networks (which is the main theoretical idea of this paper), and 2) calculate the feature learning limit.
However, at that point, it also became clear that \citep{scaling} was not written accessibly, and its formulation of Tensor Programs was cumbersome to use.
A question had to be asked: Should the feature learning paper be written immediately on such an unwieldy foundation, or should significant effort be devoted to fixing this foundation first?
Eventually, a decision was made in favor of the latter.
The \emph{Tensor Programs series} was created as way to re-organize and re-present the Tensor Programs machinery in a user-friendly way to the machine learning audience (the first 3 papers \cite{TP1,TP2,TP3} of the series), before extracting payoffs from this foundation (starting from this paper).

\section{Further Discussions on the Shallow NTK and MF Examples}
\label{sec:appendixmotivating}

\paragraph*{How does the Function Change?}
If the NTK limit does not allow features to evolve, then how does learning occur? To answer this question, note 
\[
\Delta f_{t}(\xi)=V_{0}\Delta x_{t}(\xi)+\Delta V_{t}x_{0}(\xi)+\Delta V_{t}\Delta x_{t}(\xi).
\]
In short, then, the evolution of $f_{t}(\xi)$ in the NTK limit is predominantly due to $V_{0}\Delta x_{t}(\xi)$ and $\Delta V_{t}x_{0}(\xi)$ only, while in the MF limit, $\Delta V_{t}\Delta x_{t}(\xi)$ also contributes nontrivially.

\emph{Example}: For $t=1$, $\Delta f_{1}(\xi)=V_{0}\Delta x_{1}(\xi)+n^{-2a_{v}}x_{0}^{\trsp}x_{0}(\xi)+n^{-2a_{v}}x_{0}^{\trsp}\Delta x_{1}(\xi).$ In NTP, $a_{v}=1/2$, so the term $n^{-2a_{v}}x_{0}^{\trsp}x_{0}(\xi)=\Theta(1)$ for generic $\xi,\xi_{0}$. On the other hand, $n^{-2a_{v}}x_{0}^{\trsp}\Delta x_{1}(\xi)=O(1/\sqrt{n})$ because $\Delta x_{1}(\xi)=O(1/\sqrt{n})$ as noted above. Likewise, 
\begin{align*}
V_{0}\Delta x_{1}(\xi) & \approx V_{0}[\phi'(h_{0}(\xi))\odot\Delta h_{1}(\xi)]
    =V_{0}[\phi'(h_{0}(\xi))\odot\Delta h_{1}(\xi)]\\
    & =C\sum_{\alpha=1}^{n}V_{0\alpha}\phi'(h_{0}(\xi)_{\alpha})V_{0\alpha}\phi'(h_{0\alpha})
    =C\sum_{\alpha=1}^{n}(V_{0\alpha})^{2}\phi'(h_{0}(\xi)_{\alpha})\phi'(h_{0\alpha}),
\end{align*}
where $C=\chi_{0}\xi_{0}\xi=\Theta(1)$. Now $(V_{0\alpha})^{2}=\Theta(1/n)$ and is almost surely positive. On the other hand, $\phi'(h_{0}(\xi)_{\alpha})\phi'(h_{0\alpha})=\Theta(1)$ and should have a nonzero expectation over random initialization (for example, if $\phi$ is relu then this is obvious). Therefore, the sum above should amount to $V_{0}\Delta x_{1}(\xi)\approx\Theta(1)$. In summary, in the NTK limit, $\Delta f_{1}(\xi)=\Theta(1)$ due to the interactions between $V_{0}$ and $\Delta x_{1}(\xi)$ and between $\Delta V_{1}$ and $x_{0}(\xi)$, but there is only vanishing interaction between $\Delta V_{1}$ and $\Delta x_{1}(\xi)$.

The case for general $t$, again, can be derived easily using Tensor Programs.

\section{abc-Parametrization for General Neural Architectures}
\label{sec:abcAnyArch}

We can straightforwardly generalize abc-parametrizations to an arbitrary neural architecture.
Each parameter tensor $W$ would get its own $a_W$ and $b_W$, such that $W=n^{-a_W} w$ and $w$ is the actual trainable parameter with initialization $w_{\alpha\beta}\sim \Gaus(0,n^{-2b_W})$.
The learning rate is still $\eta n^{-c}$ for some fixed $\eta$.

\subsection{Maximal Update Parametrization}
\label{sec:muPAnyArch}

\paragraph{MLP with Biases}
Suppose in \cref{eqn:MLP}, for each $l\in[L]$, we have $h^{l}(\xi) = W^{l} x^{l-1}(\xi) + b^l$ instead, for bias $b^l \in \R^n$.
Then in $\mu$P, the bias $b^l$ should have $a_{b^l} = -1/2$ and $b_{b^l} = 1/2$.
We can also have bias $b^{L+1}$ in the logits $f(\xi) = W^{L+1}x^L(\xi) + b^{L+1}$.
Then we set $a_{b^{L+1}} = b_{b^{L+1}} = 0$.

\paragraph{General Neural Architectures}
More generally, $\mu$P can be defined easily for any neural architecture whose forward pass can be written down as a Tensor Program (e.g.\ ResNet or Transformer; see \citep{TP1} for explicit programs).
The learning rate is always independent of width, i.e.\ $c=0$.
For any parameter tensor $W$, $b_W$ is always $1/2$, and $a_W$ can be defined as follows:
If $W$ is not an output weight matrix, then $a_W$ should be set to $-1 + \f 1 2 p_W$, where $p_W = \lim_{n\to\infty} \log_n \#(W)$ is a) 0 if both sides of $W$ are fixed w.r.t. $n$; b) 1 if $W$ is a vector (e.g.\ bias) or with one side being fixed dimensional (e.g.\ $W^1$); and c) 2 if $W$ is a matrix with both sides scaling like $n$ (e.g.\ weights in the middle of an MLP).
If $W$ is an output weight matrix (and thus the output dimension is fixed w.r.t.\ $n$), then $a_W$ should be $\f 1 2$.
If $W$ is an output bias, then $a_W$ should be 0.

\paragraph{Optimality Properties}

One can formalize, in this general context, the notion of \emph{stability} and the notions of a parameter tensor being \emph{updated maximally} and (a set of readout weights) being {initialized maximally}.
Then one can show that $\mu$P is the unique stable abc-parametrization such that all of its parameter tensors are updated maximally and all of its readout weights are initialized maximally.

\section{Experimental Details}
\label{sec:expdetails}

The main models in our experiments are all 1-hidden-layer linear MLPs with input dimension $d$ and output dimension $\dout$.
In our experiments, we will consider more advanced forms, but, as warmup, a basic version of such a network is given by
\begin{equation}
f(\xi)=Vh(\xi),\quad\quad h(\xi)=U\xi,\label{eqn:UV1LPGeneral}
\end{equation}
for $U\in\R^{n\times d},V\in\R^{\dout \times n}$ parametrized like $U=\sqrt{n}u,V=\frac{1}{\sqrt{n}}v$ and with initialization $u_{\alpha\beta},v_{\alpha\beta}\sim\Gaus(0,1/n)$.
In this case, \cref{cor:lin1LP} generalizes to
\begin{thm}\label{thm:lin1LPGeneral}
    Consider a 1-hidden-layer linear MLP in $\mu$P (\cref{eqn:UV1LPGeneral}) and any training routine with learning rate $\eta$.
    As $n\to\infty$, for every input $\xi \in \R^d$, $f_t(\xi) \in \R^{\dout}$ converges almost surely to $\mathring f_t(\xi)$ defined as follows:
    \begin{align*}
        \mathring f_t(\xi) &= (A_t C_t + B_t D_t)\xi \in \R^\dout,\\
        \mathring \chi_{t}&=\loss'(\mathring f_{t},y_{t}) \in \R^\dout,\\
        (A_{t+1}, B_{t+1}) &= (A_t, B_t) - \eta \mathring \chi_t \otimes (C_t \xi_t, D_t \xi_t),\\
        (C_{t+1}, D_{t+1}) &= (C_t, D_t) - \eta (A_t^\trsp \mathring \chi_t, B_t^\trsp \mathring \chi_t) \otimes \xi_t,\quad
    \end{align*}
    where $\otimes$ denotes outer product ($u \otimes v = u v^\trsp$),
    with initial condition
    \[A_0 = I_\dout \in \R^{\dout\times \dout},\ 
    D_0 = I_d \in \R^{d\times d},\ 
    B_0 = 0 \in \R^{\dout \times d},\ 
    C_0 = 0 \in \R^{d \times \dout}.\]
\end{thm}
While we will not use this theorem, we intend it to give an idea of the mathematical process underneath our implementations, which we discuss now.

\global\long\def\HH{\boldsymbol{H}}%
\global\long\def\bb{\boldsymbol{b}}%
\global\long\def\XXi{\boldsymbol{\Xi}}%
\global\long\def\WW{\boldsymbol{W}}%
\global\long\def\uu{\boldsymbol{u}}%

\global\long\def\vv{\boldsymbol{v}}%
\global\long\def\PhiPhi{\boldsymbol{\Phi}}%
\newcommand{\pluseq}{\mathrel{+}=}

\subsection{Few-Shot Learning on Omniglot via MAML}

\label{sec:mamldetails}

\subsubsection{Linear 1-Hidden-Layer \texorpdfstring{$\mu$P}{MUP} Network}
We consider a linear 1-hidden-layer MLP with bias, input dimension $d$, output dimension $\dout$, given by
\[
f(\xi)=Vh(\xi)\in\R^{\dout},\quad h(\xi)=U\xi+B\in\R^{n},
\]
where $\xi\in\R^{d}$. Following $\mu$P, we factor $U=\sqrt{n}u\in\R^{n\times d},V=\frac{1}{\sqrt{n}}v\in\R^{\dout\times n},B=\alpha\sqrt{n}\beta\in\R^{n},$
where $u, v, \beta$ are the trainable parameters.
We initialize $u_{\alpha\beta}\sim\Gaus(0,\sigma_{u}^{2}/n),v_{\alpha\beta}\sim\Gaus(0,\sigma_{v}^{2}/n)$, $\beta=0 \in \R^n$.
We can cancel the factors of $\sqrt n$ and rewrite
\[
f(\xi)=v h(\xi)\in\R^{\dout},\quad h(\xi)=u\xi+b\in\R^{n},
\]
where $b = \alpha \beta$.
We will also consider gradient clipping with threshold $g$ and weight decay with coefficient $\gamma$.
So in summary, the hyperparameters are
\[\sigma_{u},\sigma_{v}\text{ (init.\ std.)},\quad
\alpha\text{ (bias multiplier)},\quad
\eta \text{ (LR)},\quad
g \text{ (grad.\ clip)},\quad
\gamma \text{ (weight\ decay)}.\]

As in \cref{cor:lin1LP}, it's easy to see that each column of $u_{t}$ at any time $t$ is always a linear combination of the columns of $u_{0}$ and the rows of $v_{0}$ such that the coefficients of these linear combinations converge deterministically in the $n\to\infty$ limit; likewise for $b_{t}$ and the rows of $v_{t}$. To track the evolution of $f$, it suffices to track these coefficients. Therefore, for implementation, we reparametrize as follows:

\paragraph*{Coefficient matrix and vector}

Let $\mu_{1},\ldots,\mu_{d},\nu_{1},\ldots,\nu_{\dout}\in\R^{n}$ be standard Gaussian vectors such that the columns of $u_{0}$ will be initialized as $\sigma_{u}\mu_{1}/\sqrt n,\ldots,\sigma_{u}\mu_{d}/\sqrt n$ and the rows of $V_{0}$ will be initialized as $\sigma_{v}\nu_{1}/\sqrt n,\ldots,\sigma_{v}\nu_{\dout}/\sqrt n$. Write $\mu=(\mu_{1},\ldots,\mu_{d})\in\R^{n\times d},\nu=(\nu_{1},\ldots,\nu_{\dout})\in\R^{n\times\dout}$. Define coefficient matrices 
\[
\uu^{\trsp}\in\R^{d\times(d+\dout)},\vv\in\R^{\dout\times(d+\dout)},
\]
such that at any time, $(u,v^{\trsp})\in\R^{n\times(d+\dout)}$ is $
\frac{1}{\sqrt{n}}(\mu,\nu)(\uu,\vv^{\trsp})$ in the infinite-width limit. We initialize
\[\begin{pmatrix}\uu^{\trsp}\\
\vv
\end{pmatrix}\gets\begin{pmatrix}\sigma_{u}I & 0\\
0 & \sigma_{v}I
\end{pmatrix},\]
i.e. a ``diagonal'' initialization.
Likewise, define coefficient vector $\bb\in\R^{d+\dout}$, initialized at 0, such that, at any time, $b$ is approximately distributed as $\f 1 {\sqrt n}(\mu,\nu)\bb$.
To track the evolution of the infinite-width network, we will track the evolution of $\uu, \vv, \bb$.

In general, we use \textbf{bold} to denote the coefficients (in $\mu,\nu$) of a tensor (e.g. $\bb$ for coefficients of $b$). We also use capital letters to denote the batched version (e.g. $H$ for batched version of $h$).
\cref{alg:fin1LP,alg:inf1LP} below summarize the SGD training of the finite- and the infinite-width networks.
Note that aside from initialization and the hidden size ($n$ vs $d+\dout$), the algorithms are essentially identical.

\begin{multicols}{2}
    \begin{algorithm}[H]
        \caption{SGD Training of Finite-Width Linear $\mu$P 1-Hidden-Layer Network}
        \label{alg:fin1LP}
        \begin{algorithmic}[1]
          \Require Hyperparameters $n, \sigma_{u},\sigma_{v}, \alpha, \eta, g, \gamma$.
          \State Initialize $u_{\alpha\beta}\sim\Gaus(0,\sigma_{u}^{2}/n)$
          \State Initialize $v_{\alpha\beta}\sim\Gaus(0,\sigma_{v}^{2}/n)$
          \State Initialize $b\gets0$
          \For{each batch of inputs $\Xi\in\R^{B\times d}$ and labels $Y\in\R^{B\times\dout}$}
            \State {\it // Forward Pass}
            \State $H \gets \Xi u^{\trsp}+b\in\R^{B\times n}$
            \State $f(\Xi) \gets Hv^{\trsp}\in\R^{B\times\dout}$
            \State {\it // Backward Pass}
            \State $\chi \gets \loss'(f(\Xi),Y)\in\R^{B\times\dout}$
            \State $ d u \gets-v^{\trsp}\chi^{\trsp}\Xi\in\R^{n\times d}$
            \State $ d v \gets-\chi^{\trsp}H\in\R^{\dout\times n}$
            \State $ d b \gets-\alpha^{2} \onev^\trsp \chi v\in\R^{n}$
            \State {\it // Gradient Clipping}
            \State $G \gets\sqrt{\| d u\|_{F}^{2}+\| d v\|_{F}^{2}+\|\f{ d b}\alpha\|^{2}}$
            \State $\rho \gets\min(1,g/G)$
            \State $ d u \gets\rho d u$
            \State $ d v \gets\rho d v$
            \State $ d b \gets\rho d b$
            \State {\it // Gradient Step w/ Weight Decay}
            \State $u \pluseq\eta d u - \eta\gamma u\in\R^{d\times n}$
            \State $v \pluseq\eta d v - \eta\gamma v\in\R^{\dout\times n}$
            \State $b \pluseq\eta d b - \eta\gamma b\in\R^{n}$
          \EndFor
        \end{algorithmic}
    \end{algorithm}
    \begin{algorithm}[H]
        \caption{SGD Training of Infinite-Width Linear $\mu$P 1-Hidden-Layer Network}
        \label{alg:inf1LP}
        \begin{algorithmic}[1]
          \Require Hyperparameters $\sigma_{u},\sigma_{v}, \alpha, \eta, g, \gamma$.
          \State Initialize $\uu^\trsp \gets (\sigma_u I, 0)$
          \State Initialize $\vv \gets (0, \sigma_v I)$
          \State Initialize $\bb \gets 0$
          \For{each batch of inputs $\Xi\in\R^{B\times d}$ and labels $Y\in\R^{B\times\dout}$}
            \State {\it // Forward Pass}
            \State $\HH \gets\Xi\uu^{\trsp}+\bb\in\R^{B\times(d+\dout)}$
            \State ${f}(\Xi) \gets \HH\vv^{\trsp}\in\R^{B\times\dout}$
            \State{\it // Backward Pass}
            \State $ \chi \gets \loss'( f(\Xi),Y)\in\R^{B\times\dout}$
            \State $ d\uu \gets-\vv^{\trsp} \chi^{\trsp}\Xi\in\R^{(d+\dout)\times d}$
            \State $ d\vv \gets- \chi^{\trsp}\HH\in\R^{\dout\times(d+\dout)}$
            \State $ d\bb \gets-\alpha^{2}\onev^\trsp  \chi\vv\in\R^{d+\dout}$
            \State{\it // Gradient Clipping}
            \State $G \gets\sqrt{\| d\uu\|_{F}^{2}+\| d\vv\|_{F}^{2}+\|\f{ d\bb}{\alpha}\|^{2}}$
            \State $\rho \gets\min(1,g/G)$
            \State $ d\uu \gets\rho d\uu$
            \State $ d\vv \gets\rho d\vv$
            \State $ d\bb \gets\rho d\bb$
            \State{\it // Gradient Step w/ Weight Decay}
            \State $\uu \pluseq \eta d\uu-\eta\gamma\uu\in\R^{(d+\dout)\times d}$
            \State $\vv \pluseq \eta d\vv-\eta\gamma\vv\in\R^{\dout\times(d+\dout)}$
            \State $\bb \pluseq \eta d\bb-\eta\gamma\bb\in\R^{d+\dout}$
          \EndFor
        \end{algorithmic}
    \end{algorithm}
\end{multicols}

During inference, we just run the \emph{Forward Pass} section with $\Xi$ substituted with test data.

The algorithms for MAML can then be obtained by a straightforward modification of these algorithms.
(Note that in MAML, we do not clip gradients during adaptation, but rather clip the gradient against the validation loss of task; we also disable weight decay by setting the coefficient $\gamma$ to 0).

\begin{algorithm}[tb]
    \caption{MAML Training of Kernel Model with Kernel $K$}
    \label{alg:kernelMAML}
    \begin{algorithmic}[1]
      \Require Kernel $K$, adaptation step size $\epsilon$, meta learning rate $\eta$, batch size $B$, gradient clip $g$
      \State Initialize $Q = \{\}$
      \While{True}
        \State Draw a batch of tasks
        \For{each task in batch}
            \State {\it // Adaptation}
            \State Sample training set $\mathcal D$
            \For{each input/label pair $(\xi_i, y_i) \in \mathcal D$}
                \State $\chi_i \gets \loss'(f_Q(\xi_i), y_i)$
            \EndFor
            \For{each input/label pair $(\xi_i, y_i) \in \mathcal D$}
                \State $Q$\texttt{.push}$((\xi_i, -\epsilon \chi_i))$
            \EndFor
            \State {\it // Calculate Test Set Gradient}
            \State Sample test set $\hat{\mathcal D}$
            \For{each input/label pair $(\hat \xi_i, \hat y_i) \in \hat{\mathcal D}$}
                \State $\hat \chi_i \gets \loss'(f_Q(\hat \xi_i), \hat y_i)$
            \EndFor
            \For{each input/label pair $(\xi_i, y_i) \in \mathcal D$}
                \State $Q$\texttt{.pop}$((\xi_i, -\epsilon \chi_i))$
            \EndFor
            \State {\it // Gradient Clip}
            \State $G \gets \sqrt{\sum_{(\hat \xi_i, \hat y_i)\in \hat {\mathcal D}} \sum_{(\hat \xi_j, \hat y_j)\in \hat {\mathcal D}} \hat \chi_i \hat \chi_j K(\hat \xi_i, \hat \xi_j)}$
            \State $\rho \gets \min(1, g/G)$
            \State {\it // Gradient Update}
            \For{each input/label pair $(\hat \xi_i, \hat y_i) \in \hat{\mathcal D}$}
                \State $Q$\texttt{.push}$((\hat \xi_i, -\rho \eta \hat \chi_i))$
            \EndFor
        \EndFor
      \EndWhile
    \end{algorithmic}
\end{algorithm}

\paragraph{Hyperparameter Sweep}
We sweep $\sigma_u$, $\sigma_v$, $\eta$ and $\alpha$ with the following grid for finite width and $\mu$P networks.
\begin{itemize}
    \item $\sigma_u: [0.5, 1, 2, 4, 8]$,
    \item $\sigma_v: [2^{-5}, 2^{-4}, 2^{-3}, 2^{-2}, 2^{-1}]$,
    \item $\eta: [0.025, 0.05, 0.1, 0.2, 0.4]$,
    \item $\alpha: [0.25, 0.5, 1, 2, 4]$
\end{itemize}
We are interested in 1-shot, 5-way learning with Omniglot.
This means that each task provides 5 training samples, each corresponding to one of the 5 labels of the task.
Each hyperparameter combination above is used to train for 100 epochs over 3 random seeds, where each epoch consists of 100 batches of 32 tasks.
We average the validation accuracy across the last 10 epochs and document the best hyperparameters in \cref{tab:MAMLhparam}, along with the test accuracy from a 15-seed rerun\footnote{After excluding outliers at least one standard deviation away from the mean.} for better benchmarking. 
For NTK and GP, we additionally tune the initialization $\sigma_b$ for biases, which is set to $0$ for both finite and $\mu$P networks for simplicity.
\begin{table}
    \centering %
    \caption{\textbf{Best hyperparameters for the MAML experiment.}}
    \label{tab:MAMLhparam} 
    \begin{tabular}{cccccccc}
    \toprule 
    \midrule 
    $\log_2$Width/Limit     & $\sigma_u$ & $\sigma_v$ & $\sigma_b$ & $\eta$ & $\alpha$ & Val. Acc. (\%)  & Test Acc. (\%) \\
    \midrule 
    1       &     0.5 &  0.5     &      - &     0.05 &         2   & $46.72\pm4.30$ &   $55.34\pm1.24$    \\
    3       &     0.5 &  0.25    &      - &     0.1  &         1   & $65.30\pm.27$ &   $64.54\pm.70$     \\
    5       &     1   &  0.125   &      - &     0.4  &         0.5 & $68.74\pm.18$ &   $66.21\pm.15$     \\
    7       &     1   &  0.125   &      - &     0.1  &         1   & $69.03\pm.04$ &   $66.31\pm.16$     \\
    9       &     1   &  0.03125 &      - &     0.1  &         1   & $69.32\pm.07$ &   $66.43\pm.23$     \\
    11      &     1   &  0.03125 &      - &     0.1  &         1   & $69.27\pm.11$ &   $66.36\pm.22$     \\
    13      &     1   &  0.03125 &      - &     0.1  &         1   & $69.27\pm.14$ &   $66.41\pm.18$     \\ 
    $\mu$P  &     1   &  0.03125 &      - &     0.1  &         1   & $69.26\pm.13$ &   $66.42\pm.19$     \\  \midrule
    NTK     &     0.25&  1       &  1     &     0.05 &         1   & $47.47\pm.13$ &   $47.82\pm.04$     \\
    GP      &     1   &  0.25    &  1     &     0.05 &         1   & $38.92\pm.15$ &   $47.60\pm.02$     \\
    \bottomrule
    \end{tabular}
    
\end{table}

\subsubsection{NNGP and NTK for Relu Networks}

Consider a kernel $K$, which in our case will be the NNGP or NTK of a 1-hidden-layer relu network. WLOG, it is induced by an embedding $\Phi$ such that $K(\xi,\zeta)=\langle\Phi(\xi),\Phi(\zeta)\rangle$ where $\langle,\rangle$ is the inner product in the embedding space; we do not care about the details of $\Phi$ or $\langle,\rangle$ as eventually our algorithm only depends on $K$. 

In our setting, we will train a linear layer $W$ on top of $\Phi$ via MAML, $f(\xi)\defeq\langle W,\Phi(\xi)\rangle$. One can see easily that $W$ is always a linear combination of $\Phi(\zeta)$ for various $\zeta$ from the training set we've seen so far. Thus, to track $W$, it suffices to keep an array $Q$ of pairs $(\zeta,q)$ such that $W=\sum_{(\zeta,q)\in Q}q\Phi(\zeta)$ at all times. Let $f_{Q}$ be the function with $W$ given by $Q$. Then
\[f_{Q}(\xi)=\sum_{(\zeta,q_{\zeta})\in Q}q_{\zeta}K(\zeta,\xi).\]
In our case, the number of possible inputs is too large to instantiate a value $q$ for every $\zeta$, so we gradually grow a dynamic array $Q$, which we model as a stack.
Then MAML can be implemented as in \cref{alg:kernelMAML}.

\paragraph{Hyperparameter Sweep}

We sweep $\sigma_u$, $\sigma_v$, $\sigma_b$ and $\eta$ with the following grid for GP and NTK.
\begin{itemize}
    \item $\sigma_u: [0.25, 0.5, 1, 2, 4]$,
    \item $\sigma_v: [0.25, 0.5, 1, 2, 4]$,
    \item $\sigma_b: [0.25, 0.5, 1, 2, 4]$,
    \item $\eta: [0.05, 0.1, 0.2, 0.4, 0.8]$
\end{itemize}
Each hyperparameter combination above is used to train for 5 epochs (the first epoch is almost always the best) over 3 random seeds, where each epoch consists of 100 batches of 32 tasks.
We take the validation accuracy among all epochs and document the best hyperparameters in \cref{tab:MAMLhparam}, along with the test accuracy from a 15-seed rerun.

\subsection{Word2Vec Experimental Details}
\label{sec:Word2Vecdetails}
\subsubsection{\texorpdfstring{$\mu$P}{MUP} Limit}
We shall derive the training algorithm for $\mu$P Word2Vec. First,
we introduce the notation for word embeddings. We denote $\Phi^{i}\defeq h(\xi^{i})$.
If $\xi^{i}$ is a one-hot vector with the $i^{th}$ element set to
1, $\Phi^{i}$ is essentially the $i^{th}$ column of the weight matrix
$U$. We also define the following short-hands for the context embedding:
$\Phi^{J}\defeq\EV_{j\in J}\Phi^{j}=h(\xi^{J})$.
Similarly, $V^{\top}\xi^{\tau}$
describes a row in $V$; we can define $\Phi^{\hat{\tau}}\defeq\hat{h}(\xi^{\tau})\defeq V^{\top}\xi^{\tau}$
and rewrite the loss function. 
\begin{equation}
\loss(f(\xi^{J}),\xi^{\tau})=\begin{cases}
\log(1-\sigma(\Phi^{J}{}^{\top}\Phi^{\hat{\tau}})) & \tau=i\\
\log\sigma(\Phi^{J}{}^{\top}\Phi^{\hat{\tau}}) & \tau\neq i.
\end{cases}\label{eq:W2Vloss-simplified}
\end{equation}

Consequently, the backward pass becomes: 
\begin{equation}
\Delta\Phi^{j}=\frac{1}{|J|}\Delta\Phi^{J}=\frac{\eta}{|J|}\frac{\partial\loss}{\partial\Phi^{J}}=\begin{cases}
\frac{\eta}{|J|}\Phi^{\hat{\tau}}(1-\sigma(\Phi^{J}{}^{\top}\Phi^{\hat{\tau}})) & \tau=i\\
-\frac{\eta}{|J|}\Phi^{\hat{\tau}}\sigma(\Phi^{J}{}^{\top}\Phi^{\hat{\tau}}) & \tau\neq i.
\end{cases}\label{eq:W2Vbackward}
\end{equation}

Following $\mu$P, we initialize $U_{\alpha\beta}\sim\Gaus(0,\sigma_{u}n^{-1})$
and $V_{\alpha\beta}\sim\Gaus(0,\sigma_{v}n^{-1})$, where $n$ is
the width of the finite network. (Here the explicit multipliers of
$\sqrt{n}$ in $U$ and $1/\sqrt{n}$ in $V$ cancel out because the
network is linear). The tunable hyperparameters are the initialization
std $\sigma_{u}$and $\sigma_{v}$, learning rate $\eta$ and weight
decay ratio $\gamma$. Rather than tuning the hyperparameters extensively
for each width, we pick some reasonable values and use them for all
of our experiments. Specifically, we have $\sigma_{u}=\sigma_{v}=1$,
$\eta=0.05$ and $\gamma=0.001$.

Again, using \cref{cor:lin1LP}, we can train the $\mu$P limit in the coefficient
space of $\uu^{\trsp}\in\R^{|\VV|\times2|\VV|},\vv\in\R^{|\VV|\times2|\VV|}$,
with the same ``diagonal'' initialization:

\[
\begin{pmatrix}\uu^{\trsp}\\
\vv
\end{pmatrix}\gets\begin{pmatrix}\sigma_{u}I & 0\\
0 & \sigma_{v}I
\end{pmatrix},
\]

We can adopt the embedding notation and represent a row of $\uu$
with the embedding coefficient vector $\PhiPhi^{\bullet}$
and a column of $\vv$ with $\PhiPhi^{\hat{\bullet}}$.
This is computationally equivalent to training with a hidden size
of $2|\VV|$ and with embeddings initialized as rows (or columns) of one-hot vectors.
The full algorithm is described in \cref{alg:fin1LP} and \cref{alg:inf1LP}; in this case, we remove biases
and use weight decay with coefficient $\gamma=0.001$.
After training, rows of the weight matrix $u$ (resp.\ coefficient matrix $\uu$), i.e.\ $\Phi^\bullet$ (resp.\ $\PhiPhi^\bullet$), are taken as the word vectors.

\subsubsection{NTK Limit}
In the NTK parametrization, $V$ and $U$ in \cref{eq:W2Vfwd} factor as $V = \f 1 {\sqrt n} v$ and $U = u$, and the learning rate is $\Theta(1)$.
Each column $U_{\bullet i}$ of $U$ is equal to $h(\xi^i)$.
At any fixed time $t$, it is easy to see via Tensor Programs that
\[
    h_t(\xi^i) = h_0(\xi^i) + \sum_{j \in \VV} O(1/\sqrt n) v_j + O_{coord}(1/n)
\]
where $v_j$ denotes the $j$th row of $v$ at initialization, and where $O_{coord}(1/n)$ means a vector that is $O(1/n)$ coordinatewise.
Recall that $U=u$ and $v$ are initialized with iid standard Gaussian entries.
Because $\xi^i$ is one-hot, this in particular implies $h_0(\xi^i)$ has standard Gaussian entries, and $h_0(\xi^i)$ is independent from $h_0(\xi^j)$ for $i\ne j$.
Then for any $i\ne j$,
\[
    \f 1 {\sqrt n} h_t(\xi^i)^\trsp h_t(\xi^j) - \f 1 {\sqrt n} h_0(\xi^i)^\trsp h_0(\xi^j) \asto 0,\quad
    \f 1 {\sqrt n} h_0(\xi^i)^\trsp h_0(\xi^j) \distto \Gaus(0, 1)
\]
by Law of Large Numbers (or more formally, \cref{thm:PLNetsorT+MasterTheorem}) and Central Limit Theorem.
In other words, $\f 1 {\sqrt n} h_0(\xi^i)^\trsp h_0(\xi^j)$ is distributed completely randomly, with no regard to the semantic similarities of $i$ and $j$.
Likewise, the inner product in \cref{eqn:wordanalogy} is random, and the argmax is a uniform sample.%
\footnote{Here the randomness comes from initialization: the argmax is different for different random initializations, but it is fixed throughout training in the large width limit.}
Therefore, in the NTK limit, Word2Vec gives random answers and achieves an accuracy of $\f{1}{|\mathcal V| - 3}$.

\section{More Detailed Comparison with Deep Mean Field Limits}
\label{sec:comparisonwithdeepmeanfield}

The key idea of previous works \citep{araujo_mean-field_2019,sirignano_mean_2020,nguyen_nonrigorous_meanfieldMLP_2019,nguyen_rigorous_meanfieldMLP_2020,fang_modelingfromfeatures_2020} proposing multilayer mean field limits of MLPs is to initialize each $n \times n$ matrix $W$ like $W_{\alpha \beta} \gets F(u_\alpha, v_\beta) / n$ for some function $F$ and $u_\alpha \sim Z^u, v_\beta \sim Z^v$ sampled iid for each $\alpha, \beta \in [n]$, where $Z^u$ and $Z^v$ are some fixed (wrt $n$) random variables.
If $x$ is an activation with approximately iid coordinates distributed like random variable $Z^x$, then $Wx$ looks like $(Wx)_\alpha \approx \EV_{Z^v, Z^x} F(u_\alpha, Z^v) Z^x$ by Law of Large Numbers (LLN), roughly iid across $\alpha$.
This logic will in fact hold throughout training.
For well-chosen $(F, Z^u, Z^v)$, this does not get stuck at initialization but this form of initialization is very unnatural.
In Nguyen \& Pham (2020), this is adapted to iid initialization straightforwardly.
For example, this includes $W_{\alpha \beta} \gets \Gaus(0, 1)/n$.
If $x$ is as above, then $(Wx)_\alpha \to 0$ by LLN because $W$ is sampled independently from $x$ at init and has 0 mean.
This means that preactivations every layer will vanish coordinatewise to 0, from which it’s easy to see that the gradients vanish where there are more than 2 hidden layers.
Hence we say that the function gets stuck at initialization.
Contrast this $1/n$ scaling with the more typical $1/\sqrt n$ scaling, i.e. $W_{\alpha\beta} \gets \Gaus(0, 1)/\sqrt n$, which is what we deal with here.
On a technical level, their limit calculation purely goes through LLN, whereas we need to wrestle with Central Limit effects (from the $1/\sqrt n$ scaling) as well.

\section{Nuances of the Master Theorem}

\begin{rem}[Partial derivative]\label{rem:PartialDer}
    The partial derivative in \refZdot{} should be interpreted as follows.
    By a simple inductive argument, $Z^x$ for every vector $x$ in the program is defined \emph{uniquely} as a deterministic function $\varphi(\Zhat^{x^1}, \ldots, \Zhat^{x^k})$ of some $x^1, \ldots, x^k$ in $\mathcal V$ or introduced by \refMatMul{} (notationally, we are suppressing the possible dependence on limit scalars $\mathring \theta_1, \ldots, \mathring \theta_l$).
    For instance, if in a program we have $A \in \mathcal W, v \in \mathcal V$, $y = A v, x = A^\trsp y$, then $Z^x = \Zhat^x + \Zhat^v$, so $\varphi$ is given by $\varphi(a, b) = a + b$.
    Then
    \begin{equation*}
        \partial Z^x / \partial \Zhat^{x^i} \defeq \partial_i \varphi(\Zhat^{x^1}, \ldots, \Zhat^{x^k}), \quad \text{and} \quad \partial Z^x / \partial \Zhat^z \defeq 0 \text{ for any } z \not\in\{x^1, \ldots, x^k\}.
    \end{equation*}
    Note this definition depends on the precise way the program is written, not just on the underlying mathematics.
    For example, if $y,z \in \mathcal V$ and $x = \phi(W(y+z))$, then $Z^x = \phi(\Zhat^{W(y+z)})$ so that $\partial Z^x/\partial \Zhat^{Wy} = \partial Z^x/\partial \Zhat^{Wz} = 0$.
    If instead, we have $x = \phi(Wy+Wz)$, then $Z^x = \phi(\Zhat^{Wy} + \Zhat^{Wz})$ so that $\partial Z^x/\partial \Zhat^{W(x+y)} = 0$.
    However, in both cases, $\Zdot^{W^\trsp x} = (Z^y + Z^z)\EV \phi'(\Zhat^{W(y+z)})$.
\end{rem}

\begin{rem}[Partial derivative expectation]\label{rem:ExpectationPartialDer}
    The quantity $\EV\frac{\partial Z^{x}}{\partial\hat{Z}^{W^{\trsp}y}}$ is well defined if $Z^{x}$ is differentiable in $\hat{Z}^{W^{\trsp}y}$. However, even if this is not the case, e.g. if $x=\theta(W^\trsp y)$ where $\theta$ is the Heavyside step function, we can still define this expectation by leveraging Stein's lemma:
    
    In \refZdot{}, suppose $\{W^{\trsp}y^{i}\}_{i=1}^{k}$ are all elements of $\mathcal V_{W^\trsp}$ introduced before $x$.
    Define the matrix $C\in\R^{k\times k}$ by $C_{ij}\defeq \EV Z^{y^{i}}Z^{y^{j}}$ and define the vector $b\in\R^{k}$ by $b_{i}\defeq \EV\hat{Z}^{W^{\trsp}y^{i}}Z^{x}$. If $a=C^{+}b$ (where $C^{+}$ denotes the pseudoinverse of $C$), then in \refZdot{} we may set
    \begin{equation}
    \sigma_{W}^{2}\EV\frac{\partial Z^{x}}{\partial\hat{Z}^{W^{\trsp}y^{i}}}=a_{i}.\label{eq:ExpectationPartialDer}
    \end{equation}
    This definition agrees with the partial derivative expectation by Stein's lemma when the latter is well defined.
    \cref{thm:PLNetsorT+MasterTheorem} holds with this broader definition of partial derivative expectation.
\end{rem}

\paragraph{Pseudo-Lipschitz functions}

are, roughly speaking, functions whose weak derivatives are polynomially bounded.
\begin{defn}\label{defn:pseudoLipschitz}
    A function $f: \R^k \to \R$ is called \emph{pseudo-Lipschitz} of degree $d$ if $|f(x) - f(y)| \le C\|x-y\|(1 + \sum_{i=1}^k |x_i|^d + |y_i|^d)$ for some $C$.
    We say $f$ is pseudo-Lipschitz if it is so for any degree.
\end{defn}
Here are some basic properties of pseudo-Lipschitz functions:
\begin{itemize}
    \item The norm $\|\cdot\|$ in \cref{defn:pseudoLipschitz} can be any norm equivalent to the $\ell_2$ norm, e.g. $\ell_p, p\ge1,$ norms. Similarly, $\sum_{i=1}^k |x_i|^d + |y_i|^d$ can be replaced by $\|x\|^d_p + \|y\|^d_p$, for any $p \ge 1$.
    \item A pseudo-Lipschitz function is polynomially bounded.
    \item A composition of pseudo-Lipschitz functions of degrees $d_1$ and $d_2$ is pseudo-Lipschitz of degree $d_1 + d_2$.
    \item A pseudo-Lipschitz function is Lipschitz on any compact set.
\end{itemize}

We adopt the following assumption for the Master Theorem \cref{thm:PLNetsorT+MasterTheorem}.
\begin{assm}
    \label{assm:MasterTheoremSmoothness}
    Suppose
    \begin{enumerate}
        \item If a function $\phi(;-): \R^{0 + l} \to \R$ with only parameter arguments is used in \refMoment{}, then $\phi$ is continuous in those arguments.
        \item Any other function $\phi(-;-): \R^{k + l} \to \R$ with parameters (where $k>0$) used in \refNonlinPlus{} or \refMoment{} is pseudo-Lipschitz in all of its arguments (both inputs and parameters).
    \end{enumerate}
\end{assm}
Statement 1 in \cref{assm:MasterTheoremSmoothness} essentially says that if we have scalars $\theta_1, \ldots, \theta_l$ in the program, then we can produce a new scalar by applying a continuous function (a weaker restriction than a pseudo-Lipschitz function)  to them.
Indeed, if $\theta_1, \ldots, \theta_l$ converge almost surely, then this new scalar does too.
In our setting, statement 1 is used to allow any loss function whose derivative is continuous.

Other versions of the Master Theorem can be found in \citep{TP3}, for example, versions where the we do not assume any smoothness condition at all on the nonlinearities beyond that they be polynomially bounded, in exchange for assuming what's called a \emph{rank stability} condition.
This rank stability should be generically true, but checking it rigorously is subtle, so we are content with the pseudo-Lipschitz condition in this paper.

\section{A Rough Sketch of the Geometry of abc-Parametrizations}
\label{sec:abcgeometry}

By the results of \cref{{sec:abcDichotomyMain}}, the stable abc-parametrizations form a polyhedron defined by the inequalities of \cref{thm:stabilityconditionsMain}.
We call the polyhedron obtained by quotienting \cref{stmt:SGDsymmetry} the \emph{stable polyhedron}.
In this section, we remark on some geometric properties of this polyhedron.

First, observe that the stable polyhedron is unbounded (thus, we say \emph{polyhedron} instead of \emph{polytope}).
Indeed, given any stable parametrization, for any $l$, we can set $a_l \gets a_l + \theta, b_l \gets b_l -\theta$ for any $\theta \ge 0$ to obtain another stable parametrization.
This corresponds decreasing the layer $l$ learning rate, so that as $\theta \to \infty$, $W^l$ is not trained.

Second, by \cref{thm:nontriviality}, the nontrivial parametrizations reside in two facets of the stable polyhedron.
These facets are unbounded for the same reason as above.

Next, we show that NTP (as well as $\mu$P) is a vertex on the intersection of these two facets, and NTP and $\mu$P are connected by an edge.

\begin{defn}
    Consider a stable abc-parametrization of the MLP in \cref{eqn:MLP}.
    We say the body of the MLP is \emph{uniformly updated} if, for some training routine, time $t\ge1$, and input $\xi$, $\Delta W^l_t x^l_t(\xi) = \Theta(n^{-r})$ for all $l$ simultaneously, where $r$ is as defined in \cref{def:rMain}.
\end{defn}
In the results of this section below, we assume \cref{assm:phismooth}.
\begin{prop}\label{prop:uniformupdate}
    In a stable abc-parametrization, the MLP body is uniformly updated iff $r_l = r$ for all $l \in[L]$, where $r_l$ is as defined in \cref{prop:maximalupdate}.
\end{prop}

\begin{thm}\label{thm:NTPvertex}
    In NTP, the MLP body is updated uniformly and $W^{L+1}$ is both initialized and updated maximally.
    Furthermore, at initialization, $f_0$ converges in distribution%
    \footnote{as is conventional in the machine learning literature, the convergence in distribution we mean here is really over \emph{finite dimensional marginals}, i.e. $(f_0(\xi_1), \ldots, f_0(\xi_k)) \distto (\mathring f_0(\xi_1), \ldots, \mathring f_0(\xi_k))$ where $\mathring f_0$ is the limit GP.}
    to a Gaussian Process with nonzero kernel.
    NTP is the unique (modulo \cref{stmt:SGDsymmetry}) stable abc-parametrization with both of these properties.
\end{thm}

\begin{thm}\label{thm:uniformparam}
    For any $r\in [0, 1/2]$, there is a unique (modulo \cref{stmt:SGDsymmetry}) stable abc-parametrization with 1) that value of $r$ and the property that 2) the MLP body is updated uniformly and $W^{L+1}$ is both initialized and updated maximally.
    We call this parametrization the \emph{Uniform Parametrization with r-value $r$}, denoted UP$_r$.
    Its abc values are
    \[a_{l}=-\frac{1}{2}\ind(l=1)+r\ \forall l \in [L],\ a_{L+1}=1/2;\quad b_{l}=1/2-r;\quad c=0.\]
\end{thm}
\begin{figure}%
    \centering
    \includegraphics[width=0.4\textwidth]{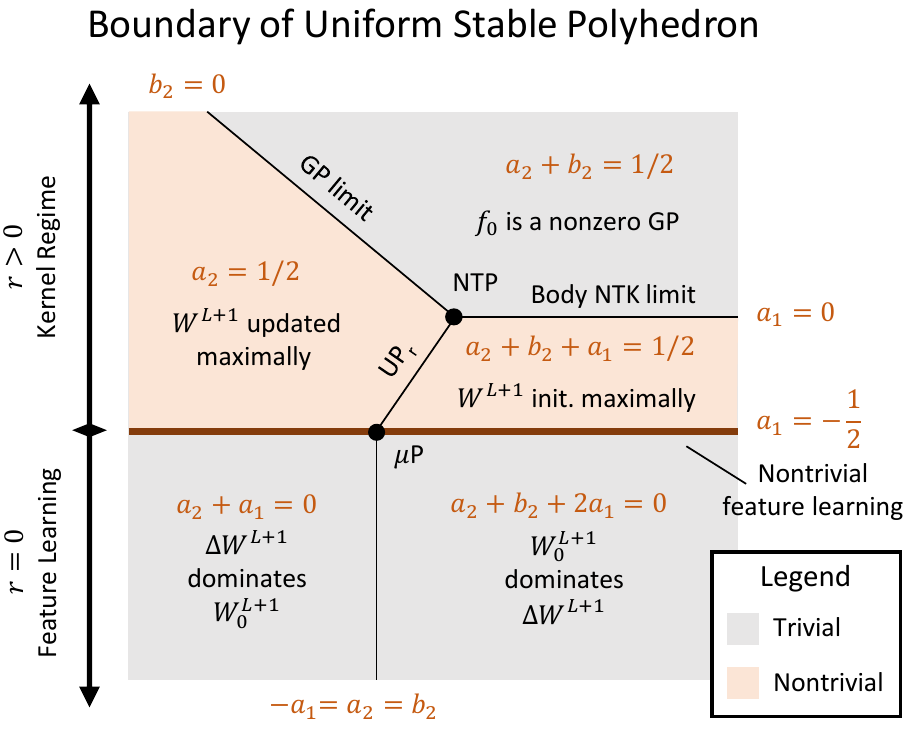}
    \caption{
        \textbf{2D Projection of the Boundary of the Uniform Stable Polyhedron (Equivalently, the Boundary of the Stable Polyhedron for $L=1$).}
        Here, we label each facet and edge of the graph with \textcolor{salmon}{orange text} to indicate the corresponding defining algebraic condition in the $L=1$ case (as part of the stable polyhedron, assuming $c=0$ and $b_1 = -a_1$), and with \textbf{black text} to indicate the verbal interpretation valid for all $L$ (as part of the uniform stable polyhedron).
        We obtain the caricature in \cref{fig:abcparamspace} by taking the \emph{nontrivial} subspace of the graph here and quotienting the two facets by their respective points at infinity.
        \emph{Explanation of some captions:}
        \emph{GP limit} means the training dynamics amounts to training only the last layer in the infinite-width limit, starting from a nonzero initial GP.
        \emph{Body NTK limit} means NTK dynamics except the last layer does not contribute to the NT kernel.}
    \label{fig:abcgeometry}
\end{figure}
In particular, UP$_0$ is $\mu$P and UP$_{1/2}$ is NTP.
For $r > 1/2$, such a uniform parametrization is not stable because $W_0$ would need to be $\Theta(n^{r-1})$, which would cause the initial GP to blow up.
Thus, geometrically, UP$_r, r \in[0, 1/2]$, form an edge of the stable polyhedron.

We can define the \emph{uniform stable polyhedron} to be the subset of the stable polyhedron corresponding to parametrizations which update the MLP body uniformly.
This is isomorphic to the stable polyhedron when $L=1$.
Since stable abc-parametrizations with $L=1$ has only 3 degrees of freedom, say $a_1, a_2, b_2$ while we fix $c=0$ (via \cref{stmt:SGDsymmetry}) and $b_1 = -a_1$, we can visualize the corresponding stable polyhedron in 3D.
However, the nontrivial parametrizations only reside in the boundary of this polyhedron.
Because of its unbounded nature, we can project its boundary in 2D and visualize it.
This is done in \cref{fig:abcgeometry}.

\section{Proofs of Main Results}

\subsection{Rigorous Statements of Main Results}

\subparagraph*{Applicable Nonlinearities}

For technical reasons, in our main results we restrict our attention to the canonical examples of nonlinearities: tanh and relu --- or rather, a smooth version of relu called gelu \citep{gelu} common in transformer models \citep{GPT3}. More precisely,
\begin{defn}\label{defn:gelu}
Define $\sigma$-gelu to be the function $x\mapsto\frac{1}{2}x\text{erf}(\sigma^{-1}x)+\sigma\frac{e^{-\sigma^{-2}x^{2}}}{2\sqrt{\pi}}+\frac{x}{2}$.
\end{defn}

$\sigma$-gelu is a smooth approximation of relu and is the integral of $\frac{1}{2}(\text{erf}(\sigma^{-1}x)+1)$ that is 0 at $-\infty$. The large $\sigma$ is, the smoother $\sigma$-gelu is. As $\sigma\to0$, $\sigma$-gelu converges to relu. We believe our results will hold for generic nonlinearities, but making this precise is outside our scope here. (See \cref{rem:nonlinspecificity} for some discussion).

\subparagraph*{Notations and Terminologies}
\begin{defn}[Big-O Notation]\label{defn:BigO}
Given a sequence of scalar random variables $c=\{c^{n}\in\R\}_{n=1}^{\infty}$, we write $c=\Theta(n^{-a})$ if there exist constants $A,B$ such that $An^{-a}\le|c|\le Bn^{-a}$ for sufficiently large $n$, almost surely\footnote{Here \emph{almost surely} means \emph{for almost every instantiation of $c^{1},c^{2},\ldots$, }i.e. it is with regard to the product probability space generated by all of $\{c^{n}\}_{n=1}^{\infty}$. In this paper, this probability space will be generated by random initializations of a neural network at every width $n$. Very importantly, note the order of the qualifiers: we are saying \emph{for almost every instantiation of $c^{1},c^{2},\ldots$, for large enough $n$, $An^{-a}\le|c|\le Bn^{-a}$.}}. Given a sequence of random vectors $x=\{x^{n}\in\R^{n}\}_{n=1}^{\infty}$, we say \emph{$x$ has coordinates of size $\Theta(n^{-a})$} and write $x=\Theta(n^{-a})$ to mean the scalar random variable sequence $\{\sqrt{\|x^{n}\|^{2}/n}\}_{n}$ is $\Theta(n^{-a})$. Similarly for the notations $O(n^{-a}),\Omega(n^{-a})$. We use the notations $\Theta_{\xi}(n^{-a}),O_{\xi}(n^{-a}),\Omega_{\xi}(n^{-a})$ if the hidden constants $A,B$ are allowed to depend on some object $\xi$. For brevity, we will often abuse notation and say $c$ itself is a random variable or $x$ itself is a random vector.
\end{defn}

Most often, the vector $x$ will have ``approximately iid'' coordinates, so the notation $x=\Theta(n^{-a})$ can be interpreted intuitively to say $x$ has coordinates of ``standard deviation'' $\Theta(n^{-a})$, which justifies the name.
\begin{defn}\label{defn:abc}
An \emph{abc-parametrization} is a joint parametrization of an MLP and the learning rate specified by the numbers $\ensuremath{\{a_{l},b_{l}\}_{l}\cup\{c\}}$ as in \cref{eqn:MLP}. Below we will often say \emph{abc-parametrization of an MLP} for short, even though the parametrization affects the learning rate as well. A \emph{training routine} is a combination of learning rate $\eta n^{-c}$, training sequence $\{(\xi_{t},y_{t})\}_{t\ge0}$, and a loss function $\loss(f(\xi),y)$ that is continuously differentiable in the prediction of the model $f(\xi)$.
\end{defn}

\subparagraph*{Main Results}

We will mainly focus on \emph{stable} parametrizations, defined below, which intuitively means 1) the preactivations $\{h^{l}\}_{l}$ and activations $\{x^{l}\}_{l}$ have $\Theta(1)$ coordinates at initialization, and 2) their coordinates and the logit $f(\xi)$ all stay $O(1)$ (i.e.\ bounded independent of $n$) throughout the course of SGD.\footnote{but they may depend on training time and $\eta$; in particular, it's possible that they diverge with time} Otherwise, they tend to $\infty$ with $n$, eventually going out of floating point range. Indeed, this is an acute and real problem common in modern deep learning, where float16 is necessary to train large models.
\begin{defn}[Stability]\label{defn:stability}
We say an abc-parametrization of an $L$-hidden layer MLP is \emph{stable} if
\begin{enumerate}
\item For every nonzero input $\xi\in\isp$,
\begin{equation}
h_{0}^{l}(\xi),x_{0}^{l}(\xi)=\Theta_{\xi}(1),\forall l\in[L],\quad\text{and}\quad\EV f_{0}(\xi)^{2}=O_{\xi}(1),\label{eq:initstable}
\end{equation}
where the expectation is taken over the random initialization.
\item For any training routine, any time $t\ge0$, $l\in[L]$, $\xi\in\isp$, we have 
\[
\Delta h_{t}^{l}(\xi),\Delta x_{t}^{l}(\xi)=O_{*}(1),\forall l\in[L],\quad\text{and}\quad f_{t}(\xi)=O_{*}(1),
\]
where the hidden constant inside $O$ can depend on the training routine, $t$, $\xi$, and the initial function values $f_{0}(\isp)$.\footnote{For e.g. the NTK limit, $f_{0}$ is a GP, so that we should expect the bounds on $\Delta h_{t}^{l}(\xi),\Delta x_{t}^{l}(\xi)$ to depend on $f_{0}$.}
\end{enumerate}
\end{defn}
Recall from the main text,
\begin{defn}
\label{def:r}For any abc-parametrization, we write $r$ for the quantity
\[
r\defeq\min(a_{L+1}+b_{L+1},2a_{L+1}+c)+c-1+\min_{l=1}^{L}\left[2a_{l}+\ind(l=1)\right].
\]
\end{defn}

For example, in NTP, $r=1/2$, while in $\mu$P, $r=0$. Intuitively, $r$ is the exponent such that $\Delta x_{t}^{L}(\xi)=\Theta_{\xi}(n^{-r})$. Thus, to avoid activation blowup, we want $r\ge0$; to perform feature learning, we want $r=0$.
\begin{thm}[Stability Characterization]
\label{thm:stabilityconditions}Suppose $\phi$ is tanh or $\sigma$-gelu for sufficiently small $\sigma$. An abc-parametrization is stable iff all of the following are true (with intuitions in parentheses):

\begin{enumerate}
\item ((pre)activations at initialization are $\Theta(1)$ and logits are $O(1)$)\label{item:actlogitinit} 
\begin{equation}
a_{1}+b_{1}=0;\quad a_{l}+b_{l}=1/2,\ \forall l\in[2,L];\quad a_{L+1}+b_{L+1}\ge1/2.\label{eq:actlogitinit}
\end{equation}
\item (features don't blowup, i.e. $\Delta x_{t}^{l}=O(1)$ for all $l$)
\begin{equation}
r\ge0.\label{eqn:DeltaWNotTooBig}
\end{equation}
\item (logits don't blow up during training, i.e. $\Delta W_{t}^{L+1}x_{t}^{L},W_{0}^{L+1}\Delta x_{t}^{L}=O(1)$) \label{item:logitblowuptrain} 
\begin{equation}
2a_{L+1}+c\ge1;\quad a_{L+1}+b_{L+1}+r\ge1.\label{eq:logitblowuptrain}
\end{equation}
\end{enumerate}
Here, $r$ is as defined in \cref{def:r}.

\end{thm}

In \cref{eq:logitblowuptrain}, $\Delta W_{t}^{L+1}$ turns out to be $\Theta(n^{-(2a_{L+1}+c)})$ and is correlated with $x_{t}^{L}=\Theta(1)$ such that their product behaves according to Law of Large Numbers; the first inequality says this should not blow up. Similarly, $W_{0}^{L+1}=\Theta(n^{-(a_{L+1}+b_{L+1})})$ and it turns out $\Delta x_{t}^{L}=\Theta(n^{-r})$ and they will interact via Law of Large Numbers, so the second inequality says their product shouldn't blow up.

Our main results concern \emph{nontrivial} parametrizations:
\begin{defn}[Nontriviality]
We say an abc-parametrization of an $L$-hidden layer MLP is \emph{trivial} if for every training routine, $f_{t}(\xi)-f_{0}(\xi)\asto0$ for any time $t\ge1$ and input $\xi\in\isp$ (i.e. the function does not evolve in the infinite-width limit). We say the parametrization is \emph{nontrivial} otherwise.
\end{defn}

\begin{restatable}[Nontriviality Characterization]{thm}{nontrivial}
    \label{thm:nontrivial}
Suppose $\phi$ is tanh or $\sigma$-gelu for sufficiently small $\sigma$. A stable abc-parametrization is nontrivial iff $a_{L+1}+b_{L+1}+r=1$ or $2a_{L+1}+c=1$.
\end{restatable}

\begin{defn}[Feature Learning]\label{defn:featurelearning}
We say an abc-parametrization of an $L$-hidden layer MLP \emph{admits feature learning in the $l$th layer} if there exists some training routine such that
\begin{equation}
\Delta x_{t}^{l}(\xi)=\Omega_{*}(1)\label{eqn:xchanges}
\end{equation}
for some $t\ge0,\xi\in\isp$, where the hidden constant inside $\Omega$ can depend on the training routine, $t$, $\xi$, and the initial function values $f_{0}(\isp)$. We say the parametrization \emph{admits feature learning} if it does so in any layer.

We say the parametrization \emph{fixes the $l$th layer features} if for all training routine, 
\[
\|\Delta x_{t}^{l}(\xi)\|^{2}/n\asto0
\]
for all $t\ge0,\xi\in\isp$. We say the parametrization \emph{fixes all features} if it does so in every layer.

We make similar definitions as above replacing \emph{feature} with \emph{prefeature} and $x^{l}$ with $h^{l}$.
\end{defn}

Note that the probabilistic nature of $\Omega_{*}(1)$ means that \emph{no feature learning} does not imply \emph{fixing all features} (because $\Delta x_{t}^{l}(\xi)$ can just fluctuate wildly between 0 and infinity), but we will see that in the context of nontrivial stable abc-parametrizations, this is true.

\begin{rem}
    We note that this is a rather weak notion of ``feature learning'', as we only require that the embedding $x^L_t(\xi)$ changes from its initialization for \emph{some} scenario, rather than, say for \emph{generic} scenarios; nor do we speak at all about the ``quality'' of feature learning, e.g.\ how it helps downstream tasks.
    But our proofs (see \cref{sec:featurelearning}) will show that ``some scenario'' in fact implies much more general scenarios.
    In addition, we argue that such formal weakness is more than compensated by our experiments, which show that infinite-width limits of feature learning (in the sense defined here) abc-parametrized MLPs outperform finite MLPs and their NTK limits on tasks (namely, Word2Vec and few-shot learning) where \emph{feature learning}, in the colloquial notion of the phrase, is crucial.
\end{rem}

A somewhat stronger notion of feature learning is that the feature kernel evolves. This is, for example, essential for linear transfer learning such as in self-supervised learning of image data.
\begin{defn}[Feature Kernel Evolution]\label{defn:featurekernelevolve}
We say an abc-parametrization of an $L$-hidden layer MLP \emph{evolves the $l$th layer feature kernel} if there exists some training routine such that
\[
x_{t}^{l}(\xi)^{\trsp}x_{t}^{l}(\zeta)/n-x_{0}^{l}(\xi)^{\trsp}x_{0}^{l}(\zeta)/n=\Omega_{*}(1)
\]
for some $t\ge0,\xi,\zeta\in\isp$, where the hidden constant inside $\Omega$ can depend on the training routine, $t$, $\xi,\zeta$, and the initial function values $f_{0}(\isp)$. We say the parametrization \emph{evolves feature kernels} if it does so in any layer.

We say the parametrization \emph{fixes the $l$th layer feature kernel} if for all training routine, 
\[
x_{t}^{l}(\xi)^{\trsp}x_{t}^{l}(\zeta)/n-x_{0}^{l}(\xi)^{\trsp}x_{0}^{l}(\zeta)/n\asto0,\quad\text{as}\quad n\to\infty,
\]
for all $t\ge0,\xi,\zeta\in\isp$. We say the parametrization \emph{fixes all feature kernels} if it does so in every layer.

We make similar definitions as above replacing \emph{feature} with \emph{prefeature} and $x^{l}$ with $h^{l}$.
\end{defn}

Intuitively, for a stable parametrization, feature kernel evolution should imply feature learning (one can see the contrapositive easily). In fact, we shall see below they are equivalent notions.

On the other hand, from the NTK example, we know certain limits can be described entirely through kernel gradient descent with some kernel. Appropriately, we make the following definition.
\begin{defn}[Kernel Regime]\label{defn:kernelregime}
We say an abc-parametrization of an $L$-hidden layer MLP \emph{is in kernel regime} if there exists a positive semidefinite kernel $K:\isp^{2}\to\R$ such that for every training routine, the MLP function evolves under kernel gradient descent, i.e. there exist random variables $\mathring{f}_{t}(\xi)$ for each time $t\ge0$ and input $\xi\in\isp$ such that, as $n\to\infty$,\footnote{Here because we want to avoid topological issues arising for convergence in distribution of infinite sequences, we only require convergence in distribution jointly in all $\xi\in\isp$ and time $t$ below some cutoff $T$ for every finite $T$.} 
\[
    \{f_{t}(\xi)\}_{t\le T,\xi\in\isp}\distto\{\mathring{f}_{t}(\xi)\}_{t\le T,\xi\in\isp},\quad \forall T\ge1,\]
where $\distto$ denotes convergence in distribution, and
\begin{equation}
\mathring{f}_{t+1}(\xi)=\mathring{f}_{t}(\xi)-\eta K(\xi,\xi_{t})\loss'(\mathring{f}_{t}(\xi_{t}),y_{t}),\quad \forall t\ge0.\label{eqn:kernelGD}
\end{equation}
\end{defn}

Observe that, in kernel regime, $\mathring{f}_{t}(\xi)$ is deterministic conditioned on $\mathring{f}_{0}(\xi)$, as evident inductively from \cref{eqn:kernelGD}. For example, in the NTK limit, $\{\mathring{f}_{0}(\xi):\xi\in\isp\}$ is a nontrivial Gaussian Process (GP), but the function evolution conditioned on this GP is deterministic.

All of the concepts defined above are related to each other by the following theorem.
\begin{restatable}[Classification of abc-Parametrizations]{thm}{main}
    \label{thm:main}
Suppose $\phi$ is tanh or $\sigma$-gelu for sufficiently small $\sigma$. Consider a nontrivial stable abc-parametrization of an $L$-hidden layer MLP. Then
\begin{enumerate}
\item The following are equivalent to $r=0$
\begin{enumerate}
\item feature learning
\item feature learning in the $L$th layer
\item feature kernels evolution
\item feature kernel evolution in the $L$th layer
\item prefeature learning
\item prefeature learning in the $L$th layer
\item prefeature kernels evolution
\item prefeature kernel evolution in the $L$th layer
\end{enumerate}
\item The following are equivalent to $r>0$
\begin{enumerate}
\item kernel regime
\item fixes all features
\item fixes features in the $L$th layer
\item fixes all feature kernels
\item fixes feature kernel in the $L$th layer
\item fixes all prefeatures
\item fixes prefeatures in the $L$th layer
\item fixes all prefeature kernels
\item fixes prefeature kernel in the $L$th layer
\end{enumerate}
\item If there is feature learning \emph{or} feature kernel evolution \emph{or} prefeature learning \emph{or} prefeature kernel evolution in layer $l$, then there is feature learning \emph{and} feature kernel evolution \emph{and} prefeature learning \emph{and} prefeature kernel evolution in layers $l,\ldots,L$.
\item If $r=0$, then for all $\xi\in\isp$, $f_{0}(\xi)\asto0$ and $f_t(\xi) \asto \mathring f_t(\xi)$ for some deterministic $\mathring f_t(\xi)$.
However, the converse is not true.\label{item:FLDeterministic}
\item If $r>0$, $a_{L+1}+b_{L+1}+r>1$ and $2a_{L+1}+c=1$, then we have the Neural Network-Gaussian Process limit.
\end{enumerate}
\end{restatable}

In particular, Statement \ref{item:FLDeterministic} implies that feature learning, at least in our context, is incompatible with Bayesian, distributional perspectives of neural network limits, such as the NNGP limit.

The characterization above then trivially implies the following dichotomy.
\begin{restatable}[Dynamical Dichotomy]{cor}{dichotomy}
\label{cor:dichotomy}For $\phi$ being tanh or $\sigma$-gelu for sufficiently small $\sigma$, a nontrivial stable parametrization of an $L$-hidden layer MLP either admits feature learning or is in kernel regime, but not both.
\end{restatable}

\begin{rem}[The Role of the $\phi$ Assumption]
    \label{rem:nonlinspecificity}The dependence on $\phi$ being tanh or $\sigma$-gelu for sufficiently small $\sigma$ is only needed to explicitly construct a training routine that leads to feature learning for $r=0$. We expect this should be true for generic $\phi$, but we leave this for future work.
    We expand more on how the $\phi$ assumption is used below.
    
    To calculate the infinite width limit of any abc-parametrization rigorously, we only need the nonlinearity to have a polynomially bounded 2nd derivative (or more generally pseudo-Lipschitz, so as to apply the Master Theorem).
    The specific choice of tanh or gelu is needed to prove the part of the Dynamical Dichotomy that says a limit cannot be simultaneously in kernel regime and in feature learning regime (which, e.g.\ is not true for linear activation).
    To do so, we use Properties \ref{prop:geluprop} and \ref{prop:tanhprop} of tanh and gelu, expanded below.
    This is really for a more convenient proof, but we believe a more general approach should work for general nonlinearities.
    Our argument is as follows (this is also overviewed in the start of \cref{sec:featurelearning}):
    If $r = 0$, we show that a sufficiently small nonzero learning rate (scaled with width in the corresponding parametrization) in 1 SGD step 1) induces a change in the features but 2) the resulting change in the NN output is not linear in the loss derivative $\chi$.
    1) means it's feature learning, and 2) means it's not in kernel regime.
    This argument involves showing certain derivatives of certain expectations with respect to learning rate is positive.
    In the case of tanh and gelu, this is checked explicitly using Properties \ref{prop:geluprop} and \ref{prop:tanhprop}.
\end{rem}

\begin{rem}\label{rem:kerneltrivializes}
    The equivalence between kernel regime and fixed feature kernel implies that linear transfer learning is trivialized in any kernel regime limit.
    This is where the classifier layer of the pretrained network is discarded and a new one (potentially outputting to a new output space) is trained on top of the body of the pretrained network.
    But we can in fact say more: any \emph{nonlinear} transfer learning, where we replace the classifier layer with a neural network instead of a linear layer, is trivialized as well.
    In addition, linear or nonlinear transfer learning has no effect even if we finetune the entire network, instead of just the new classification network.
    The intuitive reason for this is that, as discussed in \cref{sec:appendixmotivating}, the effect of $\Delta x^L(\xi)$ on the output of the MLP is solely through the interaction with $W^{L+1}_0$.
    If $W^{L+1}, W^{L+2}, \ldots,$ are sampled anew, then this effect vanishes.
    We formalize this below.
\end{rem}

\begin{restatable}[Kernel Regime Limit Trivializes Transfer Learning]{thm}{KernelTrivializes}
    \label{thm:kerneltrivializes}
    Suppose $f$ is an $L$-hidden-layer MLP%
    \footnote{the ``pretrained network''}
    in a stable kernel regime parametrization.
    Let $A$ and $B$ be two training routines.%
    \footnote{the ``pretraining dataset'' and the ``finetuning dataset''}

    For any $T,t\ge 0$,%
    \footnote{the ``pretraining time'' and ``finetuning time''} we define a network\footnote{the ``finetuned network''}  $g_{T;t}$ as follows.
    Train $f$ on $A$ for $T$ steps to obtain $f_T$.
    Then discard $W^{L+1}$ in $f_T$ and extend the body of $f_T$ into an $M$-hidden-layer MLP $g$, where $M\ge L$.%
    \footnote{If $M=L$, then this is linear transfer learning where we replace just the last layer of $f$; otherwise, it's nonlinear transfer learning.}
    Parametrize and initialize the new weights of $g$ according to any stable abc-parametrization that extends the parametrization of $f$.
    Train $g$ on $B$ for $t$ steps to obtain $g_{T;t}$.
    
    Then
    \begin{enumerate}
        \item (Finetuning the whole network)
        As $n\to\infty$, for any $\xi\in\isp$ and $T,t\ge0$, \[g_{T;t}(\xi) -  g_{0;t}(\xi)\asto 0.\]
        \item (Training only the classifier)
        The above is true even if we define $g_{T;t}$ by only training the new weights $W^{L+1}, \ldots, W^M$ in $g$.
    \end{enumerate}
\end{restatable}

\subparagraph*{The Organization for the Proof of Our Main Results Above}
\begin{defn}
Below, we will abbreviate \emph{abc-parametrization of an $L$-layer MLP} to just \emph{parametrization}. We will call parametrizations satisfying the conditions of \cref{thm:stabilityconditions} \emph{pseudostable} while we try to prove \cref{thm:stabilityconditions} (which, in this terminology, says stability and pseudostability are equivalent).
\end{defn}

We first characterize stability at initialization and prove \cref{eq:initstable} holds iff \cref{eq:actlogitinit} (\cref{sec:Stability-at-Initialization}). Then, we describe the Tensor Program encoding the SGD of an MLP, assuming its parametrization is pseudostable. The Master Theorem then naturally lets us calculate its infinite-width limit. We then divide into the case of $r>0$ and $r=0$. In the former case, we show the infinite-width limit is described by kernel gradient descent as in \cref{eqn:kernelGD}. In the latter case, we construct a training routine where feature learning occurs and where the limit is \emph{not} given by kernel gradient descent for any kernel. Finally, in \cref{sec:Altogether}, we combine all of our analyses to prove the main results in this section.

\subsection{Stability at Initialization}

\label{sec:Stability-at-Initialization}

In this section, we characterize stability at initialization, which will form a foundation for our later results.
\begin{thm}
\label{thm:initstable}Assume $\phi$ is not zero almost everywhere. For any parametrization, \cref{eq:initstable} holds iff \cref{eq:actlogitinit} holds, i.e. the following are equivalent
\end{thm}

\begin{enumerate}
\item For every nonzero input $\xi\in\isp$,
\[
h_{0}^{l}(\xi),x_{0}^{l}(\xi)=\Theta_{\xi}(1),\forall l\in[L],\quad\text{and}\quad\EV f_{0}(\xi)^{2}=O_{\xi}(1),
\]
where the expectation is taken over the random initialization.
\item $a_{1}+b_{1}=0;\quad a_{l}+b_{l}=1/2,\ \forall l\in[2,L];\quad a_{L+1}+b_{L+1}\ge1/2.$
\end{enumerate}
\begin{proof}
Fix an input $\xi\ne0$. Here, because we focus on initialization, we will suppress the time 0 subscript and $\xi$ dependence of $h^{l},x^{l}$ to mean $t=0$, applied to $\xi$. 

Obviously, $h^{1}=W^{1}\xi$ is a Gaussian vector with $\Gaus(0,n^{-(a_{1}+b_{1})}\|\xi\|^{2})$ coordinates, so $h^{1}=\Theta_{\xi}(1)$ iff $a_{1}+b_{1}=0$. Assume $a_{1}+b_{1}=0$. By Law of Large Numbers, $\frac{1}{n}\|x^{1}\|^{2}\asto\EV\phi(Z^{h^{1}})^{2}$ where $Z^{h^{1}}=\Gaus(0,\|\xi\|^{2})$. Since $\phi$ is not almost everywhere zero and $\xi\ne0$, this expectation is nonzero so that $x^{1}=\Theta_{\xi}(1)$.

We construct the following Tensor Program: the lone initial vector is $h^{1}$, the initial matrices are $\widehat{W}^{l},2\le l\le L$, and initial scalars $\theta_{l}\defeq n^{1/2-(a_{l}+b_{l})}$. We sample $h_{\alpha}^{1}\sim\Gaus(0,\|\xi\|^{2})$ and $\widehat{W}_{\alpha\beta}^{l}\sim\Gaus(0,1/n)$. Mathematically, we will represent $W^{l}=\theta_{l}\widehat{W}^{l}$. The program is then given by
\[
x^{l}=\phi(h^{l}),\forall l\in[L],\quad\hat{h}^{l}=\widehat{W}^{l}x^{l-1},h^{l}=\theta_{l}\hat{h}^{l},\forall l\in[2,L],
\]
where we used Nonlin, MatMul, and Nonlin (with parameter $\theta_{l}$).

Suppose $a_{l}+b_{l}=1/2$ (i.e. $\theta_{l}=1$) for all $2\le l\le L$. Then, $Z^{h^{l}}=Z^{\hat{h}^{l}}=\Gaus(0,\EV\phi(Z^{h^{l-1}})^{2})$ for each $l\le L$. Because $\phi$ is not everywhere zero, this inductively implies $\EV(Z^{h^{l}})^{2}>0$ (and so also $\EV(Z^{x^{l}})^{2}>0$) for all $l\le L$. By the Master Theorem, $\frac{1}{n}\|h^{l}\|^{2}\asto\EV(Z^{h^{l}})^{2}$ and $\frac{1}{n}\|x^{l}\|^{2}\asto\EV(Z^{x^{l}})^{2}$ so this implies $h^{l},x^{l}=\Theta_{\xi}(1)$ for all $l\le L$ as desired.

Conversely, suppose $m$ is the smallest $l\ge2$ such that $a_{l}+b_{l}\ne1/2$. Then by the above reasoning, $\hat{h}^{m}=\Theta_{\xi}(1)$ so $h^{m}=\Theta_{\xi}(n^{1/2-(a_{l}+b_{l})})$ is either blowing up to $\infty$ or shrinking to 0 with $n$. This shows that $h^{l},x^{l}=\Theta_{\xi}(1)$ for all $l\le L$ iff $a_{1}+b_{1}=0$ and $a_{l}+b_{l}=1/2$ for all $2\le l\le L$.

Finally, if $a_{1}+b_{1}=0$ and $a_{l}+b_{l}=1/2$ for all $2\le l\le L$, then we see $\EV f_{0}(\xi)^{2}=(n^{1/2-(a_{L+1}+b_{L+1})})^{2}\EV\|Z^{x^{L}}\|^{2}/n$. For large $n$, this is $\Theta_{\xi}((n^{1/2-(a_{L+1}+b_{L+1})})^{2})$ and is $O_{\xi}(1)$ iff $a_{L+1}+b_{L+1}\ge1/2$. 
\end{proof}
\begin{defn}
We say a parametrization is \emph{initialization-stable} if it satisfies \cref{eq:initstable} (or equivalently, \cref{eq:actlogitinit}).
\end{defn}

\subsection{Program Setup}

\label{sec:Program-Setup}

In the next section, we construct the Tensor Program that encodes the training of an $L$-hidden layer MLP under an abc-parametrization. Here we first describe the initial matrices, vectors, and scalars of the program, along with necessary notations. 

We first remark on a simplification we will make to streamline the proof.

\subparagraph*{The Size of $W_{0}^{L+1}$ vs $\Delta W_{t}^{L+1}$}

By construction, $W_{0}^{L+1}=\Theta(n^{-(a_{L+1}+b_{L+1})})$. If $x_{t}^{L}(\xi)=\Theta(1)$ as in a stable parametrization, then $\Delta W_{t}^{L+1}=\Theta(n^{-(2a_{L+1}+c)})$. Therefore, if $a_{L+1}+b_{L+1}\le2a_{L+1}+c$, then $W_{0}^{L+1}$ is at least as large as $\Delta W_{t}^{L+1}$, so that $W_{t}^{L+1}$ will stay the same order (in terms of $n$) for all $t$. If the reverse inequality is true, then $W_{0}^{L+1}$ is smaller than $W_{t}^{L+1}$for $t\ge1$. This in particular implies that the gradients at time 0 is smaller than gradients at subsequent times. For example, we can take $a_{L+1}+b_{L+1}\to\infty$ while fixing $2a_{L+1}+c$, in which case $W_{0}^{L+1}=0$ and the weight gradients at initialization are all 0 except for that of $W^{L+1}$. One can thus think of this as a ``lag'' in the training dynamics for 1 step.
\begin{assm}
For clarity of the proof, we will assume $a_{L+1}+b_{L+1}\le2a_{L+1}+c$, i.e. $W_{t}^{L+1}$ stays the same order for all $t$. The case of $a_{L+1}+b_{L+1}>2a_{L+1}+c$, corresponding to a 1-step ``lag'' as explained above, can be dealt with similarly. We will remark whenever this requires some subtlety.
\end{assm}
For the construction of the program and the application of the Master Theorem, we will also assume the following for the rest of this paper.
\begin{assm}\label{assm:phismooth}
$\phi'$ is pseudo-Lipschitz and not almost everywhere zero.
\end{assm}

\subparagraph*{Initial Matrices, Vectors, Scalars}

We will assume the parametrization is initialization-stable. For ease of presentation, we also assume the input dimension $d=1$.
\begin{enumerate}
\item Initial matrices: $W_{0}^{2},\ldots,W_{0}^{L}$, sampled like $(W_{0}^{l})_{\alpha\beta}\sim\Gaus(0,1/n)$.
\item Initial vectors: input layer matrix $W_{0}^{1}\in\R^{n\times1}$ and \emph{normalized }output layer matrix $\widehat{W}_{0}^{L+1}\defeq W_{0}^{L+1}n^{a_{L+1}+b_{L+1}}\in\R^{1\times n}$, sampled like $(W_{0}^{1})_{\alpha},(\widehat{W}_{0}^{L+1})_{\alpha}\sim\Gaus(0,1)$.
\item Initial scalars: We define the following scalars (where we explain the intuition in parenthesis). The reader can skip this part on a first read but come back when referred to.
\begin{enumerate}
\item ($n$ times the scale of coordinates of $\Delta W_{t}^{l}$) For $l\ge2$, define 
\[
\theta_{W^{l}}\defeq n^{-(a_{L+1}+b_{L+1}+c-1+2a_{l})}
\]
\item (scale of coordinates of $\Delta W_{t}^{1}$ and $\Delta h_{t}^{1}$) Define
\[
\theta_{1}=\theta_{W^{1}}\defeq n^{-(a_{L+1}+b_{L+1}+c+2a_{1})}
\]
\item (scale of coordinates of $\Delta W_{t}^{L+1}$)
\[
\theta_{L+1}=\theta_{W^{L+1}}\defeq n^{-2a_{L+1}-c}
\]
\item (scale of $\Delta h_{t}^{l}$ and $\Delta x_{t}^{l}$) For $l\in[L]$, define
\begin{align*}
\theta_{h^{l}}=\theta_{x^{l}}=\theta_{l}&\defeq\max_{m\le l}\theta_{W^{m}}=\max(\theta_{W^{l}},\theta_{l-1})\numberthis\label{eq:thetal}\\
&=n^{-(a_{L+1}+b_{L+1}+c-1+\min_{m=1}^{l}\left(2a_{m}+\ind(m=1)\right))}
\end{align*}
Note that $\theta_{L}=n^{-r}$ with $r$ defined in \cref{def:r}.
\item (scale of $W_{t}^{L+1})$
\[
\theta_{f}\defeq n^{-(a_{L+1}+b_{L+1})}
\]
\item (convenience scalars)
\begin{align*}
\theta_{x^{l-1}/h^{l}} & =\theta_{x^{l-1}}/\theta_{h^{l}}\\
\theta_{W^{l}/h^{l}} & =\theta_{W^{l}}/\theta_{h^{l}}\\
\theta_{W^{l}x^{l-1}/h^{l}} & =\theta_{W^{l}}\theta_{x^{l-1}}/\theta_{h^{l}}\\
\theta_{L+1/f} & =\theta_{L+1}/\theta_{f}\\
\theta_{L+1}' & =n\theta_{L+1}=n^{1-2a_{L+1}-c}\\
\theta_{Lf}' & =n\theta_{L}\theta_{f}=n^{1-(r+a_{L+1}+b_{L+1})}
\end{align*}
\item Depending on the the value of $a_{L+1}+b_{L+1}$, we will also construct the values of $f$ at initialization as initial scalars. See \cref{subsec:First-Forward-Pass} for an explanation.
\end{enumerate}
\end{enumerate}
By our assumption that $a_{L+1}+b_{L+1}\le2a_{L+1}+c$, the pseudostability inequalities of \cref{thm:stabilityconditions} imply all of these $\theta$s either converge to 0 or stay constant at 1. This means that, assuming appropriate regularity conditions on the nonlinearities and rank stability, we can apply the Master Theorem (if $\theta$ blows up to $\infty$ then we can't do that).

\subparagraph*{Notations}

We use $:=$ to more clearly denote assignment happening in the program, as opposed to mathematical equality. To clearly demonstrate the application of Nonlin, we will also freely introduce function symbols $\Psi$ to put things into Nonlin form.

\subparagraph*{Preview of Names for Vectors}

In the program, for each $z\in\{x^{l},h^{l}\}_{l}$, we will construct vectors $\del z_{t}(\xi)$ to mathematically represent $\theta_{z}^{-1}(z_{t}(\xi)-z_{t-1}(\xi))$ (intuition: change in $z$ scaled to have $\Theta(1)$ coordinates). Similarly, for $w\in\{W^{L+1},W^{1}\}$, we will construct $\del w_{t}$ to mathematically represent $\theta_{w}^{-1}(w_{t}-w_{t-1})$ (intuition: change in $w$ scaled to have $\Theta(1)$ coordinates). Then, mathematically, $z_{t}(\xi)=z_{t-1}(\xi)+\theta_{z}\del z_{t}(\xi),w_{t}=w_{t-1}+\theta_{w}\del w_{t}$.

We will also construct $dz$ to mathematically represent $\theta_{f}^{-1}\nabla_{z}f$ (intuition: gradient $\nabla_{z}f$ scaled to have $\Theta(1)$ coordinates). For weight changes, we have the following identity
\begin{equation}
W_{t}^{l}-W_{t-1}^{l}=-\eta n^{-c}\chi_{t-1}n^{-2a_{l}}\theta_{f}dh_{t-1}^{l}x_{t-1}^{l-1\trsp}=-\eta\chi_{t-1}\theta_{W^{l}}\frac{1}{n}h_{t-1}^{l}x_{t-1}^{l-1\trsp},\quad\forall l\in[2,L],\label{eq:delWl}
\end{equation}
 and for $l=1$,
\begin{equation}
W_{t}^{l}-W_{t-1}^{l}=-\eta n^{-c}\chi_{t-1}n^{-2a_{l}}\theta_{f}dh_{t-1}^{l}\xi_{t-1}^{\trsp}=-\eta\chi_{t-1}\theta_{W^{l}}h_{t-1}^{l}\xi_{t-1}^{\trsp}.\label{eq:delW1}
\end{equation}

\subsection{Program Construction}
\label{sec:ProgramConstruction}

Here we construct the Tensor Program encoding the SGD of an MLP. We separately describe the first forward and backward passes followed by the later forward and backward passes.

\subsubsection{First Forward Pass}

\label{subsec:First-Forward-Pass}

For every $\xi\in\isp$, we compute $\ensuremath{h_{0}^{1}(\xi):=W_{0}^{1}\xi\in\R^{n}}$ via Nonlin (as $\Psi(W_{0}^{1};\xi)$, where $\Psi$ is multiplication by $\xi$), and we construct the following vectors via Nonlin and MatMul

\begin{align}
x_{0}^{l}(\xi):=\phi(h_{0}^{l}(\xi))\in\R^{n},\quad h_{0}^{l+1}(\xi):=W_{0}^{l+1}x_{0}^{l}(\xi)\in\R^{n},\quad\text{for \ensuremath{l=1,\ldots,L-1},}
\end{align}

\subparagraph*{Function Output}

The first output is $f_{0}(\xi)=W_{0}^{L+1}x_{0}^{L}(\xi)$, but we will define $f_{0}(\xi)$ in the program slightly differently. 

\paragraph{Case when \texorpdfstring{$a_{L+1}+b_{L+1}>1/2$}{a\_\{L+1\}+b\_\{L+1\}>1/2}}

Then $f_{0}(\xi)\asto0$ for all $\xi\in\isp$. In the program, we will construct $f_{0}(\xi)$ as an \emph{initial scalar} mathematically defined by $W_{0}^{L+1}x_{0}^{L}(\xi)$.\footnote{It is completely OK to define an initial scalar using randomness from other parts of the program, as long as this scalar converges almost surely to a deterministic limit}\footnote{We cannot define it using a \refMoment{} instruction because, intuitively, the mechanism of this convergence is through CLT, not Law of Large Numbers.}

\paragraph{Case when \texorpdfstring{$a_{L+1}+b_{L+1}=1/2$}{a\_\{L+1\}+b\_\{L+1\}=1/2}}

If $a_{L+1}+b_{L+1}=1/2$, then $f_{0}(\xi)$ converges to a nontrival Gaussian via CLT \citep{TP1}, so we will condition on $f_{0}(\xi)$ for all $\xi\in\isp.$ Given values $g(\xi)\in\R$ for all $\xi\in\isp$, let $\Ee$ be the event that $f_{0}(\xi)=\frac{1}{\sqrt{n}}\widehat{W}_{0}^{L+1}x_{0}^{L}(\xi)$ equals $g(\xi)$ for all $\xi\in\isp$. The distribution of $\widehat{W}_{0}^{L+1}$ conditioned on $\Ee$ is given by
\[
\widehat{W}_{0}^{L+1}\disteq_{\Ee}\sqrt{n}X^{+}g+\Pi\widetilde{W}_{0}^{L+1}
\]
where $\widetilde{W}_{0}^{L+1}$ is an iid copy of $\widehat{W}_{0}^{L+1}$, $g\in\R^{\isp}$ is the vector of $\{g(\xi):\xi\in\isp\}$, $X\in\R^{\isp\times n}$ has $x_{0}^{L}(\xi)$ as rows, and $\Pi$ is the orthogonal projection into the orthogonal complement of the space spanned by $\{x_{0}^{L}(\xi):\xi\in\isp\}$. Here $X^{+}$ denotes the pseudo-inverse of $X$.

By standard formulas for pseudo-inverse and orthogonal projection, we can write $X^{+}=\frac{1}{n}X^{\trsp}(XX^{\trsp}/n)^{+},\Pi=I-\frac{1}{n}X^{\trsp}(XX^{\trsp}/n)^{+}X$.

Let $\Sigma\defeq XX^{\trsp}/n$ and $\gamma\defeq(X\widetilde{W}_{0}^{L+1}/n)$. Then $\Pi\widetilde{W}_{0}^{L+1}=\widetilde{W}_{0}^{L+1}-X^{\trsp}\Sigma^{+}\gamma$, and $\sqrt{n}X^{+}g=\frac{1}{\sqrt{n}}X^{\trsp}\Sigma^{+}g$.

By the Master Theorem, $\gamma\asto0$ because $\widetilde{W}_{0}^{L+1}$ is independent from $X$, and $\Sigma\asto\mathring{\Sigma}$ for some PSD matrix $\mathring{\Sigma}$. At this point in the program, all scalars we used (like $\xi$) are constant with $n$ and can be absorbed into nonlinearities. By the rank stability property of any program without scalars \citep{TP3}, the rank of $\Sigma$ is fixed for large enough $n$, almost surely, so $\Sigma^{+}\asto\mathring{\Sigma}^{+}$ by the continuity of pseudo-inverse on fixed rank matrices.

We will now replace $\widehat{W}_{0}^{L+1}$ in the program with
\[
\widehat{W}_{\Ee}^{L+1}\defeq X^{\trsp}\left(\Sigma^{+}\frac{g}{\sqrt{n}}\right)+\widetilde{W}_{0}^{L+1}-X^{\trsp}\left(\Sigma^{+}\gamma\right)
\]
constructed using Nonlin, where $\left(\Sigma^{+}\frac{g}{\sqrt{n}}\right)$ and $\left(\Sigma^{+}\gamma\right)$ are finite dimensional and formally considered (collections of) scalars involved as coefficients for linear combination of rows of $X$. Since $\Sigma^{+}\frac{g}{\sqrt{n}},\Sigma^{+}\gamma\asto0$, we have $Z^{\widehat{W}_{\Ee}^{L+1}}=Z^{\widetilde{W}_{0}^{L+1}}$. Intuitively, this means that, even after conditioning on $f_{0}=g$, the conditional distribution of $\widetilde{W}_{0}^{L+1}$ is practically the same as the original distribution. We can then proceed exactly as in the case when $a_{L+1}+b_{L+1}>1/2$, with $\widehat{W}_{\Ee}^{L+1}$ taking the role of $\widetilde{W}_{0}^{L+1}$. The program then encodes the evolution of $f$ \emph{conditioned on} $f_{0}(\xi)=g(\xi),\forall\xi\in\isp$.\footnote{Formally, we can also have $\{g(\xi):\xi\in\isp\}$ as initial scalars, but since they are fixed with $n$, they can be absorbed into the Nonlin that defines $\widehat{W}_{\Ee}^{L+1}$.}

\begin{assm}
For the above reason, we will assume $a_{L+1}+b_{L+1}>1/2$, and remark whenever the case $a_{L+1}+b_{L+1}=1/2$ involves subtleties.
\end{assm}

\subsubsection{First Backward Pass}

Next, we write the backward pass
\begin{align*}
dx_{0}^{L}(\xi) & :=\widehat{W}_{0}^{L+1}\\
dh_{0}^{l}(\xi) & :=dx_{0}^{l}(\xi)\odot\phi'(h_{0}^{l}(\xi))\\
dx_{0}^{l-1}(\xi) & :=W_{0}^{l\trsp}dh_{0}^{l}(\xi)
\end{align*}
where, recall, $dz$ mathematically equals $\theta_{f}^{-1}\nabla_{z}f$.

For $\xi=\xi_{0}$ and its label $y_{0}$, we define the first loss derivative as
\[
\chi_{0}:=\loss'(f_{0}(\xi_{0}),y_{0})\asto\mathring{\chi}_{0}(\xi)=\loss'(0,y_{0})
\]
where the convergence is because $\loss'$ is continuous by assumption.

We also define
\[
\del W_{1}^{L+1}:=-\eta\chi_{0}x_{0}^{L}(\xi_{0})
\]
to represent the (normalized) change in $W^{L+1}$ due to the first gradient step.

\subsubsection{\texorpdfstring{$t$}{t'}th Forward Pass, \texorpdfstring{$t\ge1$}{t >= 1}}

\subparagraph*{Overview}

We iteratively define $\del z_{t}(\xi)$ to mathematically represent $\theta_{z}^{-1}(z_{t}(\xi)-z_{t-1}(\xi))$, for $z\in\{x^{l},h^{l}\}_{l}$. Then we eventually set 
\[
z_{t}(\xi):=z_{0}(\xi)+\theta_{z}\del z_{1}(\xi)+\cdots+\theta_{z}\del z_{t}(\xi).
\]
Likewise, we will define $\del W_{t}^{L+1}$ so that $W_{t}^{L+1}=\theta_{f}\widehat{W}_{0}^{L+1}+\theta_{L+1}(\del W_{1}^{L+1}+\cdots+\del W_{t}^{L+1})$. In the program, we will not directly use $W_{t}^{L+1}$ but instead use 
\begin{equation}
\widehat{W}_{t}^{L+1}:=\widehat{W}_{0}^{L+1}+\theta_{L+1/f}(\del W_{1}^{L+1}+\cdots+\del W_{t}^{L+1})\label{eq:hatWL+1}
\end{equation}
where $\theta_{L+1/f}=\theta_{L+1}/\theta_{f}$. Mathematically, $\widehat{W}_{t}^{L+1}=\theta_{f}^{-1}W_{t}^{L+1}$.

Recall we shorthand $z_{t}=z_{t}(\xi_{t})$ for all $z\in\{x^{l},h^{l},dx^{l},dh^{l}\}_{l}\cup\{f,\chi\}$.

\subparagraph*{The Construction of (Pre)Activations}

We start with $h=h^{1}$: By \cref{eq:delW1}, we have
\[
\del h_{t}(\xi):=-\eta\chi_{t-1}\xi_{t-1}^{\trsp}\xi dh_{t-1}=\Psi(dh_{t-1};\xi_{t-1}^{\trsp}\xi,\eta\chi_{t-1}).
\]
(Notationally, recall we freely introduce function symbols $\Psi$ to clarify the way we apply Nonlin). For higher layers, if $h=h^{l}$, $x=x^{l-1}$, and $W=W^{l}$, then $h=Wx$. By \cref{eq:delWl}, we have, mathematically,
\begin{align*}
\theta_{h}\del h_{t}(\xi) & =\theta_{x}W_{t-1}\del x_{t}(\xi)+(W_{t}-W_{t-1})x_{t}(\xi)\\
 & =\theta_{x}\left(W_{0}\del x_{t}(\xi)+\sum_{s=1}^{t-1}(W_{s}-W_{s-1})\del x_{t}(\xi)\right)+(W_{t}-W_{t-1})x_{t}(\xi)\\
 & =\theta_{x}\left(W_{0}\del x_{t}(\xi)-\eta\theta_{W}\sum_{s=1}^{t-1}\chi_{s-1}\frac{x_{s-1}^{\trsp}\del x_{t}(\xi)}{n}dh_{s-1}\right)-\eta\chi_{t-1}\theta_{W}\frac{x_{t-1}^{\trsp}x_{t}(\xi)}{n}dh_{t-1}
\end{align*}
Recall $\theta_{x/h}=\theta_{h}^{-1}\theta_{x},\theta_{W/h}=\theta_{h}^{-1}\theta_{W},\theta_{Wx/h}=\theta_{h}^{-1}\theta_{W}\theta_{x}$. With $c_{s}$ denoting $\frac{x_{s}^{\trsp}\del x_{t}(\xi)}{n}$, we construct
\begin{align*}
\del h_{t}(\xi) & :=\theta_{x/h}W_{0}\del x_{t}(\xi)-\eta\theta_{Wx/h}\sum_{s=1}^{t-1}\chi_{s-1}c_{s-1}dh_{s-1}-\eta\chi_{t-1}\theta_{W/h}c_{t-1}dh_{t-1}\\
 & =\Psi(W_{0}\del x_{t}(\xi),dh_{0},\ldots,dh_{t-1};\eta,\theta_{x/h},\theta_{Wx/h},\theta_{W/h},\{c_{s},\chi_{s}\}_{s=0}^{t-1})
\end{align*}
If $x=x^{l}$, $h=h^{l}$, then $x=\phi(h)$, and (using $\theta_{x}=\theta_{h}$ (\cref{eq:thetal})),
\begin{align}
\del x_{t}(\xi) & :=\theta_{h}^{-1}(\phi(h_{t-1}(\xi)+\theta_{h}\del h_{t}(\xi))-\phi(h_{t-1}(\xi)))\nonumber \\
 & =\Psi(h_{t-1}(\xi),\del h_{t}(\xi);\theta_{h})\label{eq:delx}
\end{align}
where $\Psi$ is precisely the difference quotient for the function $\phi$.%
\footnote{The pseudo-Lipschitzness of $\phi'$ assumed in \cref{assm:phismooth} implies that $\Psi$ here is pseudo-Lipschitz, so that we can ultimately apply our Master Theorem.}

\subparagraph*{The Function Outputs}

We do not construct $f_{t}(\xi)$ directly, but rather through scalars $\del f_{t}(\xi)=f_{t}(\xi)-f_{t-1}(\xi)$, so that 
\[
f_{t}(\xi):=f_{0}(\xi)+\del f_{1}(\xi)+\cdots+\del f_{t}(\xi).
\]

Mathematically, $\del f_{t}(\xi)=\theta_{L+1}\del W_{t}^{L+1}x_{t}^{L}(\xi)+W_{t-1}^{L+1}\theta_{L}\del x_{t}^{L}(\xi)$, but we shall write it slightly differently in the program:
\[
\del f_{t}(\xi):=\theta_{L+1}'\frac{\del W_{t}^{L+1}x_{t}^{L}(\xi)}{n}+\theta_{Lf}'\frac{\widehat{W}_{t-1}^{L+1}\del x_{t}^{L}(\xi)}{n}
\]
where $\theta_{L+1}'=n\theta_{L+1},\theta_{Lf}'=n\theta_{L}\theta_{f}$ and $\widehat{W}_{t-1}^{L+1}$ is constructed in \cref{eq:hatWL+1}.

\subsubsection{\texorpdfstring{$t$}{t'}th Backward Pass, \texorpdfstring{$t\ge1$}{t >= 1}}

In the last layer, we construct
\[
dx_{t}^{L}(\xi):=\widehat{W}_{t}^{L+1}.
\]
For each $l=L,\ldots,1$ for $dh^{l}$ and $l=L,\ldots,2$ for $dx^{l-1}$, we also calculate
\begin{align*}
dh_{t}^{l}(\xi) & :=dx_{t}^{l}(\xi)\odot\phi'(h_{t}^{l}(\xi))\\
dx_{t}^{l-1}(\xi) & :=W_{0}^{l\trsp}dh_{t}^{l}(\xi)-\eta\theta_{W^{l}}\sum_{s=0}^{t-1}\chi_{s}c_{s}x_{s}^{l-1}\\
 & =\Psi(W_{0}^{l\trsp}dh_{t}^{l}(\xi),x_{0}^{l-1},\ldots,x_{t-1}^{l-1};\eta\theta_{W^{l}},\{\chi_{s},c_{s}\}_{s=0}^{t-1})
\end{align*}
where $c_{s}=\frac{dh_{s}^{l\trsp}dh_{t}^{l}(\xi)}{n}$. For $\xi=\xi_{t}$ and its label $y_{t}$, we define%
\footnote{Here we use \refMoment{} with the function $\phi(; f_t(\xi_t)) = \loss'(f_t(\xi_t), y_t)$ with no input and one parameter (we absorb $y_t$ into $\phi$ since it does not change with $n$).
The continuity of $\loss'$ in its first argument satisfies \cref{assm:MasterTheoremSmoothness}(1), so the Master Theorem can apply.}
\[
\chi_{t}:=\loss'(f_{t}(\xi_{t}),y_{t}).
\]
Finally, we compute the (normalized) change in $W^{L+1}$ after this SGD update.
\[
\del W_{t+1}^{L+1}:=-\eta\chi_{t}x_{t}^{L}(\xi_{t}).
\]

\subsection{The Infinite-Width Limit}
\label{sec:infwidthlimit}

In this section, we describe the $Z$ random variables (\cref{defn:netsortplusKeyIntuit}) corresponding to the vectors of the program constructed above. According to the Master Theorem, each such vector $z$ will have roughly iid coordinates distributed like $Z^{z}$ in the large $n$ limit.

Let $\mathring{\theta}_{\bullet}$ denote the limit of any $\theta_{\bullet}$ in \cref{sec:Program-Setup}. If pseudostability holds, then $\mathring{\theta}_{\bullet}$ is either 0 or 1, as one can easily verify. We can construct the $Z$ random variables for each vector in the program, as follows.
\begin{enumerate}
\item For the first forward and backward passes, we have,
\begin{align*}
    Z^{h_{0}^{1}(\xi)}&=\xi Z^{W_{0}^{1}},&
    Z^{x_{0}^{l}(\xi)}&=\phi(Z^{h_{0}^{l}(\xi)}),&
    Z^{h_{0}^{l+1}(\xi)}&=Z^{W_{0}^{l+1}x_{0}^{l}(\xi)},\\
    Z^{dx_{0}^{L}(\xi)}&=Z^{\widehat{W}_{0}^{L+1}},&
    Z^{dh_{0}^{l}(\xi)}&=Z^{dx_{0}^{l}(\xi)}\phi'(Z^{h_{0}^{l}(\xi)}),&
    Z^{dx_{0}^{l-1}(\xi)}&=Z^{W_{0}^{l\trsp}dh_{0}^{l}(\xi)}
\end{align*}
\item For $z\in\{x^{l},h^{l}\}_{l}$, we have
\begin{equation}
Z^{z_{t}(\xi)}=Z^{z_{0}(\xi)}+\mathring{\theta}_{z}Z^{\del z_{1}(\xi)}+\cdots+\mathring{\theta}_{z}Z^{\del z_{t}(\xi)}\label{eq:Zzt}
\end{equation}
\item For $l\in[L],x=x^{l},h=h^{l}$, we have $Z^{\del x_{t}(\xi)}=\Psi(Z^{h_{t-1}(\xi)},Z^{\del h_{t}(\xi)};\mathring{\theta}_{h})$ where $\Psi$ is as in \cref{eq:delx}. If $\mathring{\theta}_{h}=0$ (e.g. if $r>0$), then 
\begin{equation}
Z^{\del x_{t}(\xi)}=\phi'(Z^{h_{t-1}(\xi)})Z^{\del h_{t}(\xi)}.\label{eq:ZdelxWhenr>0}
\end{equation}
Otherwise, $\mathring{\theta}_{h}=1$, and 
\begin{equation}
Z^{\del x_{t}(\xi)}=\phi(Z^{h_{t}(\xi)})-\phi(Z^{h_{t-1}(\xi)}).\label{eq:ZdelxWhenr=00003D0}
\end{equation}
\item For $h=h^{1}$, we have
\[
Z^{\del h_{t}(\xi)}=-\eta\mathring{\chi}_{t-1}\xi_{t-1}^{\trsp}\xi Z^{dh_{t-1}}.
\]
\item For $l\ge2,h=h^{l},x=x^{l-1},W=W^{l}$, we have
\begin{align*}
Z^{\del h_{t}(\xi)}&=\mathring{\theta}_{x/h}Z^{W_{0}\del x_{t}(\xi)}-\eta\mathring{\theta}_{Wx/h}\sum_{s=0}^{t-2}\mathring{\chi}_{s}Z^{dh_{s}}\EV Z^{x_{s}}Z^{x_{t}(\xi)}\\
    &\phantomeq\quad
    -\eta\mathring{\chi}_{t-1}\mathring{\theta}_{W/h}Z^{dh_{t-1}}\EV Z^{x_{t-1}}Z^{x_{t}(\xi)}
    \numberthis\label{eq:Zdelh}
\end{align*}
where at least one of $\mathring{\theta}_{x/h}$ and $\mathring{\theta}_{W/h}$ equals 1. As usual, here we have the \refZhat{}-\refZdot{} decomposition of $Z^{W_{0}\del x_{t}(\xi)}$.
\begin{align*}
Z^{W_{0}\del x_{t}(\xi)} & =\hat{Z}^{W_{0}\del x_{t}(\xi)}+\dot{Z}^{W_{0}\del x_{t}(\xi)}\\
 & =\hat{Z}^{W_{0}\del x_{t}(\xi)}+\sum_{s=0}^{t-1}Z^{dh_{s}}\EV\frac{\partial Z^{\del x_{t}(\xi)}}{\partial\hat{Z}^{W_{0}^{\trsp}dh_{s}}}.
\end{align*}
\item For last layer weight
\begin{equation}
Z^{\del W_{t}^{L+1}}=-\eta\mathring{\chi}_{t-1}Z^{x_{t-1}^{L}}\label{eq:ZdelWlast}
\end{equation}
and
\begin{equation}
Z^{\widehat{W}_{t}^{L+1}}=Z^{\widehat{W}_{0}^{L+1}}+\mathring{\theta}_{L+1/f}(Z^{\del W_{1}^{L+1}}+\cdots+Z^{\del W_{t}^{L+1}})\label{eq:ZhatW}
\end{equation}
\item The output deltas have limits 
\begin{equation}
\del\mathring{f}_{t}(\xi)=\mathring{\theta}_{L+1}'\EV Z^{\del W_{t}^{L+1}}Z^{x_{t}^{L}(\xi)}+\mathring{\theta}_{Lf}'\EV Z^{\widehat{W}_{t-1}^{L+1}}Z^{\del x_{t}^{L}(\xi)}\label{eq:delflimit}
\end{equation}
and
\[
\mathring{f}_{t}(\xi)=\del\mathring{f}_{1}(\xi)+\cdots+\del\mathring{f}_{t}(\xi).
\]
\item For gradients:
\begin{align*}
Z^{dx_{t}^{L}(\xi)} & =Z^{\widehat{W}_{t}^{L+1}}\\
Z^{dh_{t}^{l}(\xi)} & =Z^{dx_{t}^{l}(\xi)}\phi'(Z^{h_{t}^{l}(\xi)})\\
Z^{dx_{t}^{l-1}(\xi)} & =Z^{W_{0}^{l\trsp}dh_{t}^{l}(\xi)}-\eta\mathring{\theta}_{W^{l}}\sum_{s=0}^{t-1}\mathring{\chi}_{s}Z^{x_{s}^{l-1}}\EV Z^{dh_{s}^{l}}Z^{dh_{t}^{l}(\xi)}
\end{align*}
\item Loss derivative
\[
\mathring{\chi}_{t}=\loss'(\mathring{f}_{t},y_{0}).
\]
\end{enumerate}
The following fact follows from the results of \cite{TP2} (or can be verified by straightforward calculation) and will be useful for us.
\begin{prop}
\label{prop:Zdotdx0}$\dot{Z}^{dx_{0}^{l}(\xi)}=0$ and $Z^{dx_{0}^{l}(\xi)}=\hat{Z}^{dx_{0}^{l}(\xi)}$ for any $\xi\in\isp$.
\end{prop}

If the parametrization is pseudostable, then all the $\theta_{\bullet}$ converge to 0 or 1 so \cref{setup:netsortplus} is satisfied.
Therefore, the Master Theorem applies and says that, for any collection of vectors $v^{1},\ldots,v^{k}$ such that $Z^{v^{i}}$ is defined above, we have
\[
\frac{1}{n}\sum_{\alpha=1}^{n}\psi(v_{\alpha}^{1},\ldots,v_{\alpha}^{k})\asto\EV\psi(Z^{v^{1}},\ldots,Z^{v^{k}})
\]
for any pseudo-Lipschitz $\psi$. In addition,\footnote{Again, if $a_{L+1}+b_{L+1}=1/2$, remember we are conditioning on $f_{0}(\xi),\xi\in\isp$.}
\[
\del f_{t}(\xi)\asto\del\mathring{f}_{t}(\xi),\quad f_{t}(\xi)\asto\mathring{f}_{t}(\xi),\quad\chi_{t}\asto\mathring{\chi}_{t},\quad\forall\xi\in\isp,t\ge1.
\]
We now describe some immediate consequences of this.

\subsubsection{Some Immediate Results}
\begin{prop}
\label{lem:trivialPseudostableParam}A pseudostable parametrization is trivial if
\[2a_{L+1}+c>1\quad\text{and}\quad a_{L+1}+b_{L+1}+r>1.\]
\end{prop}
\begin{proof}
In this case, $\theta_{L+1}',\theta_{Lf}',\theta_{L,L+1}'\to0$, and $\del\mathring{f}_{t}(\xi)=0$ for all $t$ and $\xi\in\isp$ by \cref{eq:delflimit}.
\end{proof}
\begin{prop}
\label{prop:pseudostable=00003D>stable}A pseudostable parametrization is stable.
\end{prop}

\begin{proof}
For a pseudostable parametrization, all of $\theta$s converge to 1 or 0, and all of the $Z^{\del h_{t}^{l}(\xi)},Z^{\del x_{t}^{l}(\xi)}$ have well defined (finite) limits, which implies $\Delta h_{t}^{l}(\xi),\Delta x_{t}^{l}(\xi)=O_{*}(1),\forall l\in[L],\quad\text{and}\quad f_{t}(\xi)=O_{*}(1)$.
\end{proof}
\begin{prop}
\label{prop:r>0FixesFeatures}Consider a pseudostable parametrization. If $r>0$, then it fixes all (pre)features and all (pre)feature kernels. In addition, $\Delta W_{t}^{L+1}\Delta x_{t}^{L}(\xi)\asto0$.
\end{prop}

\begin{proof}
If $r>0$, then $\theta_{l}\to0$ for all $l\in[L]$, so that for all $z\in\{x^{l},h^{l}\}_{l}$, $\Delta z_{t}(\xi)=z_{t}(\xi)-z_{0}(\xi)=\theta_{z}\del z_{1}(\xi)+\cdots+\theta_{z}\del z_{t}(\xi)$ has $\|\Delta z_{t}(\xi)\|^{2}/n\asto0$ by \cref{eq:Zzt} and the Master Theorem, i.e. all features are fixed. Similarly, for any pair $\xi,\bar{\xi}\in\isp$, $z_{t}(\xi)^{\trsp}z_{t}(\bar{\xi})/n-z_{0}(\xi)^{\trsp}z_{0}(\bar{\xi})/n\asto0$, so all feature kernels are fixed. Finally, $r>0$ implies $\theta_{L,L+1}'\to0$, which means $\Delta W_{t}^{L+1}\Delta x_{t}^{L}(\xi)\asto0$ by the Master Theorem.
\end{proof}
\begin{prop}
\label{lem:r<0Unstable}An initialization-stable parametrization with $r<0$ is not stable.
\end{prop}

\begin{proof}
If $r<0$, then there is some $\ell\in[L]$ such that $\theta_{L}\ge\cdots\ge\theta_{\ell}>1\ge\theta_{\ell-1}\ge\cdots\ge\theta_{1}$. For $h=h^{\ell},x=x^{\ell-1},W=W^{\ell}$, we would have $\theta_{x/h}=\theta_{\ell-1}/\theta_{\ell}\to0$, $\theta_{W/h}=1$, and $\theta_{Wx/h}=\theta_{W/h}\theta_{\ell-1}\to0$. The Tensor Program up to the definition of $\del h_{1}(\xi_{0})$ satisfies the conditions of the Master Theorem. Therefore, $\|\del h_{1}(\xi_{0})\|^{2}/2\asto\EV(Z^{\del h_{1}(\xi_{0})})^{2}=\EV(\eta\mathring{\chi}_{t-1}Z^{dh_{0}}\EV Z^{x_{0}}Z^{x_{1}(\xi_{0})})^{2}$. If $\xi_{0}\ne0$, then $\EV(Z^{dh_{0}})^{2}>0$. If $\eta$ is in addition sufficiently small but nonzero, then $\EV Z^{x_{0}}Z^{x_{1}(\xi_{0})}\approx\EV(Z^{x_{0}})^{2}>0$. Therefore, under these conditions, and with a training sequence that has $\mathring{\chi}_{0}\ne0$, we have $\EV(\eta\mathring{\chi}_{t-1}Z^{dh_{0}}\EV Z^{x_{0}}Z^{x_{1}(\xi_{0})})^{2}>0$, so that $\del h_{1}(\xi_{0})=\Theta_{\xi_{0}}(1)$. However, $\Delta h_{1}(\xi_{0})=\theta_{h}\del h_{1}(\xi_{0})$ and $\theta_{h}=\theta_{\ell}\to\infty$. Hence $\Delta h_{1}(\xi_{0})\ne O_{\xi_{0}}(1)$, as desired.
\end{proof}

\subsection{\texorpdfstring{$r>0$}{r>0} Implies Kernel Regime}
\label{sec:kernelregime}

In this section, we analyze the case when $r>0$. Our main result is deriving the corresponding infinite-width kernel gradient descent dynamics (\cref{thm:kerneldynamics}).
Nothing here depends on $\phi$ being tanh or $\sigma$-gelu.

\subparagraph*{Preliminary Derivations}

If $r>0$, then $\mathring{\theta}_{l}=\mathring{\theta}_{W^{l}}=0$ for all $l\in[L]$, so that we have 
\[
Z^{h_{t}^{l}(\xi)}=Z^{h_{0}^{l}(\xi)},Z^{x_{t}^{l}(\xi)}=Z^{x_{0}^{l}(\xi)},Z^{dh_{t}^{l}(\xi)}=Z^{dh_{0}^{l}(\xi)},Z^{dx_{t}^{l}(\xi)}=Z^{dx_{0}^{l}(\xi)},Z^{\widehat{W}_{t}^{L+1}}=Z^{\widehat{W}_{0}^{L+1}}
\]
for all $t$ and $\xi\in\isp$. Let $\ell\in[L]$ be the unique $\ell$ such that $1=\theta_{L}/\theta_{L}=\cdots=\theta_{\ell}/\theta_{L}>\theta_{\ell-1}/\theta_{L}\ge\cdots\ge\theta_{1}/\theta_{L}$. Then for $l\ge\ell+1$ and shorthand $h=h^{l},x=x^{l-1},W=W^{l}$, we have $\mathring{\theta}_{x/h}=1$, $\mathring{\theta}_{Wx/h}=0$ and, by \cref{eq:Zdelh},
\begin{align}
Z^{\del h_{t}(\xi)} & =Z^{W_{0}\del x_{t}(\xi)}-\eta\mathring{\chi}_{t-1}\mathring{\theta}_{W/h}Z^{dh_{t-1}}\EV Z^{x_{t-1}}Z^{x_{t}(\xi)},\nonumber \\
 & =Z^{W_{0}\del x_{t}(\xi)}-\eta\mathring{\chi}_{t-1}\mathring{\theta}_{W/h}Z^{dh_{0}(\xi_{t-1})}\EV Z^{x_{0}(\xi_{t-1})}Z^{x_{0}(\xi)}\label{eq:ZdelhWhenr>0}
\end{align}
where $\mathring{\theta}_{W/h}$ can be either 0 or 1. For $l=\ell$, because $\theta_{h}=\theta_{l}=\max_{m\le l}\theta_{W^{m}}=\max(\theta_{W^{l}},\theta_{l-1})=\max(\theta_{W^{l}},\theta_{x})$ so $\mathring{\theta}_{x/h}=\mathring{\theta}_{Wx/h}=0$ and $\mathring{\theta}_{W/h}=1$, we also have
\begin{align}
Z^{\del h_{t}(\xi)} & =-\eta\mathring{\chi}_{t-1}Z^{dh_{t-1}}\EV Z^{x_{t-1}}Z^{x_{t}(\xi)}\nonumber \\
 & =-\eta\mathring{\chi}_{t-1}Z^{dh_{0}(\xi_{t-1})}\EV Z^{x_{0}(\xi_{t-1})}Z^{x_{0}(\xi)}.\label{eq:basecase}
\end{align}
Finally, for all $l\in[L]$, we have, by \cref{eq:ZdelxWhenr>0},
\[
Z^{\del x_{t}(\xi)}=\phi'(Z^{h_{t-1}(\xi)})Z^{\del h_{t}(\xi)}=\phi'(Z^{h_{0}(\xi)})Z^{\del h_{t}(\xi)}.
\]

\begin{defn}
For $1\le m\le l$ and $\xi,\zeta\in\isp$, define 
\[
\Sigma^{ml}(\xi,\zeta)\defeq\EV Z^{x_{0}^{m}(\xi)}Z^{x_{0}^{m}(\zeta)}\times\EV\phi'(Z^{h_{0}^{m+1}(\xi)})\phi'(Z^{h_{0}^{m+1}(\zeta)})\times\cdots\times\EV\phi'(Z^{h_{0}^{l}(\xi)})\phi'(Z^{h_{0}^{l}(\zeta)}).
\]
We also define
\[
\Sigma^{0l}(\xi,\zeta)\defeq\xi^{\trsp}\zeta\times\EV\phi'(Z^{h_{0}^{m+1}(\xi)})\phi'(Z^{h_{0}^{m+1}(\zeta)})\times\cdots\times\EV\phi'(Z^{h_{0}^{l}(\xi)})\phi'(Z^{h_{0}^{l}(\zeta)})
\]
\end{defn}

For example, 
\begin{align*}
    \Sigma^{ll}(\xi,\zeta)&=\EV Z^{x_{0}^{l}(\xi)}Z^{x_{0}^{l}(\zeta)}\\
    \Sigma^{l,l+1}(\xi,\zeta)&=\EV Z^{x_{0}^{l}(\xi)}Z^{x_{0}^{l}(\zeta)}\EV\phi'(Z^{h_{0}^{l+1}(\xi)})\phi'(Z^{h_{0}^{l+1}(\zeta)}),
\end{align*}
and so on.

\subparagraph*{Notation}

For brevity, below we will shorthand $\vartheta_{m}=\theta_{W^{m}/h^{m}}$. We write $Z^{x}\equiv Z^{y}\mod\hat{Z}^{W\bullet}$ if $Z^{x}-Z^{y}$ is a linear combination of $\hat{Z}^{Wu}$ for various vectors $u$.
\begin{lem}
\label{lemma:ZdelhmodZhat}For any input $\xi$, any $l\ge\ell$, at any time $t$,
\begin{equation}
Z^{\del h_{t}^{l}(\xi)}\equiv-\eta\mathring{\chi}_{t-1}Z^{dh_{0}^{l}(\xi_{t-1})}\sum_{m=\ell-1}^{l-1}\mathring{\vartheta}_{m+1}\Sigma^{m,l-1}(\xi_{t-1},\xi)\mod\hat{Z}^{W_{0}^{l}\bullet}.\label{eq:kernelIH}
\end{equation}
\end{lem}

\begin{proof}
We proceed by induction.

Base Case $l=\ell$: this is given by \cref{eq:basecase}.

Induction: Assume \cref{eq:kernelIH} holds for $l-1$, and we shall prove it for $l$.

To alleviate notation, we write $x=x_{t}^{l-1},\bar{x}=x_{t-1}^{l-1},x_{0}=x_{0}^{l-1},h=h_{t}^{l-1},\bar{h}=h_{t-1}^{l-1},h_{0}=h_{0}^{l-1},\bar{\xi}=\xi_{t-1},W=W_{0}^{l}$, i.e. we use $\bar{\bullet}$ to denote time $t-1$ in contrast to $\bullet$ for time $t$, and we suppress layer index. In contrast, we will write $h_{0}^{l},h_{t}^{l}$, and $\xi$ for their usual meanings.

First, note that $Z^{\del x(\xi)}=\phi'(Z^{\bar{h}(\xi)})Z^{\del h(\xi)}$ by \cref{eq:ZdelxWhenr>0}. Because $Z^{\bar{h}(\xi)}=Z^{h_{0}(\xi)}$, and, by induction hypothesis, $Z^{\del h(\xi)}$ is a scalar multiple of $Z^{dh_{0}(\bar{\xi})}=Z^{dx_{0}(\bar{\xi})}\phi'(Z^{h_{0}(\bar{\xi})})$, $Z^{\del x(\xi)}$ is symbolically solely a function of $Z^{h_{0}(\xi)},Z^{h_{0}(\bar{\xi})},Z^{dx_{0}(\bar{\xi})},$all of which are equal to their $\hat{Z}$ versions (with the last due to \cref{prop:Zdotdx0}). Among these, only $Z^{dx_{0}(\bar{\xi})}=Z^{W^{\trsp}dh_{0}^{l}(\bar{\xi})}$ is constructed from matrix multiplication with $W_{0}^{\trsp}$. Thus,
\begin{equation}
\dot{Z}^{W_{0}\del x(\xi)}=Z^{dh_{0}^{l}(\bar{\xi})}\EV\frac{\partial Z^{\del x(\xi)}}{\partial Z^{dx_{0}(\bar{\xi})}}=Z^{dh_{0}^{l}(\bar{\xi})}\EV\phi'(Z^{h_{0}(\xi)})\frac{\partial Z^{\del h(\xi)}}{\partial Z^{dx_{0}(\bar{\xi})}}.\label{eq:ZW0delxFirst}
\end{equation}
By induction hypothesis, 
\[
\frac{\partial Z^{\del h(\xi)}}{\partial Z^{dx_{0}(\bar{\xi})}}=-\eta\mathring{\chi}_{t-1}\phi'(Z^{h_{0}(\bar{\xi})})\sum_{m=\ell-1}^{l-2}\mathring{\vartheta}_{m+1}\Sigma^{m,l-2}(\bar{\xi},\xi).
\]
Therefore, 
\[
\EV\phi'(Z^{h_{0}(\xi)})\frac{\partial Z^{\del h(\xi)}}{\partial Z^{dx_{0}(\bar{\xi})}}=-\eta\mathring{\chi}_{t-1}\EV\left[\phi'(Z^{h_{0}(\xi)})\phi'(Z^{h_{0}(\bar{\xi})})\right]\sum_{m=\ell-1}^{l-2}\mathring{\vartheta}_{m+1}\Sigma^{m,l-2}(\bar{\xi},\xi).
\]
By definition of $\Sigma^{ml}$, this equals
\[
\EV\phi'(Z^{h_{0}(\xi)})\frac{\partial Z^{\del h(\xi)}}{\partial Z^{dx_{0}(\bar{\xi})}}=-\eta\mathring{\chi}_{t-1}\sum_{m=\ell-1}^{l-2}\mathring{\vartheta}_{m+1}\Sigma^{m,l-1}(\bar{\xi},\xi).
\]
Plugging this back into \cref{eq:ZW0delxFirst}, we get
\begin{equation}
\dot{Z}^{W_{0}\del x(\xi)}=-\eta\mathring{\chi}_{t-1}Z^{dh_{0}^{l}(\bar{\xi})}\sum_{m=\ell-1}^{l-2}\mathring{\vartheta}_{m+1}\Sigma^{m,l-1}(\bar{\xi},\xi).\label{eq:ZW0delx}
\end{equation}
Finally, by \cref{eq:ZdelhWhenr>0},
\begin{align*}
Z^{\del h_{t}^{l}(\xi)} & =\dot{Z}^{W_{0}\del x(\xi)}-\eta\mathring{\chi}_{t-1}\mathring{\vartheta}_{l}Z^{dh_{0}^{l}(\bar{\xi})}\EV Z^{x_{0}(\bar{\xi})}Z^{x_{0}(\xi)}\\
 & =\dot{Z}^{W_{0}\del x(\xi)}-\eta\mathring{\chi}_{t-1}\mathring{\vartheta}_{l}Z^{dh_{0}^{l}(\bar{\xi})}\Sigma^{l-1,l-1}(\bar{\xi},\xi).
\end{align*}
Together with \cref{eq:ZW0delx}, this completes the induction.
\end{proof}
\begin{lem}
\label{lem:delwdelx0}Assume pseudostability, $r>0$, and $a_{L+1}+b_{L+1}\le2a_{L+1}+c$. If $\mathring{\theta}_{L+1/f}=1$ then $\mathring{\theta}_{Lf}'=0$.
\end{lem}

\begin{proof}
$a_{L+1}+b_{L+1}\le2a_{L+1}+c$ iff $\theta_{L+1}\le\theta_{f}$. So $\mathring{\theta}_{L+1/f}=1$ implies $\theta_{L+1}=\theta_{f}$. By pseudostability, $n\theta_{L+1}\le1$. Since $\theta_{L}=n^{-r}$, we have $\theta_{Lf}'=n\cdot n^{-r}\cdot\theta_{f}=n^{-r}\cdot n\theta_{L+1}<0$ since $r>0$. Therefore $\mathring{\theta}_{Lf}'=0$.
\end{proof}
\begin{thm}
\label{thm:kerneldynamics}Consider a pseudostable parametrization. At any time $t$, for any input $\xi\in\isp$, we have
\[
\del\mathring{f}_{t}(\xi)=-\eta\mathring{\chi}_{t-1}\Sigma(\xi_{t-1},\xi),
\]
where the kernel $\Sigma$ is defined for any $\xi,\zeta\in\isp$ by
\[
\Sigma(\zeta,\xi)\defeq\mathring{\theta}_{L+1}'\Sigma^{LL}(\zeta,\xi)+\mathring{\theta}_{Lf}'\sum_{m=\ell-1}^{L-1}\mathring{\vartheta}_{m+1}\Sigma^{mL}(\zeta,\xi).
\]
\end{thm}

Observe that in the NTK parametrization, $\ell=1$, and $\mathring{\theta}_{L+1}'=\mathring{\theta}_{Lf}'=\mathring{\vartheta}_{m+1}=1$ for all $m$, so $\Sigma=\sum_{m=0}^{L}\Sigma^{mL}$ is precisely the NTK (for MLP without biases).
\begin{proof}
By \cref{eq:delflimit,eq:ZhatW},
\begin{align*}
\del\mathring{f}_{t}(\xi) & =\mathring{\theta}_{L+1}'\EV Z^{\del W_{t}^{L+1}}Z^{x_{t}^{L}(\xi)}+\mathring{\theta}_{Lf}'\EV Z^{\widehat{W}_{t-1}^{L+1}}Z^{\del x_{t}^{L}(\xi)}\\
Z^{\widehat{W}_{t}^{L+1}} & =Z^{\widehat{W}_{0}^{L+1}}+\mathring{\theta}_{L+1/f}(Z^{\del W_{1}^{L+1}}+\cdots+Z^{\del W_{t}^{L+1}}).
\end{align*}

Now by \cref{lem:delwdelx0}, either $\mathring{\theta}_{L+1/f}=0$ or $\mathring{\theta}_{Lf}'=0$. In both cases, $(Z^{\del W_{1}^{L+1}}+\cdots+Z^{\del W_{t}^{L+1}})$ contributes 0 to $\del\mathring{f}_{t}(\xi)$. So we can replace $Z^{\widehat{W}_{t-1}^{L+1}}$ with $Z^{\widehat{W}_{0}^{L+1}}$ above, and write
\[
\del\mathring{f}_{t}(\xi)=\mathring{\theta}_{L+1}'\EV Z^{\del W_{t}^{L+1}}Z^{x_{t}^{L}(\xi)}+\mathring{\theta}_{Lf}'\EV Z^{\widehat{W}_{0}^{L+1}}Z^{\del x_{t}^{L}(\xi)}.
\]
If \cref{eq:kernelIH} is true for $l=L$, then 
\[
\EV Z^{\widehat{W}_{0}^{L+1}}Z^{\del x_{t}^{L}(\xi)}=-\eta\mathring{\chi}_{t-1}\EV Z^{\widehat{W}_{0}^{L+1}}Z^{dh_{0}^{L}(\xi_{t-1})}\phi'(Z^{h_{0}^{L}(\xi)})\sum_{m=\ell-1}^{L-1}\mathring{\vartheta}_{m+1}\Sigma^{m,L-1}(\xi_{t-1},\xi)
\]
where the contributions from $\hat{Z}^{W_{0}^{L}\bullet}$ in $Z^{\del x_{t}^{L}(\xi)}$ vanish as they are independent from $Z^{\widehat{W}_{0}^{L+1}}$. Since $Z^{dh_{0}^{L}(\xi)}=Z^{\widehat{W}_{0}^{L+1}}\phi'(Z^{h_{0}^{L}(\xi)})$, we continue
\begin{align*}
\EV Z^{\widehat{W}_{0}^{L+1}}Z^{\del x_{t}^{L}(\xi)} & =-\eta\mathring{\chi}_{t-1}\EV\left(Z^{\widehat{W}_{0}^{L+1}}\right)^{2}\phi'(Z^{h_{0}^{L}(\xi_{t-1})})\phi'(Z^{h_{0}^{L}(\xi)})\sum_{m=\ell-1}^{L-1}\mathring{\vartheta}_{m+1}\Sigma^{m,L-1}(\xi_{t-1},\xi)\\
 & =-\eta\mathring{\chi}_{t-1}\sum_{m=\ell-1}^{L-1}\mathring{\vartheta}_{m+1}\Sigma^{mL}(\xi_{t-1},\xi).
\end{align*}
Similarly, by \cref{eq:ZdelWlast},
\begin{align*}
\EV Z^{\del W_{t}^{L+1}}Z^{x_{t}^{L}(\xi)}
    &=-\eta\mathring{\chi}_{t-1}\EV Z^{x_{t-1}^{L}(\xi_{t-1})}Z^{x_{t}^{L}(\xi)}\\
    &=
        -\eta\mathring{\chi}_{t-1}\EV Z^{x_{0}^{L}(\xi_{t-1})}Z^{x_{0}^{L}(\xi)}=-\eta\mathring{\chi}_{t-1}\Sigma^{LL}(\xi_{t-1},\xi).
\end{align*}
Altogether, these prove the desired claim.
\end{proof}
\begin{cor}
\label{cor:r>0Nontrivial}A pseudostable parametrization with $r>0$ is nontrivial iff $a_{L+1}+b_{L+1}+r=1$ or $2a_{L+1}+c=1$.
\end{cor}

\begin{proof}
The kernel $\Sigma$ in \cref{thm:kerneldynamics} is nonzero iff $\mathring{\theta}_{L+1}'$ or $\mathring{\theta}_{Lf}'$ is 1, which is equivalent to saying $a_{L+1}+b_{L+1}+r=1$ or $2a_{L+1}+c=1$.
\end{proof}
\begin{cor}
\label{cor:r>0Unstable}An initialization-stable parametrization with $r>0$ but $a_{L+1}+b_{L+1}+r<1$ or $2a_{L+1}+c<1$ is not stable.
\end{cor}

\begin{proof}
If $a_{L+1}+b_{L+1}+r<1$ or $2a_{L+1}+c<1$, then $\theta_{L+1}'\to\infty$ or $\theta_{Lf}'\to\infty$. Clearly, from the definition, $\Sigma^{mL}(\xi,\xi)>0$ for any $\xi\ne0$ and $m\in[0,L]$. All of our reasoning leading up to \cref{thm:kerneldynamics} applied at $t=1$ holds, so \cref{thm:kerneldynamics} (along with the Master Theorem) implies $|\del f_{t}(\xi)|\asto\infty$.
\end{proof}
\begin{cor}
\label{cor:nngplimit}If $a_{L+1}+b_{L+1}+r>1$ and $2a_{L+1}+c=1$, then for all $\xi\in\isp$, $\mathring f_t(\xi) \asto 0$ and $\del\mathring{f}_{t}(\xi)=-\eta\mathring{\chi}_{t-1}\Sigma^{LL}(\xi_{t-1},\xi)$, i.e. we have the Neural Network-Gaussian Process (NNGP) limit.
\end{cor}

Conventionally, the NNGP limit is associated with only training the last layer and nothing else. This result says that the same limit can be achieved if we train the body of the network slightly, so that $\Delta x_{t}^{L}$ does not interact with $W_{0}^{L+1}$ enough (embodied in the inequality $a_{L+1}+b_{L+1}+r>1$) to cause changes in $f_{t}$.
\begin{proof}
The premise implies $\mathring{\theta}_{L+1}'=1$ and $\mathring{\theta}_{Lf}'=0$, and the rest follows from \cref{thm:kerneldynamics}.
\end{proof}
\begin{rem}
We have assumed for simplicity of the proof that $a_{L+1}+b_{L+1}\le2a_{L+1}+c$. If this is not the case, then we can easily see \cref{cor:nngplimit} applies anyway.
\end{rem}

\subsection{\texorpdfstring{$r=0$}{r=0} Implies Feature Learning}
\label{sec:featurelearning}

In this section, we assume $r=0$ and show any such pseudostable parametrization 1) admits (pre)feature learning and (pre)feature kernel evolution, and 2) is \emph{not} in kernel regime (\cref{cor:r=0LearnsFeatures}). The overarching logic goes like this. 
\begin{enumerate}
\item The Master Theorem shows that the specific entry $\frac{1}{n}\|x_{1}^{L}(\xi_{0})\|^{2}$ of the feature kernel converges to $\EV(Z^{x_{1}^{L}(\xi_{0})})^{2}$. If the learning rate $\eta=0$, then $x_{1}^{L}(\xi_{0})=x_{0}^{L}$ and $\EV(Z^{x_{1}^{L}(\xi_{0})})^{2}=\EV(Z^{x_{0}^{L}})^{2}$. We hope to say that as $\eta$ increases, $\EV(Z^{x_{1}^{L}(\xi_{0})})^{2}$ moves away from $\EV(Z^{x_{0}^{L}})^{2}$, which would imply feature kernel evolution in layer $L$. To do so, we compute $\partial_{\eta}^{2}\EV(Z^{x_{1}^{L}(\xi_{0})})^{2}$ evaluated at $\eta=0$ and show it is nonzero (it turns out $\partial_{\eta}$ vanishes, so the next best thing is $\partial_{\eta}^{2}$). This then also implies feature learning in layer $L$. Analogous results for prefeatures and for other layers can be derived similarly.
\item If the parametrization is in the kernel regime with kernel $K$, the first step of SGD in the large width limit would look like $\mathring{f}_{1}(\xi)-\mathring{f}_{0}(\xi)=-\eta\mathring{\chi}_{0}K(\xi,\xi_{0})$; in particular, $\mathring{f}_{1}(\xi)-\mathring{f}_{0}(\xi)$ is linear in $\eta$. To show that a pseudostable parametrization with $r=0$ is not in the kernel regime, we will show $\partial_{\eta}^{3}(\mathring{f}_{1}(\xi)-\mathring{f}_{0}(\xi))=\partial_{\eta}^{3}\mathring{f}_{1}(\xi)$ is nonzero. (It turns out $\partial_{\eta}^{2}$ vanishes, so the next best thing is $\partial_{\eta}^{3}$).
\end{enumerate}
To calculate these $\eta$ derivatives, we will derive recurrence relations involving quantities defined below (see \cref{lem:dLambdaGammaRec} and \cref{thm:d2LambdaGammaRec}).

\subparagraph*{Setup and Notation}

First, write 
\[
Z_{t}^{l}\defeq Z^{h_{t}^{l}(\xi_{0})},\hat{Z}_{t}^{l}\defeq\hat{Z}^{W^{l}x_{t}^{l-1}(\xi_{0})},\tilde{Z}_{0}^{l}\defeq Z^{dx_{0}^{l}}.
\]
Note that $\tilde{Z}_{0}^{l}$ is a centered Gaussian independent from $\hat{Z}_{0}^{l}=Z_{0}^{l}$. Then we define
\[
\gamma^{l}(\eta)\defeq\EV\phi(Z_{0}^{l})\phi(Z_{1}^{l}),\quad\gamma_{11}^{l}(\eta)\defeq\EV\phi'(Z_{0}^{l})\phi'(Z_{1}^{l}),\quad\gamma_{02}^{l}(\eta)\defeq\EV\phi(Z_{0}^{l})\phi''(Z_{1}^{l})
\]
\[
\gamma_{20}^{l}(\eta)\defeq\EV\phi''(Z_{0}^{l})\phi(Z_{1}^{l}),\quad\lambda^{l}(\eta)\defeq\EV\phi(Z_{1}^{l})^{2}
\]
where the dependence on $\eta$ is from $Z_{1}^{l}$. Naturally, since $\phi$ and $\phi'$ are not almost everywhere zero, we have $\gamma^{l}(0),\lambda^{l}(0),\gamma_{11}^{l}(0)>0$. Note at $\eta=0$, we have $Z_{1}^{l}=Z_{0}^{l}$, so $\gamma^{l}(0)=\lambda^{l}(0)=\EV\phi(Z_{0}^{l})^{2}$. Observe that $(\hat{Z}_{1}^{l},\hat{Z}_{0}^{l})$ is jointly Gaussian with mean zero and covariance 
\begin{equation}
\Gamma^{l}(\eta)\defeq\begin{pmatrix}\lambda^{l}(\eta) & \gamma^{l}(\eta)\\
\gamma^{l}(\eta) & \lambda^{l}(0)
\end{pmatrix}.\label{eq:Gammal}
\end{equation}
WLOG, for simplicity of notation, we assume we choose a training routine such that $\mathring{\chi}_{0}=1$. We assume $\xi_{0}\ne0$. 

Since $r=0$, WLOG we can suppose for some $\ell\in[L]$, we have $\theta_{L}=\cdots=\theta_{\ell}=1>\theta_{\ell-1}\ge\cdots\ge\theta_{1}$.
\begin{lem}
\label{lem:Z1decomp}With the setup above, we have
\[
Z_{0}^{\ell-1}=Z_{1}^{\ell-1},\ldots,Z_{0}^{1}=Z_{1}^{1},
\]
and
\[
Z_{1}^{l}=\hat{Z}_{1}^{l}+\eta\beta^{l}\tilde{Z}_{0}^{l}\phi'(Z_{0}^{l}),\quad\forall l\in[\ell,L],
\]
where $\beta^{l}$ is defined recursively by
\begin{align*}
\beta^{l}=\beta^{l}(\eta) & \defeq-\gamma^{l-1}(\eta)+\beta^{l-1}(\eta)\gamma_{11}^{l-1}(\eta)\\
\beta^{\ell-1}(\eta) & \defeq0.
\end{align*}
Additionally, $\beta^{l}(0)<0$ for all $l\ge\ell$.
\end{lem}

\begin{proof}
Straightforward calculation using Moment and Zdot. Here, $-\gamma^{l-1}(\eta)$ comes from $\Delta W_{1}^{l}x_{1}^{1}(\xi_{0})$ and $\beta^{l-1}(\eta)\gamma_{11}^{l-1}(\eta)$ comes from $\dot{Z}^{h_{1}^{l}(\xi_{0})}$. Since $\gamma^{l}(0),\gamma_{11}^{l-1}(0)>0$ for all $l$, the recurrence on $\beta^{l}$ implies that $\beta^{l}(0)<0$ for all $l\ge\ell$.
\end{proof}

\subsubsection{Deriving Recurrence Relations on \texorpdfstring{$\partial_{\eta}\lambda^{l},\partial_{\eta}\gamma^{l},\partial_{\eta}^{2}\lambda^{l},\partial_{\eta}^{2}\gamma^{l}$}{Partial Derivatives of lambda and gamma}}

Below, we derive the recurrence relations required for our main result. They depend on the following constants.
\[
\kappa_{1}^{l}\defeq\EV\left[(\phi^{2})''(Z_{0}^{l})\right],\quad\kappa_{2}^{l}\defeq\EV\left[(\phi^{2})''(Z_{0}^{l})\phi'(Z_{0}^{l})^{2}\right],\quad\kappa_{3}^{l}\defeq\EV\left[\phi(Z_{0}^{l})\phi''(Z_{0}^{l})\phi'(Z_{0}^{l})^{2}\right].
\]

\begin{lem}
\label{lem:dLambdaGammaRec}With the setup above, we have, for all $l\in[L]$,
\begin{align}
\partial_{\eta}\lambda^{l}(0) & =\frac{1}{2}\kappa_{1}^{l}\partial_{\eta}\lambda^{l-1}(0)\label{eq:dlambda0}\\
\partial_{\eta}\gamma^{l}(0) & =\frac{1}{2}\gamma_{02}^{l}\partial_{\eta}\lambda^{l-1}(0)+\gamma_{11}^{l}\partial_{\eta}\gamma^{l-1}(0).\nonumber 
\end{align}
\end{lem}

\begin{proof}
We first derive the recurrence on $\partial_{\eta}\lambda^{l}$. By \cref{lem:dExpect} below, we have
\[
\partial_{\eta}\lambda^{l}=2\EV\phi(Z_{1}^{l})\partial_{\eta}\phi(Z_{1}^{l})+\frac{1}{2}\EV(\phi^{2})''(Z_{1}^{l})\partial_{\eta}\lambda^{l-1}.
\]
Note the second item in the sum evaluated at $\eta=0$ is exactly the RHS of \cref{eq:dlambda0}.
For the first item, since
\begin{align}
\partial_{\eta}\phi(Z_{1}^{l}) & =\phi'(Z_{1}^{l})(\beta^{l}\tilde{Z}_{0}^{l}\phi'(Z_{0}^{l})+\eta\tilde{Z}_{0}^{l}\phi'(Z_{0}^{l})\partial_{\eta}\beta^{l}),\label{eq:dphi(Z1)}
\end{align}
we compute
\begin{align*}
\EV\phi(Z_{1}^{l})\partial_{\eta}\phi(Z_{1}^{l}) & =\EV\phi(Z_{1}^{l})\phi'(Z_{1}^{l})(\beta^{l}\tilde{Z}_{0}^{l}\phi'(Z_{0}^{l})+\eta\tilde{Z}_{0}^{l}\phi'(Z_{0}^{l})\partial_{\eta}\beta^{l}).
\end{align*}
At $\eta = 0$, $Z^l_1 = Z^l_0$, so that $\tilde{Z}_{0}^{l}$ is independent from everything else in the expectation.
This implies the expectation is 0 and yields the recurrence for $\partial_{\eta}\lambda^{l}(0)$.

For $\partial_{\eta}\gamma^{l}$, let $\Sigma=\Sigma(\eta)\defeq\begin{pmatrix}\gamma_{02}^{l} & \gamma_{11}^{l}\\
\gamma_{11}^{l} & \gamma_{20}^{l}
\end{pmatrix}$. With $\Gamma^{l-1}$ as in \cref{eq:Gammal}, we have
\[
\partial_{\eta}\gamma^{l}=\EV\phi(Z_{0}^{l})\partial_{\eta}\phi(Z_{1}^{l})+\frac{1}{2}\langle\Sigma,\partial_{\eta}\Gamma^{l-1}\rangle
\]
By same reasoning as in \cref{eq:dlambda0}, the first term of this sum is zero when evaluated at $\eta=0$.
Since $\partial_{\eta}\Gamma^{l-1}(\eta)\defeq\begin{pmatrix}\partial_{\eta}\lambda^{l-1}(\eta) & \partial_{\eta}\gamma^{l-1}(\eta)\\
\partial_{\eta}\gamma^{l}(\eta) & 0
\end{pmatrix}$, we have
\[
\partial_{\eta}\gamma^{l}=\frac{1}{2}\langle\Sigma,\partial_{\eta}\Gamma^{l-1}\rangle=\frac{1}{2}\gamma_{02}^{l}\partial_{\eta}\lambda^{l-1}+\gamma_{11}^{l}\partial_{\eta}\gamma^{l-1}.
\]
\end{proof}
\begin{lem}
\label{lem:dExpect}Consider a twice continuously differentiable $f$ and Gaussian vector $Z\sim\Gaus(0,\Sigma)$ such that $f$ and $\Sigma$ both depend on a parameter $\eta$. Then
\[
\partial_{\eta}\EV f(Z)=\EV\partial_{\eta}f(Z)+\frac{1}{2}\langle\EV\nabla^{2}f(z),\partial_{\eta}\Sigma\rangle,
\]
where $\nabla^{2}$ denotes Hessian wrt $z$, and $\langle,\rangle$ denotes trace inner product of matrices.
\end{lem}

\begin{proof}
Let $p(z)$ denote the PDF of $Z$. We have
\[
\partial_{\eta}\EV f(Z)=\partial_{\eta}\int f(z)p(z)\dd z=\int\partial_{\eta}f(z)p(z)\dd z+\int f(z)\partial_{\eta}p(z)\dd z
\]
The first integral is $\EV\partial_{\eta}f(Z)$. The second integral can be rewritten using integration-by-parts as $\langle\EV\nabla^{2}f(z),\partial_{\eta}\Sigma\rangle$. (e.g. see Lemma F.18 of \citep{yang_mean_2019})
\end{proof}
We then easily have
\begin{thm}
\label{thm:dLambdaGamma0=00003D0}For all $l\in[L]$, 
\[
\partial_{\eta}\gamma^{l}(0)=\partial_{\eta}\lambda^{l}(0)=0.
\]
\end{thm}

\begin{proof}
For $l<\ell$, we obviously have $\partial_{\eta}\gamma^{l}(\eta)=\partial_{\eta}\lambda^{l}(0)=0$ for all $\eta$. Then this follows from \cref{lem:dLambdaGammaRec} and a simple induction.
\end{proof}
Unfortunately, this means that the first $\eta$ derivative doesn't give us what we need. So we try the second derivative, which will turn out to work.
\begin{thm}
\label{thm:d2LambdaGammaRec}For all $l<\ell,$$\partial_{\eta}^{2}\lambda^{l}(0)=\partial_{\eta}^{2}\gamma^{l}(0)=0$, and for all $l\ge\ell$,
\begin{align*}
\partial_{\eta}^{2}\lambda^{l}(0) & =C\kappa_{2}^{l}+\frac{1}{2}\kappa_{1}^{l}\partial_{\eta}^{2}\lambda^{l-1}(0)\\
\partial_{\eta}^{2}\gamma^{l}(0) & =C\kappa_{3}^{l}+\frac{1}{2}\gamma_{02}^{l}(0)\partial_{\eta}^{2}\lambda^{l-1}(0)+\gamma_{11}^{l}(0)\partial_{\eta}^{2}\gamma^{l-1}(0),
\end{align*}
where $C=2(\beta^{l}(0))^{2}\EV(\tilde{Z}_{0}^{l})^{2}>0.$
\end{thm}

\begin{proof}
We start with the $\partial_{\eta}^{2}\lambda^{l}(0)$ recurrence. For $l\ge\ell$, $\partial_{\eta}^{2}\lambda^{l}$ is a sum of 3 terms, representing 1) 2 derivatives in the integrand, 2) 2 derivatives in the Gaussian variance, and 3) 1 derivative each. When evaluated at $\eta=0$, only the first two terms survive because $\partial_{\eta}\lambda^{l-1}(0)=0$ by \cref{thm:dLambdaGamma0=00003D0}:
\[
\partial_{\eta}^{2}\lambda^{l}(0)=\EV\partial_{\eta}^{2}\phi^{2}(Z_{1}^{l})|_{\eta=0}+\frac{1}{2}\EV(\phi^{2})''(Z_{0}^{l})\partial_{\eta}^{2}\lambda^{l-1}(0).
\]
Now
\begin{align*}
\EV\partial_{\eta}^{2}\phi^{2}(Z_{1}^{l}) & =2\partial_{\eta}(\EV\phi(Z_{1}^{l})\phi'(Z_{1}^{l})(\beta^{l}\tilde{Z}_{0}^{l}\phi'(Z_{0}^{l})+\eta\tilde{Z}_{0}^{l}\phi'(Z_{0}^{l})\partial_{\eta}\beta^{l}))\\
 & =2\EV(\phi^{2})''(Z_{1}^{l})(\beta^{l}\tilde{Z}_{0}^{l}\phi'(Z_{0}^{l})+\eta\tilde{Z}_{0}^{l}\phi'(Z_{0}^{l})\partial_{\eta}\beta^{l})^{2}+\cdots
\end{align*}
where other terms appear in this sum but they vanish at $\eta=0$ because $\tilde{Z}_{0}^{l}$ appears unpaired in the expectation.
Thus,
\[
\EV\partial_{\eta}^{2}\phi^{2}(Z_{1}^{l})|_{\eta=0}=2(\beta^{l}(0))^{2}\EV(\tilde{Z}_{0}^{l})^{2}\EV(\phi^{2})''(Z_{0}^{l})\phi'(Z_{0}^{l})^{2}.
\]
Plugging this back in, we get the recurrence on $\partial_{\eta}^{2}\lambda^{l}(0)$.

The $\partial_{\eta}^{2}\gamma^{l}(0)$ recurrence is derived similarly.
\end{proof}
The following result will be useful for showing $\partial_{\eta}^{3}\mathring{f}_{1}(\xi_{0})\ne0$.
\begin{thm}
\label{thm:d2gamma11}Define 
\[
\dot{\kappa}_{3}^{l}\defeq\EV\left[\phi'''(Z_{0}^{l})\phi'(Z_{0}^{l})^{3}\right],\quad\gamma_{13}^{l}\defeq\EV\phi'(Z_{0}^{l})\phi'''(Z_{0}^{l}),\quad\gamma_{22}^{l}\defeq\EV\phi''(Z_{0}^{l})^{2}.
\]
Then for all $l\ge\ell$,
\begin{align*}
\partial_{\eta}^{2}\gamma_{11}^{l}(0) & =C\dot{\kappa}_{3}^{l}+\frac{1}{2}\gamma_{13}^{l}\partial_{\eta}^{2}\lambda^{l-1}(0)+\gamma_{22}^{l}\partial_{\eta}^{2}\gamma^{l-1}(0),
\end{align*}
where $C=2(\beta^{l}(0))^{2}\EV(\tilde{Z}_{0}^{l})^{2}>0.$
\end{thm}

\begin{proof}
Similar to the proof of \cref{thm:d2LambdaGammaRec}.
\end{proof}
The following result will be useful for showing prefeature kernel evolution.
\begin{thm}
\label{thm:prefeaturekernelmoves}For all $l\ge\ell$,
\begin{align*}
\partial_{\eta}^{2}\EV(Z_{1}^{l})^{2}|_{\eta=0} & =2C+\gamma_{11}^{l}(0)\partial_{\eta}^{2}\lambda^{l-1}(0),
\end{align*}
where $C=2(\beta^{l}(0))^{2}\EV(\tilde{Z}_{0}^{l})^{2}>0.$
\end{thm}

\begin{proof}
Similar to the proof of \cref{thm:d2LambdaGammaRec}.
\end{proof}

\subsubsection{Applications to \texorpdfstring{$\sigma$}{sigma}-Gelu}

The following proposition regarding $\sigma$-gelu is easy to verify.
\begin{prop}
\label{prop:geluprop}Let $\phi$ be $\sigma$-gelu. For any centered Gaussian $Z\in\R$ with nonzero variance,
\[
\EV(\phi^{2})''(Z),\EV(\phi^{2})''(Z)\phi'(Z)^{2},\EV\phi(Z)\phi''(Z)\phi'(Z)^{2},\EV\phi(Z)\phi''(Z),\EV\phi''(Z)^{2}>0,
\]
and they converge to $0$ as $\sigma\to0$. Also,
\[
\EV\phi'''(Z)\phi'(Z)^{3},\EV\phi'(Z)\phi'''(Z)<0,
\]
and they converge to $-\infty$ as $\sigma\to0$.
\end{prop}

This particularly implies that $\kappa_{1}^{l},\kappa_{2}^{l},\kappa_{3}^{l},\gamma_{02}^{l}(0),\gamma_{22}^{l}>0$ and converges to 0 with small $\sigma$, but $\dot{\kappa}_{3}^{l},\gamma_{13}^{l}<0$ and diverges to $-\infty$ with small $\sigma$.

\begin{thm}
\label{thm:geluD2GammaLambda>0}Consider a pseudostable parametrization with $r=0.$ If $\phi$ is $\sigma$-gelu, then for all $l\ge\ell$,
\[
\partial_{\eta}^{2}\gamma^{l}(0),\partial_{\eta}^{2}\lambda^{l}(0)>0
\]
and they converge to 0 as $\sigma\to0$.
\end{thm}

\begin{proof}
We always have $(\beta^{l}(0))^{2},\EV(\tilde{Z}_{0}^{l})^{2}>0$. By \cref{prop:geluprop}, $\kappa_{1}^{l},\kappa_{2}^{l}>0$ as well. Thus, by \cref{thm:d2LambdaGammaRec}, $\partial_{\eta}^{2}\lambda^{l}(0)>0$ for all $l\ge\ell$. By \cref{prop:geluprop}, $\kappa_{3}^{l},\gamma_{02}^{l}(0)>0$, so by \cref{thm:d2LambdaGammaRec}, $\partial_{\eta}^{2}\gamma^{l}(0)>0$ for all $l\ge\ell$ as well. As $\sigma\to0$, $\kappa_{1}^{l},\kappa_{2}^{l},\kappa_{3}^{l},\gamma_{02}^{l}(0)\to0$, so $\partial_{\eta}^{2}\lambda^{l}(0),\partial_{\eta}^{2}\gamma^{L}(0)\to0$.
\end{proof}
\begin{thm}
\label{thm:geluFunctionMoves}Consider a pseudostable parametrization with $r=0.$ Suppose $a_{L+1}+b_{L+1}+r=1$ or $2a_{L+1}+c=1$. If $\phi$ is $\sigma$-gelu for sufficiently small $\sigma$, then 
\[
\partial_{\eta}^{3}\mathring{f}_{1}(\xi_{0})\ne0.
\]
\end{thm}

\begin{proof}
We have $\mathring{f}_{1}(\xi_{0})=\mathring{\theta}_{L+1}'\EV Z^{\del W_{1}^{L+1}}Z^{x_{1}^{L}(\xi_{0})}+\mathring{\theta}_{Lf}'\EV Z^{\widehat{W}_{0}^{L+1}}Z^{\del x_{1}^{L}(\xi_{0})}$, where at least one of $\mathring{\theta}_{Lf}'$ and $\mathring{\theta}_{L+1}'$ is 1 because $a_{L+1}+b_{L+1}+r=1$ or $2a_{L+1}+c=1$. We have
\begin{align*}
\EV Z^{\del W_{1}^{L+1}}Z^{x_{1}^{L}(\xi_{0})} & =-\eta\EV Z^{x_{0}^{L}}Z^{x_{1}^{L}(\xi_{0})}\\
\EV Z^{\widehat{W}_{0}^{L+1}}Z^{x_{1}^{L}(\xi_{0})} & =\EV Z^{\widehat{W}_{0}^{L+1}}\phi(Z^{h_{0}^{L}}-\eta Z^{\widehat{W}_{0}^{L+1}}\phi'(Z^{h_{0}^{L}})\EV Z^{x_{0}^{L-1}}Z^{x_{1}^{L-1}(\xi_{0})})\\
 & =-\eta\EV\phi'(Z^{h_{1}^{L}(\xi_{0})})\phi'(Z^{h_{0}^{L}})\EV Z^{x_{0}^{L-1}}Z^{x_{1}^{L-1}(\xi_{0})}
\end{align*}
where we used Stein's Lemma for the last equality. Thus
\[
\partial_{\eta}^{3}\mathring{f}_{1}(\xi_{0})=-\left(\mathring{\theta}_{L+1}'\partial_{\eta}^{2}\gamma^{L}(0)+\mathring{\theta}_{Lf}'\partial_{\eta}^{2}(\gamma_{11}^{L}\gamma^{L-1})(0)\right).
\]

Below we will show that for small $\sigma$, $\partial_{\eta}^{2}\gamma^{L}(0)$ is small and positive and $\partial_{\eta}^{2}(\gamma_{11}^{L}\gamma^{L-1})(0)$ is large and negative, so $\partial_{\eta}^{3}\mathring{f}_{1}(\xi_{0})$ cannot be 0 no matter the values of $\mathring{\theta}_{L+1}'$ and $\mathring{\theta}_{Lf}'$.

Claim: For sufficiently small $\sigma$, $\partial_{\eta}^{2}\gamma_{11}^{L}(0)<0$. It converges to $-\infty$ as $\sigma\to0$.

Proof: By \cref{thm:d2gamma11}, $\partial_{\eta}^{2}\gamma_{11}^{l}(0)=C\dot{\kappa}_{3}^{l}+\frac{1}{2}\gamma_{13}^{l}\partial_{\eta}^{2}\lambda^{l-1}(0)+\gamma_{22}^{l}\partial_{\eta}^{2}\gamma^{l-1}(0).$ Note $\partial_{\eta}^{2}\lambda^{l-1}(0)\ge0$ by \cref{thm:geluD2GammaLambda>0}. Also, by \cref{prop:geluprop}, $\dot{\kappa}_{3}^{l},\gamma_{13}^{l}<0,\gamma_{22}^{l}>0$, and as $\sigma\to0$, $\dot{\kappa}_{3}^{l},\gamma_{13}^{l}\to-\infty,\gamma_{22}^{l}\to0$ (as well as $\partial_{\eta}^{2}\gamma^{L-1}(0),\partial_{\eta}^{2}\lambda^{l}(0)\to0$ by \cref{thm:geluD2GammaLambda>0}). One can see that $C$ converges to a positive constant as $\sigma\to0$ as well. Therefore, for small enough $\sigma$, $\partial_{\eta}^{2}\gamma_{11}^{l}(0)<0$, and as $\sigma\to0$, $\partial_{\eta}^{2}\gamma_{11}^{L}(0)\to-\infty$.

Claim: For sufficiently small $\sigma$, $\partial_{\eta}^{2}(\gamma_{11}^{L}\gamma^{L-1})(0)<0$. It converges to $-\infty$ as $\sigma\to0$.

Proof: Observe $\partial_{\eta}^{2}(\gamma_{11}^{L}\gamma^{L-1})(0)=\partial_{\eta}^{2}\gamma_{11}^{L}(0)\gamma^{L-1}(0)+\gamma_{11}^{L}(0)\partial_{\eta}^{2}\gamma^{L-1}(0)$ because $\partial_{\eta}\gamma^{L-1}(0)=0$ by \cref{thm:dLambdaGamma0=00003D0}. So the above claim and \cref{thm:geluD2GammaLambda>0} yield the desired results.

Finishing the main proof: Therefore, if $\mathring{\theta}_{L+1}'=1$ but $\mathring{\theta}_{Lf}'=0$, then $-\partial_{\eta}^{3}\mathring{f}_{1}(\xi_{0})>0$ because $\partial_{\eta}^{2}\gamma^{L}(0)>0$; if $\mathring{\theta}_{L+1}'=0$ but $\mathring{\theta}_{Lf}'=1$, then $-\partial_{\eta}^{3}\mathring{f}_{1}(\xi_{0})<0$ for small $\sigma$ because $\partial_{\eta}^{2}(\gamma_{11}^{L}\gamma^{L-1})(0)<0$; if $\mathring{\theta}_{L+1}'=\mathring{\theta}_{Lf}'=1$, then $-\partial_{\eta}^{3}\mathring{f}_{1}(\xi_{0})<0$ for small $\sigma$ because $\partial_{\eta}^{2}(\gamma_{11}^{L}\gamma^{L-1})(0)\to-\infty$ while $\partial_{\eta}^{2}\gamma^{L}(0)\to0$ as $\sigma\to0$.
\end{proof}

\subsubsection{Applications to Tanh}

The following property of $\tanh$ is easy to verify.
\begin{prop}
\label{prop:tanhprop}Let $\phi=\tanh$. For any centered Gaussian $Z\in\R$ with nonzero variance,
\[
\EV(\phi^{2})''(Z),\EV(\phi^{2})''(Z)\phi'(Z)^{2},\EV\phi''(Z)^{2}>0,
\]
and
\[
\EV\phi(Z)\phi''(Z)\phi'(Z)^{2},\EV\phi(Z)\phi''(Z),\EV\phi'''(Z)\phi'(Z)^{3},\EV\phi'(Z)\phi'''(Z)<0.
\]
In particular, this means
\[
\kappa_{1}^{l},\kappa_{2}^{l},\gamma_{22}^{l}>0,\quad\kappa_{3}^{l},\gamma_{02}^{l}(0),\dot{\kappa}_{3}^{l},\gamma_{13}^{l}<0.
\]
\end{prop}

\begin{thm}
\label{thm:tanhD2GammaLambda}Consider a pseudostable parametrization with $r=0.$ If $\phi$ is $\tanh$, then for all $l\ge\ell$,
\[
\partial_{\eta}^{2}\gamma^{l}(0)<0,\quad\partial_{\eta}^{2}\lambda^{l}(0)>0.
\]
\end{thm}

\begin{proof}
Similar to the proof of \cref{thm:geluD2GammaLambda>0}, except that here $\kappa_{3}^{l},\gamma_{02}^{l}(0)<0$, making $\partial_{\eta}^{2}\gamma^{l}(0)<0$.
\end{proof}
\begin{thm}
\label{thm:tanhfunctionmoves}Consider a pseudostable parametrization with $r=0.$ Suppose $a_{L+1}+b_{L+1}+r=1$ or $2a_{L+1}+c=1$. If $\phi$ is $\tanh$, then 
\[
\partial_{\eta}^{3}\mathring{f}_{1}(\xi_{0})\ne0.
\]
\end{thm}

\begin{proof}
Similar to the proof of \cref{thm:geluFunctionMoves}, except in the expression
\[
\partial_{\eta}^{3}\mathring{f}_{1}(\xi_{0})=-\left(\mathring{\theta}_{L+1}'\partial_{\eta}^{2}\gamma^{L}(0)+\mathring{\theta}_{Lf}'\partial_{\eta}^{2}(\gamma_{11}^{L}\gamma^{L-1})(0)\right),
\]
$\partial_{\eta}^{2}\gamma^{L}(0)$ and $\partial_{\eta}^{2}(\gamma_{11}^{L}\gamma^{L-1})(0)$ are both negative. The former is because of \cref{thm:tanhD2GammaLambda}. The latter is because $\partial_{\eta}^{2}\gamma^{L-1}(0)\le0$ for the same reason, and $\partial_{\eta}^{2}\gamma_{11}^{L}(0)<0$ since $\dot{\kappa}_{3}^{l},\gamma_{13}^{l}<0,\gamma_{22}^{l}>0$ by \cref{prop:tanhprop}.
\end{proof}

\subsubsection{Main Results}
\begin{prop}
\label{prop:r=00003D0Nontrivial}Suppose $\phi$ is tanh or $\sigma$-gelu for sufficiently small $\sigma$. A pseudostable parametrization with $r=0$ is nontrivial iff $a_{L+1}+b_{L+1}=1$ or $2a_{L+1}+c=1$.
\end{prop}

\begin{proof}
If $a_{L+1}+b_{L+1}+r=1$ or $2a_{L+1}+c=1$, then \cref{thm:geluFunctionMoves} and \cref{thm:tanhfunctionmoves} show that the parametrization is nontrivial. Otherwise, it is trivial by \cref{lem:trivialPseudostableParam}.
\end{proof}
\begin{thm}
\label{cor:r=0LearnsFeatures}Suppose $\phi$ is tanh or $\sigma$-gelu for sufficiently small $\sigma$. For any nontrivial pseudostable parametrization with $r=0$, the following are true of the parametrization:
\begin{enumerate}
\item not in kernel regime
\item feature learning
\item feature learning in the $L$th layer
\item feature kernels evolution
\item feature kernel evolution in the $L$th layer
\item prefeature learning
\item prefeature learning in the $L$th layer
\item prefeature kernels evolution
\item prefeature kernel evolution in the $L$th layer
\item if there is feature learning \emph{or} feature kernel evolution \emph{or} prefeature learning \emph{or} prefeature kernel evolution in layer $l$, then there is feature learning \emph{and} feature kernel evolution \emph{and} prefeature learning \emph{and} prefeature kernel evolution in layers $l,\ldots,L$.
\end{enumerate}
\end{thm}

\begin{proof}
The parametrization cannot be in kernel regime since $\partial_{\eta}^{3}\mathring{f}_{1}(\xi_{0})\ne0$ by \cref{thm:tanhfunctionmoves} or \cref{thm:geluFunctionMoves}. By \cref{thm:geluD2GammaLambda>0} or \cref{thm:tanhD2GammaLambda}, $\partial_{\eta}^{2}\lambda^{l}(0)>0$ for all $l\ge\ell$, so the feature kernel evolves in layer $\ell,\ldots,L$, for some normalized learning rate $\eta>0$. This implies feature learning in layer $\ell,\ldots,L$, since $Z^{x_{1}^{L}(\xi_{0})}-Z^{x_{0}^{L}}\ne0$ in this case. This then implies $Z^{h_{1}^{L}(\xi_{0})}-Z^{h_{0}^{L}}\ne0$, so we have prefeature learning in layer $\ell,\ldots,L$. Prefeature kernel evolution in layer $\ell,\ldots,L$ is implied by \cref{thm:prefeaturekernelmoves}. Finally, the last statement follows clearly from our logic above.
\end{proof}
\begin{cor}
\label{cor:r=00003D0Unstable}Suppose $\phi$ is tanh or $\sigma$-gelu for sufficiently small $\sigma$. Consider any initialization-stable parametrization with $r=0$. If $a_{L+1}+b_{L+1}<1$ or $2a_{L+1}+c<1$, then the parametrization is not stable.
\end{cor}

\begin{proof}
First suppose $a_{L+1}+b_{L+1}<1$ and $2a_{L+1}+c\ge1$. Then $\theta_{Lf}'=n^{1-(a_{L+1}+b_{L+1})}\to\infty$ but $\mathring{\theta}_{L+1}'\le1$. As in the proof of \cref{thm:geluFunctionMoves}, there is some $\eta\ne0$ such that $\EV Z^{\widehat{W}_{0}^{L+1}}Z^{\del x_{1}^{L}(\xi_{0})}=R$ for some $R\ne0$. Therefore, by the Master Theorem, $\frac{1}{n}\widehat{W}_{0}^{L+1}\del x_{1}^{L}(\xi_{0})\asto R\implies|W_{0}^{L+1}\Delta x_{1}^{L}(\xi_{0})|=\Theta(n^{1-(a_{L+1}+b_{L+1})})\to\infty$. This dominates $\Delta W_{1}^{L+1}x_{1}^{L}(\xi_{0})$, which by similar reasoning is $O(1)$. So $f_{1}(\xi_{0})$ diverges and the parametrization is not stable.

Now suppose $a_{L+1}+b_{L+1}\ge1$ and $2a_{L+1}+c<1$. This violates our simplifying assumption that $a_{L+1}+b_{L+1}\le2a_{L+1}+c$, but it's easy to see that $\frac{1}{n}\del W_{1}^{L+1}x_{1}^{L}(\xi_{0})\asto-\eta\mathring{\chi}_{0}\EV Z^{x_{0}^{L}}Z^{x_{1}^{L}(\xi_{0})}$. For $\eta$ small enough, this is close to $-\eta\mathring{\chi}_{0}\EV(Z^{x_{0}^{L}})^{2}$ and thus is nonzero. Then $|\Delta W_{1}^{L+1}x_{1}^{L}(\xi_{0})|=\Theta(n^{1-(2a_{L+1}+c)})\to\infty$. This dominates $W_{0}^{L+1}\Delta x_{1}^{L}(\xi_{0})=O(1)$, so $f_{1}(\xi_{0})$ diverges. Therefore, the parametrization is not stable.

Finally, suppose both $a_{L+1}+b_{L+1},2a_{L+1}+c<1$. If $a_{L+1}+b_{L+1}\ne2a_{L+1}+c$, then we have one of $\Delta W_{1}^{L+1}x_{1}^{L}(\xi_{0})$ and $W_{0}^{L+1}\Delta x_{1}^{L}(\xi_{0})$ dominate the other like the above, leading to divergence. If $a_{L+1}+b_{L+1}=2a_{L+1}+c$, then in the case of $\sigma$-gelu with small $\sigma$, $W_{0}^{L+1}\Delta x_{1}^{L}(\xi_{0})$ will dominate $\Delta W_{1}^{L+1}x_{1}^{L}(\xi_{0})$, as in \cref{thm:geluFunctionMoves}; and in the case of tanh, both have the same sign, as in \cref{thm:tanhfunctionmoves}. In either case, $f_{1}(\xi_{0})$ diverges, so the parametrization is not stable.
\end{proof}

\subsection{Putting Everything Together}
\label{sec:Altogether}
Finally, in this section we tie all of our insights above to prove our main theorems.

\begin{restatable}{thm}{stablePseudostable}
    \label{thm:stablePseudostable}
Suppose $\phi$ is tanh or $\sigma$-gelu for sufficiently small $\sigma$. A parametrization is stable iff it is pseudostable.
\end{restatable}

\begin{proof}
The ``if'' direction is given by \cref{prop:pseudostable=00003D>stable}. We now show that when any (in)equality of pseudostability is violated, the parametrization is not stable.

First, if \cref{eq:actlogitinit} is not satisfied, then \cref{thm:initstable} shows lack of stability.

Second, if \cref{eq:actlogitinit} is satisfied but $r<0$, then \cref{lem:r<0Unstable} shows lack of stability.

Finally, if \cref{eq:actlogitinit} is satisfied and $r\ge0$ but $a_{L+1}+b_{L+1}<1$ or $2a_{L+1}+c<1$, then \cref{cor:r=00003D0Unstable} or \cref{cor:r>0Unstable} shows lack of stability.
\end{proof}
Given this result, we will now just say ``stable'' instead of ``pseudostable'' from here on.

\nontrivial*

\begin{proof}
The case of $r=0$ and the case of $r>0$ are resp. given by \cref{prop:r=00003D0Nontrivial} and \cref{cor:r>0Nontrivial}.
\end{proof}

\main*

\begin{proof}
A nontrivial stable parametrization has either $r=0$ or $r>0$. By \cref{cor:r=0LearnsFeatures}, \cref{prop:r>0FixesFeatures}, and \cref{thm:kerneldynamics}, $r=0$ implies all of the statements in (1) and $r>0$ implies all of the statements in (2). Consequently, if feature learning happens, then clearly $r$ cannot be positive, so $r$ must be 0. Likewise, all of the statements in (1) imply $r=0$. Symmetrically, all of the statements in (2) about \emph{fixing features} imply $r>0$. Finally, if the parametrization is in kernel regime, then by \cref{cor:r=0LearnsFeatures}(1), $r$ cannot be 0, so $r>0$. This proves (1) and (2).

If the premise of (3) holds, then by the above, $r=0$, so the conclusion follows from \cref{cor:r=0LearnsFeatures}. This proves (3).

If $r=0$, then nontriviality means $a_{L+1}+b_{L+1}\ge1$. This implies $f_{0}(\xi)\asto0$ for all $\xi\in\isp$ (more precisely, $f_{0}(\xi)$ has standard deviation $\Theta(n^{1/2-(a_{L+1}+b_{L+1})})\to0$ by Central Limit Theorem).
The program describes the unconditional SGD trajectory of $f$ (as opposed to the case when $a_{L+1}+b_{L+1}=1/2$), so $f_t(\xi) \asto \mathring f_t(\xi)$ does not depend on $f_0$.
The converse is not true, for example because of \cref{cor:nngplimit}.
This prove (4).

(5) follows from \cref{cor:nngplimit} (which actually allows much more general $\phi$).
\end{proof}

\subparagraph{Proofs of \cref{thm:1LPMUPLimit,thm:decoupledSGD2LP,thm:SGD2LP}}
For any finite subset $\isp$ of the input space $\R^d$ (where $d=1$ here), we can write out the SGD computation as a Tensor Program like in \cref{sec:ProgramConstruction}.
Then the Master Theorem implies the convergence of $f_t(\xi) \asto \mathring f_t(\xi)$ for every $\xi \in\isp$.
Let $\isp_1 \sbe \cdots \sbe \isp_k \sbe \cdots$ be an infinite chain of finite subsets of $\R^d$ such that $\bigcup_k \isp_k$ is a dense subset of $\R^d$.
Then the convergence of $f_t(\xi) \asto \mathring f_t(\xi)$ holds for every $\xi \in\bigcup_k \isp_k$ (because we have almost sure convergence).
Finally, we apply a continuity argument to get this convergence for all of $\R^d$:

Because $\phi'$ and thus $\phi$ are pseudo-Lipschitz, they are locally Lipschitz (i.e.\ Lipschitz on any compact set).
In addition, the operator norms of $W^L$ are almost surely bounded from standard matrix operator bounds.
Thus one can see that the Tensor Program is locally Lipschitz in $\xi$.
Consequently, $\mathring f_t(\xi)$ is continuous in $\xi$.
This allows to pass from $\bigcup_k \isp_k$ to $\R^d$.

\subparagraph{Proofs of \cref{{prop:maximalupdate},prop:lastlayermaximal,prop:uniformupdate,thm:NTPvertex,thm:uniformparam}}
follow by dividing into cases of $r>0$ and $r=0$ and easy modification of the reasoning in \cref{sec:featurelearning,sec:kernelregime}.

\subparagraph{Proof of \cref{thm:kerneltrivializes}}
follows from straightforward calculations.
The basic outline of the calculations is:
1) During pretraining, $f$'s change is purely due to a) the interaction betwen $\Delta W^l, l\le L,$ and $W_0^{L+1}$, and b) the interaction between $x^L$ and $\Delta W^{L+1}$.
2) When $W^{L+1}$ is re-initialized in $g$, these interactions are killed.
The pretrained $\Delta W^l, l\le L,$ will cause $x^M$ to differ by $\Theta(1/\sqrt n)$ coordinatewise compared to if $\Delta W^l, l\le L,$ are all reset to 0, but this difference is uncorrelated with the last layer weights $W^{M+1}$ of $g$, so their interaction is subleading in $n$, i.e.\ in the infinite-width limit,
\[g_{T;t}(\xi) -  g_{0;t}(\xi)\asto 0,\]
whether all of $g$ or just the new weights are trained during fintetuning.

\end{document}